\documentclass{article}

\usepackage{microtype}
\usepackage{graphicx}
\usepackage{subfigure}
\usepackage{booktabs} % for professional tables
\usepackage[inline]{enumitem}

\usepackage{amsmath}
\usepackage{amsthm}
\usepackage[normalem]{ulem}
\usepackage{amssymb}
\usepackage{bbm}
\usepackage{mathtools}
\usepackage{color}
\usepackage{adjustbox}
\usepackage{multirow}
\usepackage{wrapfig}
\usepackage{xcolor}

\definecolor{mydarkblue}{rgb}{0,0.08,0.45}
\usepackage[colorlinks=true,
            linkcolor=mydarkblue,
            citecolor=mydarkblue,
            filecolor=mydarkblue,
            urlcolor=mydarkblue]{hyperref}
\usepackage{cleveref}
\usepackage{url}

\crefname{theorem}{\text{Theorem}}{\text{Theorems}}
\crefname{corollary}{\text{Corollary}}{\text{Corollary}}

\DeclareMathOperator*{\argmax}{arg\,max}

\usepackage[final]{neurips_2019}

\usepackage[utf8]{inputenc} % allow utf-8 input
\usepackage[T1]{fontenc}    % use 8-bit T1 fonts
\usepackage{hyperref}       % hyperlinks
\usepackage{url}            % simple URL typesetting
\usepackage{booktabs}       % professional-quality tables
\usepackage{amsfonts}       % blackboard math symbols
\usepackage{nicefrac}       % compact symbols for 1/2, etc.
\usepackage{microtype}      % microtypography

\renewcommand{\cite}{\citep}

\title{A Convex Relaxation Barrier to Tight Robustness Verification of Neural Networks
}

\author{%
  Hadi Salman\thanks{Work done as part of the \href{https://www.microsoft.com/en-us/research/academic-program/microsoft-ai-residency-program/}{Microsoft AI Residency Program}.}\\
  Microsoft Research AI\\
  \texttt{hadi.salman@microsoft.com} 
  \And
  Greg Yang \\
  Microsoft Research AI \\
  \texttt{gregyang@microsoft.com}
  \And
  Huan Zhang \\
  UCLA \\
  \texttt{huan@huan-zhang.com}
  \And
  Cho-Jui Hsieh \\
  UCLA \\
  \texttt{chohsieh@cs.ucla.edu}
  \And
  Pengchuan Zhang \\
  Microsoft Research AI \\
  \texttt{penzhan@microsoft.com}
}

\begin{document}

\maketitle

% Define symbols
\newcommand{\tail}[2]{#1_{\overline{[#2]}}}
\newcommand{\abs}[1]{|#1|}
\newcommand{\tabs}[1]{\left|#1\right|}
\newcommand{\wh}{\widehat}
\newcommand{\wt}{\widetilde}
\newcommand{\ov}{\overline}
\newcommand{\eps}{\epsilon}
\newcommand{\N}{\mathcal{N}}
\newcommand{\R}{\mathbb{R}}
\newcommand{\volume}{\mathrm{volume}}
\newcommand{\RHS}{\mathrm{RHS}}
\newcommand{\LHS}{\mathrm{LHS}}
\newcommand{\bone}{\mathbf{1}}
\renewcommand{\i}{\mathbf{i}}
\newcommand{\norm}[1]{\left\lVert#1\right\rVert}
\renewcommand{\varepsilon}{\epsilon}
\renewcommand{\tilde}{\wt}
\renewcommand{\hat}{\wh}
\renewcommand{\R}{\mathbb{R}}
\renewcommand{\N}{\mathcal{N}}
%\newcommand{\ReLU}{{$\mathsf{ReLU}$}}
% \makeatletter
\newcommand*{\rom}[1]{\expandafter\@slowromancap\romannumeral #1@}
% \makeatother

\newcommand{\x}[1]{{x^{#1}}}
\newcommand{\tidx}[1]{{\tilde{x}^{#1}}}
\newcommand{\X}[1]{{X^{#1}}}
\newcommand{\tidX}[1]{{\tilde{X}^{#1}}}
\newcommand{\upx}[1]{{\overline{x}^{#1}}}
\newcommand{\lwx}[1]{{\underline{x}^{#1}}}
\newcommand{\z}[1]{z^{#1}}
\newcommand{\tidz}[1]{{\tilde{z}^{#1}}}
\newcommand{\Z}[1]{{Z^{#1}}}
\newcommand{\tidZ}[1]{{\tilde{Z}^{#1}}}
\newcommand{\upz}[1]{\overline{z}^{#1}}
\newcommand{\lwz}[1]{\underline{z}^{#1}}
\newcommand{\lwL}[1]{{\underline{L}^{#1}}}
\newcommand{\upL}[1]{{\overline{L}^{#1}}}
\newcommand{\xo}{{x_0}}
\newcommand{\W}[1]{\mathbf{W}^{#1}}
\newcommand{\DD}[1]{\mathbf{D}^{#1}}
\newcommand{\Lam}[1]{\bm{\Lambda}^{#1}}
\newcommand{\upbias}[1]{\mathbf{T}^{#1}}
\newcommand{\lwbias}[1]{\mathbf{H}^{#1}}
\newcommand{\upbnd}[1]{\bm{u}^{#1}}
\newcommand{\lwbnd}[1]{\bm{l}^{#1}}
\newcommand{\y}{\bm{y}}
\newcommand{\bias}[1]{{b}^{#1}}
\newcommand{\setA}{\mathcal{A}}
\newcommand{\setIpos}[1]{\mathcal{I}^{+}_{#1}}
\newcommand{\setIneg}[1]{\mathcal{I}^{-}_{#1}}
\newcommand{\setIuns}[1]{\mathcal{I}_{#1}}
\newcommand{\set}[1]{\mathcal{#1}}
\newcommand{\Lipsloc}{L_{q,x_0}^j}

\newcommand{\gradl}[1]{\tilde{\bm{l}}^{#1}}
\newcommand{\gradu}[1]{\tilde{\bm{u}}^{#1}}
\newcommand{\gradC}[1]{\mathbf{C}^{#1}}
\newcommand{\gradL}[1]{\mathbf{L}^{#1}}
\newcommand{\gradU}[1]{\mathbf{U}^{#1}}
\newcommand{\gradLp}[1]{\mathbf{L'}^{#1}}
\newcommand{\gradUp}[1]{\mathbf{U'}^{#1}}
\newcommand{\Y}[1]{\mathbf{Y}^{#1}}
\newcommand{\Yp}[1]{\mathbf{Y'}^{#1}}

\newcommand{\Au}[1]{\mathbf{\Lambda}^{#1}}
\newcommand{\Al}[1]{\mathbf{\Omega}^{#1}}
\newcommand{\Du}[1]{\mathbf{\lambda}^{#1}}
\newcommand{\Dl}[1]{\mathbf{\omega}^{#1}}
\newcommand{\Ball}{\mathbb{B}_{p}(\xo,\epsilon)}
\newcommand{\Ph}[1]{\Phi_{#1}}
\newcommand{\upslp}[2]{\mathbf{\alpha}^{#1}_{U,{#2}}}
\newcommand{\lwslp}[2]{\mathbf{\alpha}^{#1}_{L,{#2}}}
\newcommand{\upicp}[2]{\mathbf{\beta}^{#1}_{U,{#2}}}
\newcommand{\lwicp}[2]{\mathbf{\beta}^{#1}_{L,{#2}}}

\newcommand{\gradupbnd}[1]{\bm{u'}^{#1}}
\newcommand{\gradlwbnd}[1]{\bm{l'}^{#1}}
\newcommand{\M}[1][]{\mathbf{M}^{#1}}
\newcommand{\setYp}[1]{\mathcal{T}^{+}_{#1}}
\newcommand{\setYn}[1]{\mathcal{T}^{-}_{#1}}
\newcommand{\setYo}[1]{\mathcal{T}_{#1}}
\newcommand{\A}{\mathbf{A}}
\newcommand{\B}{\mathbf{B}}

\newcommand{\hatW}[1]{\widehat{\mathbf{W}}^{#1}}
\newcommand{\checkW}[1]{\widecheck{\mathbf{W}}^{#1}}

\newcommand{\setS}{\mathcal{S}}
\newcommand{\CalC}{\mathcal{C}}
\newcommand{\CalD}{\mathcal{D}}
\newcommand{\CalA}{\mathcal{A}}
\newcommand{\CalB}{\mathcal{B}}

\newcommand{\dist}{\mathrm{dist}}

%%% theorem style
\newtheorem{lemma}{Lemma}[section]
\newtheorem{theorem}{Theorem}[section]
\newtheorem{proposition}[theorem]{Proposition}
\newtheorem{corollary}[theorem]{Corollary}
\newtheorem{remark}{Remark}[section]
\newtheorem{definition}[theorem]{Definition}
\newtheorem{assumption}[theorem]{Assumption}
\newtheorem{problem}[theorem]{Problem}

%%% Comments
\newcommand\gy[1]{\textcolor{red}{[GY: #1]}}
\newcommand\pz[1]{\textcolor{blue}{[PZ: #1]}}
\newcommand\hz[1]{\textcolor{green}{[HZ: #1]}}

% \newcommand\gy[1]{}
% \newcommand\pz[1]{}
% \newcommand\hz[1]{}

%%% Names of algorithms
\newcommand\lpd{\textsc{LP-greedy}}
\newcommand\lpp{\textsc{LP-last}}
\newcommand\lpo{\textsc{LP-all}}

\begin{abstract}
Verification of neural networks enables us to gauge their robustness against adversarial attacks.
Verification algorithms fall into two categories: \textit{exact} verifiers that run in exponential time and \textit{relaxed} verifiers that are efficient but incomplete.
In this paper, we unify all existing LP-relaxed verifiers, to the best of our knowledge, under a general convex relaxation framework. This framework works for neural networks with diverse architectures and nonlinearities and covers both primal and dual views of neural network verification.
Next, we perform large-scale experiments, amounting to more than 22 CPU-years, to obtain exact solution to the convex-relaxed problem that is optimal within our framework for ReLU networks.
We find the exact solution does not significantly improve upon the gap between PGD and existing relaxed verifiers for various networks trained normally or robustly on MNIST and CIFAR datasets.
Our results suggest there is an inherent \emph{barrier} to tight verification for the large class of methods captured by our framework.
We discuss possible causes of this barrier and potential future directions for bypassing~it. Our code and trained models are available at \url{http://github.com/Hadisalman/robust-verify-benchmark}.
\end{abstract}

\section{Introduction}
\label{sec:intoroduction}

\newcommand{\myopt}{\mathcal{O}}

A classification neural network $f: \R^n \to \R^K$ (where $f_i(x)$ should be thought of as the $i$th logit) is considered \emph{adversarially robust} with respect to an input $x$ and its neighborhood $\setS_{in}(x)$ if 
\begin{equation}\label{eq:rv}
    \min_{x' \in \setS_{in}(x), i \ne i^*} f_{i^*} (x) - f_i(x') > 0, \quad
    \text{where} \quad i^* = \argmax_j f_j(x).
\end{equation}

\vspace{-1em}
Many recent works have proposed robustness verification methods by lower-bounding  \cref{eq:rv}; the positivity of this lower bound proves the robustness w.r.t. $\setS_{in}(x)$.
A dominant approach thus far has tried to relax \cref{eq:rv} into a convex optimization problem, from either the primal view~\cite{zhang2018crown,gehr2018ai,singh2018fast,weng2018towards} or the dual view~\cite{wong2018provable,dvijotham2018dual,wang2018efficient}.
In \textbf{our first main contribution}, we propose a layer-wise convex relaxation framework that unifies these works and reveals the relationships between them (Fig.~\ref{fig:main}). We further show that the performance of methods within this framework is subject to a theoretical limit: the performance of the optimal layer-wise convex relaxation.

This then begs the question: is the road to fast and accurate robustness verification paved by just faster and more accurate layer-wise convex relaxation that approaches the theoretical limit?
In our {\bf second main contribution}, we answer this question in the \emph{negative}.
We perform extensive experiments with deep ReLU networks to compute the optimal layer-wise convex relaxation and compare with the LP-relaxed dual formulation from \citet{wong2018provable}, the PGD attack from \citet{madry2017towards}, and the mixed integer linear programming  (MILP) exact verifier from \citet{tjeng2018evaluating}.

Over different models, sizes, training methods, and datasets (MNIST and CIFAR-10), we find that (i) in terms of lower bounding the \textit{minimum $l_\infty$ adversarial distortion}\footnote{The radius of the largest $l_\infty$ ball in which no adversarial examples can be found.}, the optimal layer-wise convex relaxation only slightly improves the lower bound found by \citet{wong2018provable}, especially when compared with the upper bound provided by the PGD attack, which is consistently 1.5 to 5 times larger; (ii) in terms of upper bounding the \textit{robust error}, the optimal layer-wise convex relaxation does not significantly close the gap between the PGD lower bound (or MILP exact answer) and the upper bound from \citet{wong2018provable}.
Therefore, there seems to be an inherent barrier blocking our progress on this road of layer-wise convex relaxation, and we hope this work provokes much thought in the community on how to bypass it.

\begin{figure}
    \centering
  \includegraphics[width=0.65\textwidth]{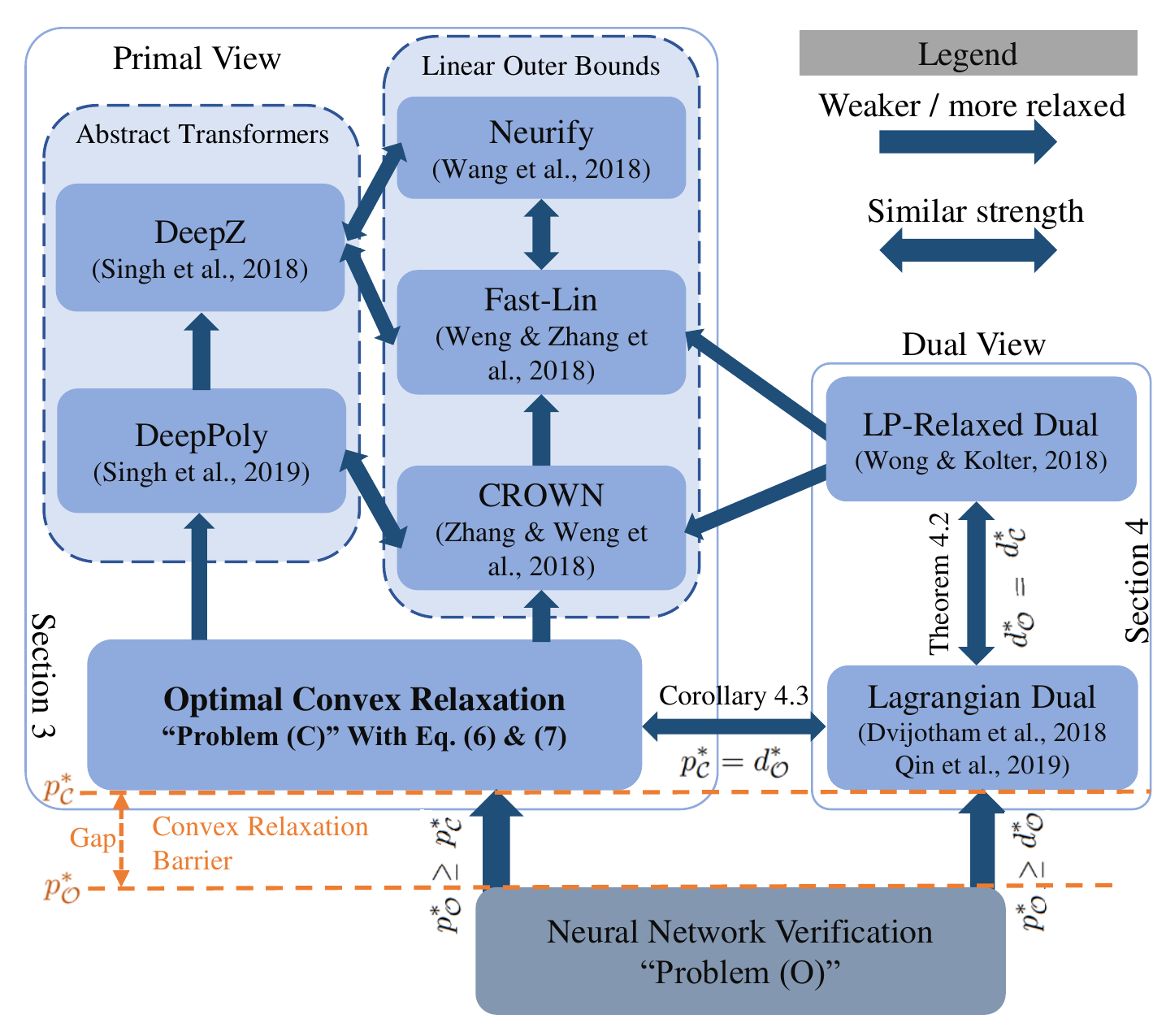}
  \label{fig:main}
  \caption{Relationship between existing relaxed algorithms and our framework. See Appendix~\ref{sec:greedy} for detailed discussions of each unlabeled arrow from the ``Primal view'' side.}
\end{figure}

\section{Preliminaries and Related Work}

\paragraph{Exact verifiers and NP-completeness.} For ReLU networks (piece-wise linear networks in general), exact verifiers solve the robustness verification problem~\eqref{eq:rv} by typically employing MILP solvers \cite{cheng2017maximum,lomuscio2017approach,dutta2018output,fischetti2017deep,tjeng2018evaluating,xiao2018training} or Satisfiability Modulo Theories (SMT) solvers \cite{scheibler2015towards,katz2017reluplex,carlini2017ground,ehlers2017formal}. However, due to the NP-completeness for solving such a problem \cite{katz2017reluplex,weng2018towards}, it can be really challenging to scale these to large networks.
It can take Reluplex~\cite{katz2017reluplex} several hours to find the minimum distortion of an example for a ReLU network with 5 inputs, 5 outputs, and 300 neurons. A recent work by \citet{tjeng2018evaluating} uses MILP to exactly verify medium-size networks, but the verification time is very sensitive to how a network is trained;  for example, it is fast for networks trained using the LP-relaxed dual formulation of \citet{wong2018provable}, but much slower for normally trained networks. A concurrent work by \citet{xiao2018training} trains networks with the objective of speeding up the MILP verification problem, but this compromises on the performance of the network.

\paragraph{Relaxed and efficient verifiers.} These verifiers solve a relaxed, but more computationally efficient, version of \eqref{eq:rv}, and have been proposed from different perspectives.
From the primal view, one can relax the nonlinearity in \eqref{eq:rv} into linear inequality constraints. This perspective has been previously explored as in the framework of ``abstract transformers''~ \citep{singh2018fast,singh2019abstract,Singh2019robustness,gehr2018ai,mirman2018differentiable}, via linear outer bounds of activation functions~\citep{zhang2018crown,weng2018towards,wang2018mixtrain,wang2018efficient}, or via interval bound propagation~\citep{gowal2018effectiveness,mirman2018differentiable}. From the dual view, one can study the dual of the relaxed problem \citep{wong2018provable,wong2018scaling} or study the dual of the original nonconvex verification problem \citep{dvijotham2018dual,dvijotham2018training,qin2018verification}.
In this paper, we unify both views in a common convex relaxation framework for NN verification, clarifying their relationships (as summarized in Fig.~\ref{fig:main}).

\citet{raghunathan2018semidefinite} formulates the verification of ReLU networks as a quadratic programming problem and then relaxes and solves this problem with a semidefinite programming (SDP) solver. 
While our framework does not cover this SDP relaxation, it is not clear to us how to extend the SDP relaxed verifier to general nonlinearities, for example max-pooling, which can be done in our framework on the other hand.
Other verifiers have been proposed to certify via an intermediary step of bounding the local Lipschitz constant \citep{hein2017formal,weng2018towards,raghunathan2018certified,zhang2018recurjac}, and others have used \textit{randomized smoothing} to certify with high-probability \citep{lecuyer2018certified, li2018second, cohen2019certified, salman2019provably}.
These are outside the scope of our framework.

Combining exact and relaxed verifiers, hybrid methods have shown some effectiveness \cite{bunel2018unified, Singh2019robustness}.
In fact, many exact verifiers also use relaxation as a subroutine to speed things up, and hence can be viewed as hybrid methods as well.
In this paper, we are not concerned with such techniques but only focus on \textit{relaxed verifiers}.

\section{Convex Relaxation from the Primal View}
\paragraph{Problem setting.}
In this paper, we assume that the neighborhood $\setS_{in}(x^{\text{nom}})$ is a convex set. An example of this is $\setS_{in}(x^{\text{nom}}) = \{x: \|x - x^{\text{nom}}\|_{\infty} \le \epsilon\}$, which is the constraint on $x$ in the $\ell_\infty$ adversarial attack model. We also assume that $f(x)$ is an $L$-layer feedforward NN. For notational simplicity, we denote $\{0, 1, \dots, L-1\}$ by $[L]$ and $\{\x{(0)}, \x{(1)}, \dots, \x{(L-1)}\}$ by $\x{[L]}$. We define $f(x)$ as,
\begin{equation}\label{eq:nn}
\begin{aligned}
    \x{(l+1)} =\sigma^{(l)}(\W{(l)} \x{(l)} + \bias{(l)}) \quad \forall l \in [L],\quad
    \text{and} \quad f(x) \coloneqq \z{(L)} =  \W{(L)} \x{(L)} + \bias{(L)},
\end{aligned}
\end{equation}
where $\x{(l)}\in \R^{n^{(l)}}$, $\z{(l)}\in \R^{n_z^{(l)}} $ , $\x{(0)}\coloneqq x \in \R^{n^{(0)}}$ is the input, $\W{(l)} \in \R^{n_z^{(l)} \times n^{(l)}}$ and $\bias{(l)} \in \R^{n_z^{(l)}}$ are the weight matrix and bias vector of the $l^{\text{th}}$ linear layer, and $\sigma^{(l)}: \R^{n_z^{(l)}} \to \R^{n^{(l+1)}}$ is a (nonlinear) activation function like (leaky-)ReLU, the sigmoid family (including sigmoid, arctan, hyperbolic tangent, etc), and the pooling family (MaxPool, AvgPool, etc). Our results can be easily extended to networks with convolutional layers and skip connections as well, similar to what is done in \citet{wong2018scaling}, as these can be seen as special forms of~\eqref{eq:nn}.

Consider the following optimization problem $\myopt(c, c_0, L, \lwz{[L]}, \upz{[L]})$:
\begin{equation}\label{eq:rvbnds}
\begin{aligned}
    \min_{(\x{[L+1]}, \z{[L]}) \in \CalD} \quad & c^{\top} \x{(L)} + c_0 \\
    \text{s.t.}\quad & \z{(l)} = \W{(l)} \x{(l)} + \bias{(l)}, l \in [L], \\
    & \x{(l+1)}=\sigma^{(l)}(\z{(l)}), l \in [L],
\end{aligned}
    \tag{$\myopt$}
\end{equation}
where the optimization domain $\CalD$ is
the set of activations and preactivations $\{\x{(0)}, \x{(1)}, \dots, \x{(L)}, \z{(0)}, \z{(1)}, \dots, \z{(L-1)}\}$   satisfying the bounds $\underline{z}^{(l)} \le z^{(l)} \le \overline z^{(l)}~ \forall l \in [L]$, i.e.,

\begin{equation}\label{def:CalD}
\begin{aligned}
    \CalD = \big\{(\x{[L+1]}, \z{[L]}) :{} &\x{(0)} \in \setS_{in}(x^{\text{nom}}),
    &\lwz{(l)} \le \z{(l)} \le \upz{(l)}, l \in [L] \big\}.
\end{aligned}
\end{equation}

If $c^{\top} = \W{(L)}_{i^{\text{nom}},:} - \W{(L)}_{i,:}$, $c_0 = \bias{(L)}_{i^{\text{nom}}} - \bias{(L)}_{i}$,  $\lwz{[L]} = - \infty$, and $\upz{[L]} = \infty$, then \eqref{eq:rvbnds} is equivalent to problem~\eqref{eq:rv}.
However, when we have better information about valid bounds $\lwz{[l]}$ and $\upz{[l]}$ of $\z{[l]}$, we can significantly narrow down the optimization domain and, as will be detailed shortly, achieve tighter solutions when we relax the nonlinearities.
We denote the minimal value of $\myopt(c, c_0, L, \lwz{[L]}, \upz{[L]})$ by $p^*(c, c_0, L, \lwz{[L]}, \upz{[L]})$, or just $p_{\myopt}^*$ when no confusion arises.

\paragraph{Obtaining lower and upper bounds $(\lwz{[L]}, \upz{[L]})$ by solving sub-problems.} 
This can be done by \textit{recursively} solving \eqref{eq:rvbnds} with specific choices of $c$ and $c_0$, which is a common technique used in many works~\cite{wong2018provable,dvijotham2018dual}. For example, one can obtain $\underline{z}_j^{(\ell)}$, a lower bound of $\z{(\ell)}_j$, by solving $\myopt(\W{(\ell)}_{j,:}{}^\top, \bias{(\ell)}_{j}, \ell, \lwz{[\ell]}, \upz{[\ell]})$; this shows that one can estimate $\lwz{(l)}$ and $\upz{(l)}$ inductively in $l$.
However, we may have millions of sub-problems to solve because practical networks can have millions of neurons. Therefore, it is crucial to have \textit{efficient} algorithms to solve \eqref{eq:rvbnds}.

\paragraph{Convex relaxation in the primal space.} Due to the nonlinear activation functions $\sigma^{(l)}$, the feasible set of  \eqref{eq:rvbnds} is nonconvex, which leads to the NP-completeness of the neural network verification problem \cite{katz2017reluplex,weng2018towards}. One natural idea is to do convex relaxation of its feasible set. Specifically, one can relax the nonconvex equality constraint $\x{(l+1)}=\sigma^{(l)}(\z{(l)})$ to convex inequality constraints, i.e.,
\begin{equation}\label{eq:rv_cmu}
\small
\begin{aligned}
    \min_{(\x{[L+1]}, \z{[L]}) \in \CalD} c^{\top} \x{(L)} + c_0 \quad
    \text{s.t.} \quad \z{(l)} = \W{(l)} \x{(l)} + \bias{(l)}, 
    \underline{\sigma}^{(l)}(\z{(l)}) \le \x{(l+1)} \le \overline{\sigma}^{(l)}(\z{(l)}), \forall l \in [L],
\end{aligned}
    \tag{$\CalC$}
\end{equation}
where $\underline{\sigma}^{(l)}(z)$ ( $\overline{\sigma}^{(l)}(z)$) is convex (concave) and satisfies  $\underline{\sigma}^{(l)}(z) \le \sigma^{(l)}(z) \le \overline{\sigma}^{(l)}(z)$ for $\lwz{(l)} \le z \le \upz{(l)}$.
We denote the feasible set of \eqref{eq:rv_cmu} by $\setS_{\mathcal{C}}$ and its minimum by $p^*_{\mathcal{C}}$. Naturally, we have that $\setS_{\mathcal{C}}$ is convex and $p^*_{\mathcal{C}} \le p_{\myopt}^*$.
For example, \citet{ehlers2017formal} proposed the following relaxations for the ReLU function $\sigma_{ReLU}(z) = \max(0,z)$ and MaxPool $\sigma_{MP}(z) = \max_k z_k$:
% and the following for MaxPool $\sigma(z) = \max_k z_k$:
\begin{gather}
    \underline{\sigma}_{ReLU}(\z{}) = \max(0, \z{}), \quad \overline{\sigma}_{ReLU}(\z{}) = \tfrac{\upz{}}{\upz{} - \lwz{}} \left(\z{} - \lwz{}\right),
        \label{eq:relu_ehlers}
        \\
    \underline{\sigma}_{MP}(\z{}) = \max_k \z{}_k \ge \sum_k (z_k-\upz{}_k) + \max_k \upz{}_k, \quad  \overline{\sigma}_{MP}(\z{}) = \sum_k (z_k + \lwz{}_k) - \max_k \lwz{}_k.
        \label{eq:maxpool_ehlers}
\end{gather}

\paragraph{The optimal layer-wise convex relaxation.}
As a special case, we consider the optimal layer-wise convex relaxation, where 
\begin{equation}\label{def:underoversigma}
\begin{aligned}
    \text{$\underline{\sigma}_{\text{opt}}(z)$ is the greatest convex function majored by $\sigma$},
        \\
    \text{$\overline{\sigma}_{\text{opt}}(z)$ is the smallest concave function majoring $\sigma$}.
\end{aligned}
\end{equation}
A precise definition can be found in \eqref{def:underoversigma_app} in Appendix~\ref{sec:optimal_convex}. In Fig.~\ref{fig:relax}, we show the optimal convex relaxation for several common activation functions. It is easy to see that \eqref{eq:relu_ehlers} is the optimal convex relaxation for ReLU, but \eqref{eq:maxpool_ehlers} is not optimal for the MaxPool function.
Under mild assumptions (non-interactivity as defined in \cref{defn:nonInteractivity}), the optimal convex relaxation of a nonlinear layer $\x{}=\sigma(\z{})$, i.e., its convex hull, is simply $\underline{\sigma}_{\text{opt}}(z) \le x \le \overline{\sigma}_{\text{opt}}(z)$ (see \cref{prop:hull_layer}). We denote the corresponding optimal relaxed problem as $\CalC_{\text{opt}}$, with its objective $p^*_{\CalC_{\text{opt}}}$.

 \enlargethispage{\baselineskip}

We emphasize that by \textit{optimal}, we mean the optimal convex relaxation of the \textit{single} nonlinear constraint $\x{(l+1)}=\sigma^{(l)}(\z{(l)})$ (see Proposition~\eqref{prop:hull_layer}) instead of the optimal convex relaxation of the nonconvex feasible set of the original problem~\eqref{eq:rvbnds}.
As such, techniques as in \cite{anderson2018strong,raghunathan2018semidefinite} are outside our framework; see \cref{sec:outOfScopeRelaxations} for more discussions.

\begin{figure}[t]
  \centering
  \includegraphics[width=0.6\textwidth]{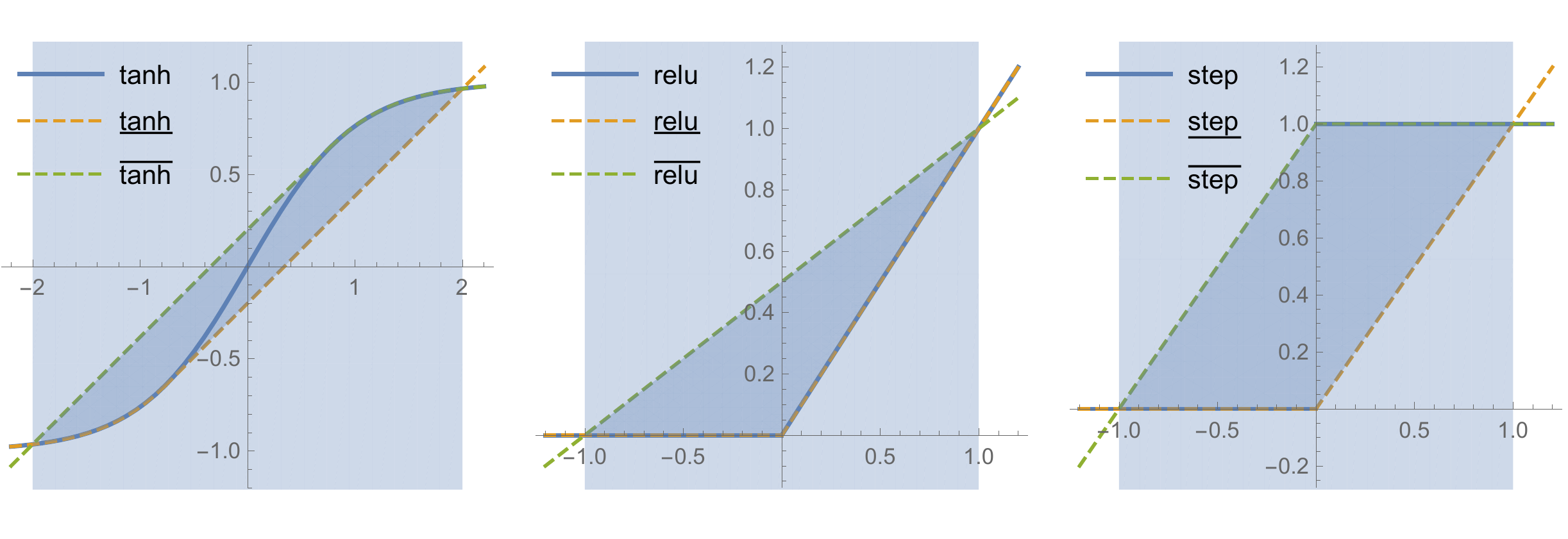}
  \includegraphics[width=0.35\textwidth]{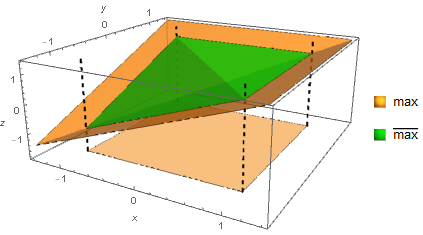}
  \caption{Optimal convex relaxations for common nonlinearities. For $\tanh$, the relaxation contains two linear segments and parts of the $\tanh$ function. For ReLU and the step function, the optimal relaxations are written as 3 and 4 linear constraints, respectively. For $z=\max(x,y)$, the light orange shadow indicates the pre-activation bounds for $x$ and $y$, and the optimal convex relaxation is lower bounded by the $\max$ function itself.}
  \label{fig:relax}
%   \vspace{-0.5cm}
\end{figure}

\paragraph{Greedily solving the primal with linear bounds.}
As another special case, when there are \textit{exactly one} linear upper bound  and one linear lower bound for each nonlinear layer in~\eqref{eq:rv_cmu} as follows:
\begin{equation}\label{eq:linearbounds}
\overline{\sigma}^{(l)}(z^{(l)})  \coloneqq \overline{a}^{(l)} z^{(l)} + \overline{b}^{(l)}, \qquad \underline{\sigma}^{(l)}(z^{(l)}) \coloneqq \underline{a}^{(l)} z^{(l)} + \underline{b}^{(l)}.
\end{equation}
the objective $p^*_{\CalC}$ can be \emph{greedily} bounded in a layer-by-layer manner. We can derive one linear upper and one linear lower bound of $\z{L}:=c^T \x{L} + c_0$ with respect to $z^{(L-1)}$, using the fact that  $z^{(L)}=c^T \sigma^{(L-1)}(\z{(L-1)}) + c_0$ and that $\sigma^{(L-1)}(\z{(L-1)})$ is linearly upper and lower bounded by $\overline{\sigma}^{(L-1)}(\z{(L-1)})$ and $\underline{\sigma}^{(L-1)}(\z{(L-1)})$. Because a linear combination of linear bounds (coefficients are related to the entries in~$c$) can be relaxed to a single linear bound, we can apply this technique again and replace $z^{(L-1)}$ with its upper and lower bounds with respect to $z^{(L-2)}$, obtaining the bound for $z^{(L)}$ with respect to ${z}^{(L-2)}$. Applying this repeatedly eventually leads to linear lower and upper bounds of $z^{(L)}$ with respect to the input $x^{(0)} \in \setS_{in}(x^{\text{nom}})$.

This perspective covers  Fast-Lin~\cite{weng2018towards}, DeepZ~\cite{singh2018fast} and Neurify~\cite{wang2018efficient}, where the proposed linear lower bound has the same slope as the upper bound, i.e.,  $\underline{a}^{(l)} = \overline{a}^{(l)}$. The resulting shape is referred to as a \textit{zonotope} in \citet{gehr2018ai} and \citet{singh2018fast}. In CROWN~\cite{zhang2018crown} and DeepPoly~\cite{singh2019abstract}, this restriction is lifted and they can achieve better verification results than Fast-Lin and DeepZ. Fig.~\ref{fig:main} summarizes the relationships between these algorithms. Importantly, each of these works has its own merits on solving the verification problem; our focus here is to give a unified view on how they perform convex relaxation of the original verification problem \eqref{eq:rvbnds} in our framework. See Appendix~\ref{sec:greedy} for more discussions and other related algorithms.

\section{Convex Relaxation from the Dual View}
We now tackle the verification problem from the dual view and connect it to the primal view.

\textbf{Strong duality for the convex relaxed problem.\ }
As in \citet{wong2018provable}, we introduce the dual variables for \eqref{eq:rv_cmu} and write its Lagrangian dual as
\begin{equation}\label{eq:rv_cmu_dual}
\small
\begin{aligned}
     g_{\mathcal{C}}(\mu^{[L]}, \underline\lambda^{[L]}, \overline\lambda^{[L]}) \coloneqq &\min_{(\x{[L+1]}, \z{[L]}) \in \CalD} \quad  c^{\top} \x{(L)} + c_0  + \sum_{l=0}^{L-1} \mu^{(l)}{}^\top ( \z{(l)} - \W{(l)} \x{(l)} - \bias{(l)} ) \\
    &- \sum_{l=0}^{L-1} \underline\lambda^{(l)}{}^\top ( \x{(l+1)} -\underline{\sigma}^{(l)}(\z{(l)}) ) + \sum_{l=0}^{L-1} \overline \lambda^{(l)}{}^\top ( \x{(l+1)}-\overline{\sigma}^{(l)}(\z{(l)}) ).
\end{aligned}
\end{equation}
By weak duality \cite{boyd2004convex},
\begin{equation}\label{eq:rv_cmu_dual_weak}
\begin{aligned}
    d^*_{\mathcal{C}} \coloneqq \max_{\mu^{[L]},\underline \lambda^{[L]} \ge 0, \overline \lambda^{[L]} \ge 0} g_{\mathcal{C}}(\mu^{[L]}, \underline \lambda^{[L]}, \overline \lambda^{[L]})
        \le p^*_{\mathcal{C}},
\end{aligned}
\end{equation}
but in fact we can show strong duality under mild conditions as well (note that the following result cannot be obtained by trivially applying Slater's condition; see \cref{app:strongdual} and \cref{fig:strongduality}).
\begin{theorem}[$p^*_{\mathcal{C}} = d^*_{\mathcal{C}}$]\label{thm:strongdual}
Assume that both $\underline{\sigma}^{(l)}$ and $\overline{\sigma}^{(l)}$ have a finite Lipschitz constant in the domain $[\lwz{(l)}, \upz{(l)}]$ for each $l \in [L]$. Then strong duality holds between \eqref{eq:rv_cmu} and \eqref{eq:rv_cmu_dual_weak}.
\end{theorem}

\paragraph{The optimal layer-wise dual relaxation.}
\cref{thm:strongdual} shows that taking the dual of the layer-wise convex relaxed problem \eqref{eq:rv_cmu} cannot do better than the original relaxation.
To obtain a tighter dual problem, one could directly study the Lagrangian dual of the original \eqref{eq:rvbnds},
\begin{equation}\label{eq:rv_dm}
\small
\begin{aligned}
     g_{{\myopt}}(\mu^{[L]}, \lambda^{[L]}) \coloneqq 
     \min_{\CalD} c^{\top} \x{(L)} + c_0  + \sum_{l=0}^{L-1} \mu^{(l)}{}^\top ( \z{(l)} - \W{(l)} \x{(l)} - \bias{(l)} )  + \sum_{l=0}^{L-1} \lambda^{(l)}{}^\top ( \x{(l+1)}-\sigma^{(l)}(\z{(l)}) ),
\end{aligned}
\end{equation}
where the min is taken over $\{(\x{[L+1]}, \z{[L]}) \in \CalD\}$.
This was first proposed in \citet{dvijotham2018dual}.
Note, again, by weak duality,
\begin{equation}\label{eq:rv_dm_dual_weak}
\begin{aligned}
    d^*_{{\myopt}} &\coloneqq \max_{\mu^{[L]},\lambda^{[L]}} g_{{\myopt}}(\mu^{[L]}, \lambda^{[L]}) \le p_{\myopt}^*,
\end{aligned}
\end{equation}
and $d^*_{\myopt}$ would seem to be strictly better than $d^*_{\mathcal C}$.
Unfortunately, they turn out to be equivalent:
\begin{theorem}[$d^*_{{\myopt}} = d^*_{\mathcal{C}_{\text{opt}}}$]\label{thm:cmudm}
    Assume that the nonlinear layer $\sigma^{(l)}$ is non-interactive (\cref{defn:nonInteractivity}) and the optimal layer-wise relaxation $\underline{\sigma}^{(l)}_{\text{opt}}$ and $\overline{\sigma}^{(l)}_{\text{opt}}$ are defined in \eqref{def:underoversigma}. Then the lower bound $d^*_{\mathcal{C}_{\text{opt}}}$ provided by the dual of the optimal layer-wise convex-relaxed problem \eqref{eq:rv_cmu_dual_weak} and $d^*_{{\myopt}}$ provided by the dual of the original problem \eqref{eq:rv_dm_dual_weak} are the same. 
\end{theorem}

The complete proof is in Appendix~\ref{app:cmudm} \footnote{Theorem 2 in \citet{dvijotham2018dual} is a special case of our Theorem~\ref{thm:cmudm}, when applied to ReLU networks.
Our proof makes use of the Fenchel-Moreau theorem to deal with general nonlinearities, which is different from that in \citet{dvijotham2018dual}. 
}.
Theorem~\ref{thm:cmudm} combined with the strong duality result of Theorem~\ref{thm:strongdual} implies that the primal relaxation~\eqref{eq:rv_cmu} and the two kinds of dual relaxations, \eqref{eq:rv_cmu_dual_weak} and \eqref{eq:rv_dm_dual_weak}, are all blocked by the same \emph{barrier}.
As concrete examples:
\begin{corollary}[$p^*_{\mathcal{C}_{\text{opt}}} = d^*_{{\myopt}}$]\label{cor:cmuledm}
Suppose that the nonlinear activation functions $\sigma^{(l)}$  for all $l \in [L]$ are (for example) among the following: ReLU, step, ELU, sigmoid, tanh, polynomials and max pooling with disjoint windows. Assume that $\underline{\sigma}_{\text{opt}}^{(l)}$ and $\overline{\sigma}_{\text{opt}}^{(l)}$ are defined in \eqref{def:underoversigma}, respectively. Then we have that the lower bound $p^*_{\mathcal{C}_{\text{opt}}}$ provided by the primal optimal layer-wise relaxation \eqref{eq:rv_cmu} and $d^*_{{\myopt}}$ provided by the dual relaxation \eqref{eq:rv_dm_dual_weak} are the same.
\end{corollary}

\paragraph{Greedily solving the dual with linear bounds.} 
When the relaxed bounds $\underline{\sigma}$ and $\overline{\sigma}$ are linear as defined in \eqref{eq:linearbounds},
the dual objective~\eqref{eq:rv_cmu_dual_weak} can be lower bounded as below: 
\begin{equation*}
% \label{eq:dual_suboptimal_val}
\begin{aligned}
    p^*_{\mathcal{C}} = d^*_{\mathcal{C}} \ge  \sum_{l=0}^{L-1} \left( \overline{b}^{(l)\top} \left(\lambda^{(l)}\right)_{+} - \underline{b}^{(l)\top} \left(\lambda^{(l)}\right)_{-} - \bias{(l)}{}^\top \mu^{(l)} \right) + c_0 - \sup_{\x{} \in \setS_{in}(x^{\text{nom}})} \left( \W{(0){\top}} \mu^{(0)} \right)^{\top} \x{}, 
\end{aligned}
\end{equation*}
where the dual variables $(\mu^{[L]}, \lambda^{[L]})$ are determined by a backward propagation
\begin{equation*}
% \label{eq:dual_suboptimal_mu}
    \lambda^{(L-1)} = -c, \quad \mu^{(l)} = \overline{a}^{(l)} \left(\lambda^{(l)}\right)_{+} + \underline{a}^{(l)} \left(\lambda^{(l)}\right)_{-}, \quad \lambda^{(l-1)} =  \W{(l)}{}^\top \mu^{(l)} \quad \forall l\in [L-1],
\end{equation*}
We provide the derivation of this algorithm in Appendix~\ref{app:greedydual}.
It turns out that this algorithm can exactly recover the algorithm proposed in \citet{wong2018provable}, where 
\begin{align*}
    \underline{\sigma}^{(l)}(\z{(l)}) \coloneqq \alpha^{(l)} \z{(l)}, \quad \overline{\sigma}^{(l)}(\z{(l)})  \coloneqq \tfrac{\overline{z}^{(l)}}{\overline{z}^{(l)} - \underline{z}^{(l)}} (\z{(l)} - \underline{z}^{(l)}),
\end{align*}
and $0 \leq \alpha^{(l)} \leq 1$ represents the slope of the lower bound.
When $\alpha^{(l)} = \frac{\overline{z}^{(l)}}{\overline{z}^{(l)} - \underline{z}^{(l)}}$, the greedy algorithm also recovers Fast-Lin~\cite{weng2018towards}, which explains the arrow from \citet{wong2018provable} to \citet{weng2018towards} in Fig.~\ref{fig:main}.  When $\alpha^{(l)}$ is chosen adaptively as in CROWN~\cite{zhang2018crown}, the greedy algorithm then recovers CROWN, which explains the arrow from \citet{wong2018provable} to \citet{zhang2018crown} in Fig.~\ref{fig:main}.
See Appendix~\ref{sec:greedy} for more discussions on the relationship between the primal and dual greedy solvers.

\section{Optimal LP-relaxed Verification}

In the previous sections, we presented a framework that subsumes all existing layer-wise convex-relaxed verification algorithms except that of \citet{raghunathan2018semidefinite}.
For ReLU networks, being piece-wise linear, these correspond exactly to the set of all existing LP-relaxed algorithms, as discussed above.
% LP-relaxed verification algorithms.
We showed the existence of a barrier, $p^*_{\mathcal C}$, that limits all such algorithms.
Is this just theoretical babbling or is this barrier actually problematic in practice? 

In the next section, we perform extensive experiments on deep ReLU networks, evaluating the tightest convex relaxation afforded by our framework (denoted \textbf{\lpo{}}) against a greedy dual algorithm (Algorithm 1 of \citet{wong2018provable}, denoted \textbf{\lpd{}}) as well as another algorithm \textbf{\lpp{}}, intermediate in speed and accuracy between them.
Both \lpd{} and \lpp{} solve the bounds $\lwz{[L]}, \upz{[L]}$ by setting the dual variables heuristically (see previous section), but \lpd{} solves the adversarial loss in the same manner while  \lpp{} solves this final LP exactly.
We also compare them with the opposite bounds provided by PGD attack \citep{madry2017towards}, as well as exact results from MILP \citep{tjeng2018evaluating} \footnote{Note that in practice (as in \cite{tjeng2018evaluating}), MILP has a time budget, and usually not every sample can be verified within that budget, so that in the end we still obtain only lower and upper bounds given by samples verified to be robust or nonrobust}.

For the rest of the main text, we are only concerned with ReLU networks, so (\ref{eq:rv_cmu}) subject to \eqref{eq:relu_ehlers} is in fact an LP.

\subsection{$\lpo$ Implementation Details}

In order to exactly solve the tightest LP-relaxed verification problem of a ReLU network, two steps are required: (A)  obtaining the tightest pre-activation upper and lower bounds of all the neurons in the NN, excluding those in the last layer, then (B) solving the LP-relaxed verification problem exactly for the last layer of the NN. 

\paragraph{Step A: Obtaining Pre-activation Bounds.}
This can be done by solving sub-problems of the orginial relaxed problem \eqref{eq:rv_cmu} subject to \eqref{eq:relu_ehlers}.
Given a NN with $L_0$ layers, for each layer $l_0 \in [L_0]$, we obtain a lower (resp. upper) bound $\lwz{(l_0)}_j$ (resp. $\upz{(l_0)}_j$) of $\z{(l_0)}_j$, for all neurons $j \in [n^{(l_0)}]$.
We do this by setting
\begin{align*}
    &L \gets l_0,\quad
    c^{\top} \gets \W{(l_0)}_{j,:} \text{ (resp. }c^{\top} \gets -\W{(l_0)}_{j,:}\text{)},\quad
    c_0 \gets \bias{(l_0)}_{j} \text{ (resp. }c_0 \gets -\bias{(l_0)}_{j}\text{)}
\end{align*}
in \eqref{eq:rv_cmu} and computing the exact optimum.
However, we need to solve an LP for each neuron, and practical networks can have millions of them.
We utilize the fact that in each layer $l_0$, computing the bounds $\upz{(l_0)}_j$ and $\lwz{(l_0)}_j$ for each $j \in [n^{(l_0)}]$ can proceed independently in parallel.
Indeed, we design a scheduler to do so on a cluster with 1000 CPU-nodes.
See Appendix~\ref{sec:cluster} for details.

\paragraph{Step B: Solving the LP-relaxed Problem for the Last Layer.}
After obtaining the pre-activation bounds on all neurons in the network using step (A), we solve the LP in \eqref{eq:rv_cmu} subject to \eqref{eq:relu_ehlers} for all $j \in [n^{(L_0)}]\backslash\{j^{\text{nom}}\} $ obtained by setting 
\begin{align*}
    &L \gets L_0,\quad
    c^{\top} \gets \W{(L_0)}_{j^{\text{nom}},:} - \W{(L_0)}_{j,:},\quad
    c_0 \gets \bias{(L_0)}_{j^{\text{nom}}} - \bias{(L_0)}_{j}
\end{align*}
again in \eqref{eq:rv_cmu} and computing the exact minimum.
Here,  $j^{\text{nom}}$ is the true label of the data point $x^{\text{nom}}$ at which we are verifying the network.
\textit{We can certify the network is robust around $x^{\text{nom}}$ iff the solutions of all such LPs are positive, i.e. we cannot make the true class logit lower
than any other logits.} 
Again, note that these LPs are also independent of each other, so we can solve them in parallel.

Given any $x^{\text{nom}}$,  \lpo{}  follows steps (A) then (B) to produce a certificate whether the network is robust around a given datapoint or not. \lpp{} on the other hand solves only step (B), and instead of doing (A), it finds the preactivation bounds greedily as in Algorithm 1 of \citet{wong2018provable}.

\section{Experiments}

We conduct two experiments to assess the tightness of \lpo{}: 1) finding certified upper bounds on the robust error of several NN classifiers, 2) finding certified lower bounds on the minimum adversarial distortion $\epsilon$ using different algorithms.
All experiments are conducted on MNIST and/or CIFAR-10 datasets.

\paragraph{Architectures.} We conduct experiments on a range of ReLU-activated feedforward networks.
\textsc{MLP-A} and \textsc{MLP-B} refer to multilayer perceptrons: \textsc{MLP-A} has 1
hidden layer with 500 neurons, and \textsc{MLP-B} has 2 hidden layers with 100 neurons each. \textsc{CNN-small}, \textsc{CNN-wide-k}, and \textsc{CNN-deep-k} are the ConvNet architectures used in \citet{wong2018scaling}. Full details are in Appendix~\ref{sec:appendix-architectures}.

\paragraph{Training Modes.} We conduct experiments on networks trained with a regular cross-entropy (CE) loss function and
networks trained to be robust. These networks  are identified by a prefix corresponding to the method used to train them: \textbf{\textsc{LPd}} when the LP-relaxed dual formulation of \citet{wong2018provable} is used for robust training, \textbf{\textsc{Adv}} when adversarial examples generated using PGD are used for robust training, as in \citet{madry2017towards}, and \textbf{\textsc{Nor}} when the network is normally trained using the CE loss function. Training details are in Appendix~\ref{sec:appendix-training-modes}.

\paragraph{Experimental Setup.} We run experiments on a cluster with 1000 CPU-nodes.
The total run time amounts to more than 22 CPU-years.
Appendix \ref{sec:cluster} provides additional details about the computational resources and the scheduling scheme used, and Appendix \ref{sec:computation-time} provides statistics of the verification time in these experiments.

\begin{table*}[tbp]
\centering
\caption{Certified bounds on the robust error on the test set of MNIST for normally and robustly trained networks. The prefix of each network corresponds to the training method used:
\textbf{\textsc{Adv}} for PGD training \cite{madry2017towards}, \textbf{\textsc{Nor}} for normal CE loss training, and \textbf{\textsc{LPd}} when the LP-relaxed dual formulation of \citet{wong2018provable} is used for robust training.
}
\vskip 0.15in
\label{table:adv_error_bounds}
\begin{sc}
\begin{adjustbox}{max width=0.85\textwidth}
\begin{tabular}{rrrrrrrc}
\toprule

\multicolumn{1}{c}{\multirow{2}[2]{*}{Network}}
& \multicolumn{1}{c}{\multirow{2}[1]{*}{$\epsilon$}}
& \multicolumn{1}{c}{\multirow{2}[1]{*}{\begin{tabular}[c]{@{}c@{}}Test\\ Error\end{tabular}}}
& \multicolumn{2}{c}{Lower Bound}
& \multicolumn{3}{c}{Upper Bound}
\\\cmidrule(lr){4-5} \cmidrule(lr){6-8}

& & & \multicolumn{1}{c}{PGD} & \multicolumn{1}{c}{MILP} & \multicolumn{1}{c}{MILP} & \multicolumn{1}{c}{\lpo{}} & \multicolumn{1}{c}{\lpd{}}\\

\midrule
Adv-MLP-B
& 0.03 & 1.53\% & 4.17\%  & 4.18\%  & 5.78\%  & 10.04\% & 13.40\% \\

Adv-MLP-B
& 0.05   & 1.62\% & 6.06\%  & 6.11\%  & 11.38\%         &  23.29\% & 33.09\% \\

Adv-MLP-B
 & 0.1 & 3.33\% & 15.86\%  & 16.25\%  & 34.37\%       & 61.59\% & 71.34\%   \\

Adv-MLP-A
 & 0.1 & 4.18\% & 11.51\% & 14.36\% & 30.81\% & 60.14\% & 67.50\%   \\

\midrule
Nor-MLP-B
  & 0.02 & 2.05\% & 10.06\%  & 10.16\%  & 13.48\%       & 26.41\% & 35.11\%  \\
Nor-MLP-B
  & 0.03 & 2.05\% & 20.37\%  & 20.43\%  & 48.67\%        & 65.70\% & 75.85\%  \\
Nor-MLP-B
  & 0.05 & 2.05\% & 53.37\%  & 53.37\%  & 94.04\%        & 97.95\% & 99.39\%  \\

\midrule

LPd-MLP-B
  & 0.1 & 4.09\% & 13.39\%  & 14.45\%  & 14.45\%        & 
  17.24\% & 18.32\%  \\
LPd-MLP-B
  & 0.2 & 15.72\% & 33.85\%  & 36.33\%  & 36.33\%        & 
  37.50\% & 41.67\% \\
LPd-MLP-B
  & 0.3 & 39.22\% & 57.29\%  & 59.85\%  & 59.85\%        & 
  60.17\% & 66.85\%\\
LPd-MLP-B
  & 0.4 & 67.97\% & 81.85\%  & 83.17\%  & 83.17\%        & 
  83.62\%  & 87.89\%\\

\bottomrule

\end{tabular}
\end{adjustbox}
\end{sc}
\end{table*}

\subsection{Certified Bounds on the Robust Error}

Table~\ref{table:adv_error_bounds} presents the clean test errors and (upper and lower) bounds on the true robust errors for a range of classifiers trained with different procedures on MNIST. 
For both \textsc{Adv-} and \textsc{LPd-}trained networks, the $\epsilon$ in Table~\ref{table:adv_error_bounds} denotes the $l_\infty$-norm bound used for training \emph{and} robust testing;
for \textsc{Nor}mally-trained networks, $\epsilon$ is only used for the latter.

Lower bounds on the robust error are calculated by finding adversarial examples for inputs that are not robust. This is done by using PGD, a strong first-order attack, or using MILP \cite{tjeng2018evaluating}.
Upper bounds on the robust error are calculated by providing certificates of robustness for input that
is robust. This is done using MILP, the dual formulation (\lpd{}) presented by \citet{wong2018provable}, or our \lpo{} algorithm.

For the MILP results, we use the code accompanying the paper by \citet{tjeng2018evaluating}. We run the code in parallel on a cluster with 1000 CPU-nodes, and set the MILP solver's time limit to 3600 seconds. Note that this time limit is reached for \textsc{Adv} and \textsc{Nor}, and therefore the upper and lower bounds are separated by a gap that is especially large for some of the \textsc{Nor}mally trained networks. On the other hand, for \textsc{LPd}-trained networks, the MILP solver finishes within the time limit, and thus the upper and lower bounds match.

\textbf{Results.\ }
For all \textsc{Nor}mally and \textsc{Adv}-trained networks, we see that the certified upper bounds using \lpd{} and \lpo{} are very loose when we compare the gap between them to the lower bounds found by PGD and MILP.
As a sanity check, note that  \lpo{} gives a tighter bound than \lpd{} in each case, as one would expect.
Yet this improvement is not significant enough to close the gap with the lower bounds.

This sanity check also passes for \textsc{LPd}-trained networks, where the \lpd{}-certified robust error upper bound is, as expected, much closer to the true error (given by MILP here) than for other networks.
For $\epsilon = 0.1$, the improvement of $\lpo{}$-certified upper bound over $\lpd{}$ is at most modest, and the PGD lower bound is tighter to the true error.
For large $\epsilon$, the improvement is much more significant in relative terms, but the absolute improvement is only $4-7\%$.
In this large $\epsilon$ regime, however, both the clean and robust errors are quite large, so the tightness of $\lpo{}$ is less useful.

\begin{figure}
    \centering
    \includegraphics[width=0.60\textwidth]{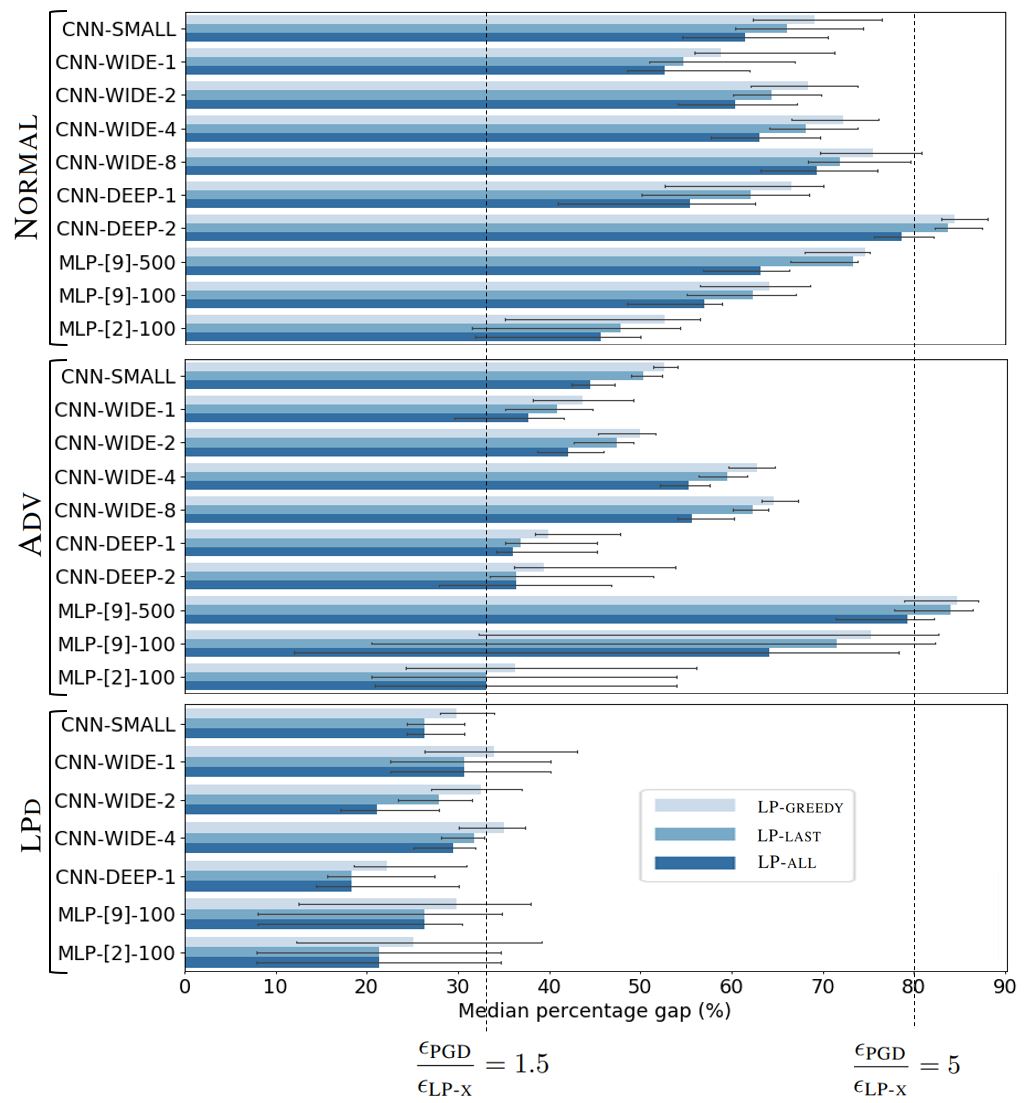}
    % \vspace{-1em}
    \caption{The median percentage gap between the convex-relaxed algorithms (\lpo{}, \lpp{}, and \lpd{}) and PGD estimates of the minimum adversarial distortion $\epsilon$ on ten samples of MNIST. The error bars correspond to 95\% confidence intervals. We highlight the $1.5\times$ and $5\times$ gaps between the $\epsilon$ value estimated by PGD, and those estimated by the LP-relaxed algorithms. For more details, please refer to Table~\ref{table:epsilon_bounds} in Appendix~\ref{sec:exp2results_appendix}.}
    \label{fig:gap_mnist}
\end{figure}

\subsection{Certified  Bounds on the Minimum Adversarial Distortion $\epsilon$}\label{sec:exp2}
We are interested in searching for the minimum adversarial distortion $\epsilon$, which is the radius of the largest $l_\infty$ ball in which no adversarial examples can be crafted. An upper bound on $\epsilon$ is calculated using PGD, and lower bounds are calculated using \lpd{}, \lpp{}, or our \lpo{}, all via binary search.
Since solving \lpo{} is expensive, we find the $\epsilon$-bounds only for ten samples of the MNIST and CIFAR-10 datasets. In this experiment, both \textsc{Adv-} and \textsc{LPd-}networks are trained with an $l_\infty$ maximum allowed perturbation of 0.1 and $8/255$ on MNIST and CIFAR-10, respectively.
See Appendix~\ref{sec:exp2_implementation} for details. Fig.~\ref{fig:gap_mnist} and \ref{fig:gap_cifar} in the Appendix show the median percentage gap (defined in Appendix~\ref{sec:exp2results_appendix}) between the convex-relaxed algorithms and PGD bounds of $\epsilon$ for MNIST and CIFAR, respectively. Details are reported in Tables~\ref{table:epsilon_bounds} and \ref{table:epsilon_bounds-cifar} in Appendix~\ref{sec:exp2results_appendix}.

On MNIST, the results show that for all networks trained \textsc{Nor}mally or via \textsc{Adv}, the certified lower bounds on $\epsilon$ are 1.5 to 5 times smaller than the upper bound found by PGD;
for \textsc{LPd} trained networks, below 1.5 times smaller. 
On CIFAR-10, the bounds are between 1.5 and 2 times smaller across all models.
The smaller gap for \textsc{LPd} is of course as expected following similar observations in prior work \cite{wong2018provable, tjeng2018evaluating}.
Furthermore, the improvement of \lpo{} and \lpp{} over \lpd{} is not significant enough to close the gap with the PGD upper bound. Note that similar results hold as well for randomly initialized networks (no training). To avoid clutter, we report these in Appendix~\ref{sec:appendix-random-networks}.

\section{Conclusions and Discussions}
In this work, we first presented a layer-wise convex relaxation framework that unifies all previous LP-relaxed verifiers, in both primal and dual spaces.
Then we performed extensive experiments to show that even the optimal convex relaxation for ReLU networks in this framework cannot obtain tight bounds on the robust error in all cases we consider here.
Thus any method will face a \emph{convex relaxation barrier} as soon as it can be described by our framework.
We look at how to bypass this barrier in Appendix~\ref{sec:howtobypass}.

Note that different applications have different requirements for the tightness of the verification, so our barrier could be a problem for some but not for others.
In so far as the ultimate goal of robustness verification is to construct a training method to lower certified error, this barrier is not necessarily problematic --- some such method could still produce networks for which convex relaxation as described by our framework produces accurate robust error bounds.
An example is the recent work of \citet{gowal2018effectiveness} which shows that interval bound propagation, which often leads to loose certification bounds, can still be used for verified training, and is able to achieve state-of-the-art verified accuracy when carefully tuned.
However, without a doubt, in all cases, tighter estimates should lead to better results, and we reveal a definitive ceiling on most current methods.

\clearpage
\bibliography{robust_verification}

\begin{thebibliography}{40}
\providecommand{\natexlab}[1]{#1}
\providecommand{\url}[1]{\texttt{#1}}
\expandafter\ifx\csname urlstyle\endcsname\relax
  \providecommand{\doi}[1]{doi: #1}\else
  \providecommand{\doi}{doi: \begingroup \urlstyle{rm}\Url}\fi

\bibitem[Anderson et~al.(2018)Anderson, Huchette, Tjandraatmadja, and
  Vielma]{anderson2018strong}
Ross Anderson, Joey Huchette, Christian Tjandraatmadja, and Juan~Pablo Vielma.
\newblock Strong convex relaxations and mixed-integer programming formulations
  for trained neural networks.
\newblock \emph{arXiv preprint arXiv:1811.01988}, 2018.

\bibitem[Boyd and Vandenberghe(2004)]{boyd2004convex}
Stephen Boyd and Lieven Vandenberghe.
\newblock \emph{Convex optimization}.
\newblock Cambridge university press, 2004.

\bibitem[Bunel et~al.(2018)Bunel, Turkaslan, Torr, Kohli, and
  Mudigonda]{bunel2018unified}
Rudy~R Bunel, Ilker Turkaslan, Philip Torr, Pushmeet Kohli, and Pawan~K
  Mudigonda.
\newblock A unified view of piecewise linear neural network verification.
\newblock In \emph{Advances in Neural Information Processing Systems}, pages
  4795--4804, 2018.

\bibitem[Carlini et~al.(2017)Carlini, Katz, Barrett, and
  Dill]{carlini2017ground}
Nicholas Carlini, Guy Katz, Clark Barrett, and David~L Dill.
\newblock Provably minimally-distorted adversarial examples.
\newblock \emph{arXiv preprint arXiv:1709.10207}, 2017.

\bibitem[Cheng et~al.(2017)Cheng, N{\"u}hrenberg, and Ruess]{cheng2017maximum}
Chih-Hong Cheng, Georg N{\"u}hrenberg, and Harald Ruess.
\newblock Maximum resilience of artificial neural networks.
\newblock In \emph{International Symposium on Automated Technology for
  Verification and Analysis}, pages 251--268. Springer, 2017.

\bibitem[Cohen et~al.(2019)Cohen, Rosenfeld, and Kolter]{cohen2019certified}
Jeremy~M Cohen, Elan Rosenfeld, and J~Zico Kolter.
\newblock Certified adversarial robustness via randomized smoothing.
\newblock \emph{arXiv preprint arXiv:1902.02918}, 2019.

\bibitem[Diamond and Boyd(2016)]{cvxpy}
Steven Diamond and Stephen Boyd.
\newblock {CVXPY}: A {P}ython-embedded modeling language for convex
  optimization.
\newblock \emph{Journal of Machine Learning Research}, 17\penalty0
  (83):\penalty0 1--5, 2016.

\bibitem[Domahidi et~al.(2013)Domahidi, Chu, and Boyd]{bib:Domahidi2013ecos}
A.~Domahidi, E.~Chu, and S.~Boyd.
\newblock {ECOS}: {A}n {SOCP} solver for embedded systems.
\newblock In \emph{European Control Conference (ECC)}, pages 3071--3076, 2013.

\bibitem[Dutta et~al.(2018)Dutta, Jha, Sankaranarayanan, and
  Tiwari]{dutta2018output}
Souradeep Dutta, Susmit Jha, Sriram Sankaranarayanan, and Ashish Tiwari.
\newblock Output range analysis for deep feedforward neural networks.
\newblock In \emph{NASA Formal Methods Symposium}, pages 121--138. Springer,
  2018.

\bibitem[Dvijotham et~al.(2018{\natexlab{a}})Dvijotham, Gowal, Stanforth,
  Arandjelovic, O'Donoghue, Uesato, and Kohli]{dvijotham2018training}
Krishnamurthy Dvijotham, Sven Gowal, Robert Stanforth, Relja Arandjelovic,
  Brendan O'Donoghue, Jonathan Uesato, and Pushmeet Kohli.
\newblock Training verified learners with learned verifiers.
\newblock \emph{arXiv preprint arXiv:1805.10265}, 2018{\natexlab{a}}.

\bibitem[Dvijotham et~al.(2018{\natexlab{b}})Dvijotham, Stanforth, Gowal, Mann,
  and Kohli]{dvijotham2018dual}
Krishnamurthy Dvijotham, Robert Stanforth, Sven Gowal, Timothy Mann, and
  Pushmeet Kohli.
\newblock A dual approach to scalable verification of deep networks.
\newblock \emph{UAI}, 2018{\natexlab{b}}.

\bibitem[Ehlers(2017)]{ehlers2017formal}
Ruediger Ehlers.
\newblock Formal verification of piece-wise linear feed-forward neural
  networks.
\newblock In \emph{International Symposium on Automated Technology for
  Verification and Analysis}, pages 269--286. Springer, 2017.

\bibitem[Fischetti and Jo(2017)]{fischetti2017deep}
Matteo Fischetti and Jason Jo.
\newblock Deep neural networks as 0-1 mixed integer linear programs: A
  feasibility study.
\newblock \emph{arXiv preprint arXiv:1712.06174}, 2017.

\bibitem[Gehr et~al.(2018)Gehr, Mirman, Drachsler-Cohen, Tsankov, Chaudhuri,
  and Vechev]{gehr2018ai}
Timon Gehr, Matthew Mirman, Dana Drachsler-Cohen, Petar Tsankov, Swarat
  Chaudhuri, and Martin Vechev.
\newblock {AI} 2: Safety and robustness certification of neural networks with
  abstract interpretation.
\newblock In \emph{2018 IEEE Symposium on Security and Privacy (SP)}, 2018.

\bibitem[Gowal et~al.(2018)Gowal, Dvijotham, Stanforth, Bunel, Qin, Uesato,
  Mann, and Kohli]{gowal2018effectiveness}
Sven Gowal, Krishnamurthy Dvijotham, Robert Stanforth, Rudy Bunel, Chongli Qin,
  Jonathan Uesato, Timothy Mann, and Pushmeet Kohli.
\newblock On the effectiveness of interval bound propagation for training
  verifiably robust models.
\newblock \emph{arXiv preprint arXiv:1810.12715}, 2018.

\bibitem[Hein and Andriushchenko(2017)]{hein2017formal}
Matthias Hein and Maksym Andriushchenko.
\newblock Formal guarantees on the robustness of a classifier against
  adversarial manipulation.
\newblock In \emph{Advances in Neural Information Processing Systems (NIPS)},
  pages 2266--2276, 2017.

\bibitem[Katz et~al.(2017)Katz, Barrett, Dill, Julian, and
  Kochenderfer]{katz2017reluplex}
Guy Katz, Clark Barrett, David~L Dill, Kyle Julian, and Mykel~J Kochenderfer.
\newblock Reluplex: An efficient smt solver for verifying deep neural networks.
\newblock In \emph{International Conference on Computer Aided Verification},
  pages 97--117. Springer, 2017.

\bibitem[Lecuyer et~al.(2018)Lecuyer, Atlidakis, Geambasu, Hsu, and
  Jana]{lecuyer2018certified}
Mathias Lecuyer, Vaggelis Atlidakis, Roxana Geambasu, Daniel Hsu, and Suman
  Jana.
\newblock Certified robustness to adversarial examples with differential
  privacy.
\newblock \emph{arXiv preprint arXiv:1802.03471}, 2018.

\bibitem[Li et~al.(2018)Li, Chen, Wang, and Carin]{li2018second}
Bai Li, Changyou Chen, Wenlin Wang, and Lawrence Carin.
\newblock Second-order adversarial attack and certifiable robustness.
\newblock \emph{arXiv preprint arXiv:1809.03113}, 2018.

\bibitem[Lomuscio and Maganti(2017)]{lomuscio2017approach}
Alessio Lomuscio and Lalit Maganti.
\newblock An approach to reachability analysis for feed-forward relu neural
  networks.
\newblock \emph{arXiv preprint arXiv:1706.07351}, 2017.

\bibitem[Madry et~al.(2017)Madry, Makelov, Schmidt, Tsipras, and
  Vladu]{madry2017towards}
Aleksander Madry, Aleksandar Makelov, Ludwig Schmidt, Dimitris Tsipras, and
  Adrian Vladu.
\newblock Towards deep learning models resistant to adversarial attacks.
\newblock \emph{arXiv preprint arXiv:1706.06083}, 2017.

\bibitem[Mirman et~al.(2018)Mirman, Gehr, and Vechev]{mirman2018differentiable}
Matthew Mirman, Timon Gehr, and Martin Vechev.
\newblock Differentiable abstract interpretation for provably robust neural
  networks.
\newblock In \emph{International Conference on Machine Learning}, pages
  3575--3583, 2018.

\bibitem[Qin et~al.(2019)Qin, Dvijotham, O'Donoghue, Bunel, Stanforth, Gowal,
  Uesato, Swirszcz, and Kohli]{qin2018verification}
Chongli Qin, Krishnamurthy~Dj Dvijotham, Brendan O'Donoghue, Rudy Bunel, Robert
  Stanforth, Sven Gowal, Jonathan Uesato, Grzegorz Swirszcz, and Pushmeet
  Kohli.
\newblock Verification of non-linear specifications for neural networks.
\newblock \emph{ICLR}, 2019.

\bibitem[Raghunathan et~al.(2018{\natexlab{a}})Raghunathan, Steinhardt, and
  Liang]{raghunathan2018certified}
Aditi Raghunathan, Jacob Steinhardt, and Percy Liang.
\newblock Certified defenses against adversarial examples.
\newblock \emph{International Conference on Learning Representations (ICLR),
  arXiv preprint arXiv:1801.09344}, 2018{\natexlab{a}}.

\bibitem[Raghunathan et~al.(2018{\natexlab{b}})Raghunathan, Steinhardt, and
  Liang]{raghunathan2018semidefinite}
Aditi Raghunathan, Jacob Steinhardt, and Percy~S Liang.
\newblock Semidefinite relaxations for certifying robustness to adversarial
  examples.
\newblock In \emph{Advances in Neural Information Processing Systems}, pages
  10900--10910, 2018{\natexlab{b}}.

\bibitem[Rockafellar(2015)]{rockafellar2015convex}
Ralph~Tyrell Rockafellar.
\newblock \emph{Convex analysis}.
\newblock Princeton university press, 2015.

\bibitem[Salman et~al.(2019)Salman, Li, Razenshteyn, Zhang, Zhang, Bubeck, and
  Yang]{salman2019provably}
Hadi Salman, Jerry Li, Ilya Razenshteyn, Pengchuan Zhang, Huan Zhang, Sebastien
  Bubeck, and Greg Yang.
\newblock Provably robust deep learning via adversarially trained smoothed
  classifiers.
\newblock In \emph{Advances in Neural Information Processing Systems}, pages
  11289--11300, 2019.

\bibitem[Scheibler et~al.(2015)Scheibler, Winterer, Wimmer, and
  Becker]{scheibler2015towards}
Karsten Scheibler, Leonore Winterer, Ralf Wimmer, and Bernd Becker.
\newblock Towards verification of artificial neural networks.
\newblock In \emph{MBMV}, pages 30--40, 2015.

\bibitem[Singh et~al.(2018)Singh, Gehr, Mirman, P{\"u}schel, and
  Vechev]{singh2018fast}
Gagandeep Singh, Timon Gehr, Matthew Mirman, Markus P{\"u}schel, and Martin
  Vechev.
\newblock Fast and effective robustness certification.
\newblock In \emph{Advances in Neural Information Processing Systems}, pages
  10825--10836, 2018.

\bibitem[Singh et~al.(2019{\natexlab{a}})Singh, Gehr, Püschel, and
  Vechev]{Singh2019robustness}
Gagandeep Singh, Timon Gehr, Markus Püschel, and Martin Vechev.
\newblock Robustness certification with refinement.
\newblock \emph{ICLR}, 2019{\natexlab{a}}.

\bibitem[Singh et~al.(2019{\natexlab{b}})Singh, Gehr, P{\"u}schel, and
  Vechev]{singh2019abstract}
Gagandeep Singh, Timon Gehr, Markus P{\"u}schel, and Martin Vechev.
\newblock An abstract domain for certifying neural networks.
\newblock \emph{Proceedings of the ACM on Programming Languages}, 3\penalty0
  (POPL):\penalty0 41, 2019{\natexlab{b}}.

\bibitem[Tjeng et~al.(2019)Tjeng, Xiao, and Tedrake]{tjeng2018evaluating}
Vincent Tjeng, Kai~Y. Xiao, and Russ Tedrake.
\newblock Evaluating robustness of neural networks with mixed integer
  programming.
\newblock In \emph{International Conference on Learning Representations}, 2019.
\newblock URL \url{https://openreview.net/forum?id=HyGIdiRqtm}.

\bibitem[Wang et~al.(2018{\natexlab{a}})Wang, Chen, Abdou, and
  Jana]{wang2018mixtrain}
Shiqi Wang, Yizheng Chen, Ahmed Abdou, and Suman Jana.
\newblock Mixtrain: Scalable training of formally robust neural networks.
\newblock \emph{arXiv preprint arXiv:1811.02625}, 2018{\natexlab{a}}.

\bibitem[Wang et~al.(2018{\natexlab{b}})Wang, Pei, Whitehouse, Yang, and
  Jana]{wang2018efficient}
Shiqi Wang, Kexin Pei, Justin Whitehouse, Junfeng Yang, and Suman Jana.
\newblock Efficient formal safety analysis of neural networks.
\newblock In \emph{Advances in Neural Information Processing Systems}, pages
  6369--6379, 2018{\natexlab{b}}.

\bibitem[Weng et~al.(2018)Weng, Zhang, Chen, Song, Hsieh, Boning, Dhillon, and
  Daniel]{weng2018towards}
Tsui-Wei Weng, Huan Zhang, Hongge Chen, Zhao Song, Cho-Jui Hsieh, Duane Boning,
  Inderjit~S Dhillon, and Luca Daniel.
\newblock Towards fast computation of certified robustness for {ReLU} networks.
\newblock In \emph{International Conference on Machine Learning}, 2018.

\bibitem[Wong and Kolter(2018)]{wong2018provable}
Eric Wong and Zico Kolter.
\newblock Provable defenses against adversarial examples via the convex outer
  adversarial polytope.
\newblock In \emph{International Conference on Machine Learning (ICML)}, pages
  5283--5292, 2018.

\bibitem[Wong et~al.(2018)Wong, Schmidt, Metzen, and Kolter]{wong2018scaling}
Eric Wong, Frank Schmidt, Jan~Hendrik Metzen, and J~Zico Kolter.
\newblock Scaling provable adversarial defenses.
\newblock \emph{Advances in Neural Information Processing Systems (NIPS)},
  2018.

\bibitem[Xiao et~al.(2019)Xiao, Tjeng, Shafiullah, and Madry]{xiao2018training}
Kai~Y. Xiao, Vincent Tjeng, Nur Muhammad~(Mahi) Shafiullah, and Aleksander
  Madry.
\newblock Training for faster adversarial robustness verification via inducing
  re{LU} stability.
\newblock In \emph{International Conference on Learning Representations}, 2019.
\newblock URL \url{https://openreview.net/forum?id=BJfIVjAcKm}.

\bibitem[Zhang et~al.(2018)Zhang, Weng, Chen, Hsieh, and
  Daniel]{zhang2018crown}
Huan Zhang, Tsui-Wei Weng, Pin-Yu Chen, Cho-Jui Hsieh, and Luca Daniel.
\newblock Efficient neural network robustness certification with general
  activation functions.
\newblock In \emph{Advances in Neural Information Processing Systems (NIPS)},
  dec 2018.

\bibitem[Zhang et~al.(2019)Zhang, Zhang, and Hsieh]{zhang2018recurjac}
Huan Zhang, Pengchuan Zhang, and Cho-Jui Hsieh.
\newblock Recurjac: An efficient recursive algorithm for bounding jacobian
  matrix of neural networks and its applications.
\newblock \emph{AAAI Conference on Artificial Intelligence}, 2019.

\end{thebibliography}
\bibliographystyle{plainnat}

%%%%%%%%%%%%%%%%%%%%%%%%%%%%%%%%%%%%%%%%%%%%%%%%%%%%%%%%%%%%%%%%%%%%%%%%%%%%%%%
%%%%%%%%%%%%%%%%%%%%%%%%%%%%%%%%%%%%%%%%%%%%%%%%%%%%%%%%%%%%%%%%%%%%%%%%%%%%%%%
% Appendix
%%%%%%%%%%%%%%%%%%%%%%%%%%%%%%%%%%%%%%%%%%%%%%%%%%%%%%%%%%%%%%%%%%%%%%%%%%%%%%%
%%%%%%%%%%%%%%%%%%%%%%%%%%%%%%%%%%%%%%%%%%%%%%%%%%%%%%%%%%%%%%%%%%%%%%%%%%%%%%%
\clearpage
\appendix
% \onecolumn
% \section{Additional Theoretical Results and Proofs}

\section{How to bypass the barrier?}
\label{sec:howtobypass}
The primal problem in our framework \eqref{eq:rv_cmu} has several possible sources of looseness:
\begin{enumerate}[label={(\arabic*)}]
    \item 
        We relax the nonlinearity $\sigma^{(l)}$ on a box domain $\{\lwz{(l)} \le z \le \upz{(l)}\}$.
        This relaxation is simple to perform, but might come at the cost of losing some correlations between the coordinates of $z$ and of obtaining a looser relaxation. Note that our framework does consider the correlations between coordinates of $z^{(l)}$ to get bounds for all later layers, however it relies on $\lwz{(l)}$ and $\upz{(l)}$ which are considered individually, without interactions within the same layer.
    \item
        We solve for the bounds $\lwz{[l]}, \upz{[l]}$ recursively, and we incur some gap for every recursion; a loose bound in earlier layers will make bounds for later layers even looser. This can be problematic for very deep networks or recurrent networks.
    \item
        In the specific case of ReLU, we lose a bit every time we relax over an unstable neuron;
        one possible future direction is to combine branch-and-bound with convex relaxation to strategically split the domains of unstable neurons.
\end{enumerate}
Any method that improves on any of the above issues can possibly bypass the barrier; see, e.g., SDP-based verifiers \citep{raghunathan2018semidefinite} can consider the interaction between each neuron within one layer; \cite{anderson2018strong} can relax the combination of one ReLU layer and one affine layer.
On the other hand, exact verifiers \citep{katz2017reluplex,ehlers2017formal}, local Lipschitz-constant-based verifiers \citep{zhang2018recurjac,  raghunathan2018certified}, and hybrid approaches \citep{bunel2018unified,Singh2019robustness} do not fall under the purview of our framework.
In general, none of them are strictly better than the convex relaxation approach and they make trading-offs between speed and accuracy.
However, it would be fruitful to consider combinations of these methods in the future, as done in \citet{Singh2019robustness}.
We hope our work will foster much thought in the community toward new relaxation paradigms for tight neural network verification.

\section{The optimal layer-wise convex relaxation}
\label{sec:optimal_convex}
\subsection{The optimal convex relaxation of a single nonlinear neuron}
In this section, we give the optimal convex relaxation of a single nonlinear neuron $\x{} = \sigma(z)$, which is the convex hull of its graph. Although the proof is elementary, we provide it for completeness. 
\begin{proposition}\label{prop:hull}
Suppose the activation function $\sigma:[\lwz{}, \upz{} ] \subset \R^{n_z} \to \R$ is bounded from above and below. Let $\underline{\sigma}_{\text{opt}}$ and $-\overline{\sigma}_{\text{opt}}$ be the greatest closed convex functions majored by $\sigma$ and $-\sigma$, respectively, i.e.,
\begin{equation}\label{def:underoversigma_app}
\begin{aligned}
    \underline{\sigma}_{\text{opt}}(z) &\coloneqq \sup_{(\alpha, \gamma) \in \mathcal{A}} \alpha^{\top} z + \gamma, \quad
        \text{where } \mathcal A = \{(\alpha, \gamma): \alpha^\top z' + \gamma \le \sigma(z'), \forall z' \in [\lwz{}, \upz{}]\},
        \\
    \overline{\sigma}_{\text{opt}}(z) &\coloneqq \inf_{(\alpha, \gamma) \in \mathcal A'} \alpha^{\top} z + \gamma,
    \quad
        \text{where } \mathcal A' = \{(\alpha, \gamma): \alpha^\top z' + \gamma \ge \sigma(z'), \forall z' \in [\lwz{}, \upz{}]\}
\end{aligned}
\end{equation}
 Then we have,
 \begin{enumerate}
     \item Both $\underline{\sigma}_{\text{opt}}$ and $\overline{\sigma}_{\text{opt}}$ are continuous in $[\lwz{}, \upz{}]$.
     \item
\begin{equation*}
\begin{aligned}
    &\big\{(\z{}, \x{}): \underline{\sigma}_{\text{opt}}(\z{}) \le \x{} \le \overline{\sigma}_{\text{opt}}(\z{}), \lwz{} \le \z{} \le \upz{} \big\} =
     \overline{\text{conv}}\big(\{(\z{}, \x{}): \x{}=\sigma(\z{}), \lwz{} \le \z{} \le \upz{}\}\big),
\end{aligned}
\end{equation*}
where $\overline{\text{conv}}$ denotes the closed convex hull.
 \end{enumerate}
\end{proposition}
\begin{proof}$ $\newline

1. By the boundedness of $\sigma$ on $[\lwz{}, \upz{}]$, we know that the effective domain of $\underline{\sigma}_{\text{opt}}$ and $-\overline{\sigma}_{\text{opt}}$ is $[\lwz{}, \upz{}]$. By definition \eqref{def:underoversigma_app}, $\underline{\sigma}_{\text{opt}}$ and $-\overline{\sigma}_{\text{opt}}$ are closed convex functions. By Theorem 10.2 in \citet{rockafellar2015convex}, we know that both $\underline{\sigma}_{\text{opt}}$ and $-\overline{\sigma}_{\text{opt}}$ are continuous in $[\lwz{}, \upz{}]$, so is $\overline{\sigma}_{\text{opt}}$.

2. We first decompose the left-hand-side into 3 terms:
\begin{equation*}
\begin{aligned}
    &\big\{(\z{}, \x{}): \underline{\sigma}_{\text{opt}}(\z{}) \le \x{} \le \overline{\sigma}_{\text{opt}}(\z{}), \lwz{} \le \z{} \le \upz{} \big\} = 
     \big\{\underline{\sigma}_{\text{opt}}(\z{}) \le \x{} \big\} \cap \big\{\x{} \le \overline{\sigma}_{\text{opt}}(\z{}) \big\} \cap \big\{ \lwz{} \le \z{} \le \upz{} \big\}.
\end{aligned}
\end{equation*}
Let $\underline{\mathcal{F}}= \{(\alpha, \gamma): \alpha^T z'+\gamma \le \sigma(z'), \forall z'\in [\lwz{}, \upz{}]\}$ and $\overline{\mathcal{F}}= \{(\alpha, \gamma): \alpha^T z'+\gamma \ge \sigma(z'), \forall z'\in [\lwz{}, \upz{}]\}$. For the first term, by definition \eqref{def:underoversigma_app} we have
\begin{equation*}
\begin{aligned}
    \big\{\underline{\sigma}_{\text{opt}}(\z{}) \le \x{} \big\} &=\cap_{\{(\alpha, \gamma): \alpha^T z'+\gamma \le \sigma(z'), \forall z'\in [\lwz{}, \upz{}]\}} \{ \alpha^T z + \gamma \le x\} \\
    & = \cap_{\{(\alpha, \beta, \gamma): \beta < 0, \alpha^T z'+\beta\sigma(z') + \gamma \le 0, \forall z'\in [\lwz{}, \upz{}]\}} \{ \alpha^T z + \beta x + \gamma \le 0\}.
\end{aligned}
\end{equation*}
For the second term, by definition \eqref{def:underoversigma_app} we have
\begin{equation*}
\begin{aligned}
    \big\{\x{} \le \overline{\sigma}_{\text{opt}}(\z{}) \big\} &=\cap_{\{(\alpha, \gamma): \alpha^T z'+\gamma \le -\sigma(z'), \forall z'\in [\lwz{}, \upz{}]\}} \{ \alpha^T z + \gamma \le -x\} \\
    & = \cap_{\{(\alpha, \beta, \gamma): \beta > 0, \alpha^T z'+\beta\sigma(z') + \gamma \le 0, \forall z'\in [\lwz{}, \upz{}]\}} \{ \alpha^T z + \beta x + \gamma \le 0\}.
\end{aligned}
\end{equation*}
For the third term, we have
\begin{equation*}
\begin{aligned}
    \big\{ \lwz{} \le \z{} \le \upz{} \big\} &=\cap_{\{(\alpha, \gamma): \alpha^T z'+\gamma \le 0, \forall z'\in [\lwz{}, \upz{}]\}} \{ \alpha^T z + \gamma \le 0\} \\
    & = \cap_{\{(\alpha, \beta, \gamma): \beta = 0, \alpha^T z'+\beta\sigma(z') + \gamma \le 0, \forall z'\in [\lwz{}, \upz{}]\}} \{ \alpha^T z + \beta x + \gamma \le 0\}.
\end{aligned}
\end{equation*}
Combining the three terms, we conclude the proof by
\begin{equation*}
\begin{aligned}
    & \big\{(\z{}, \x{}): \underline{\sigma}_{\text{opt}}(\z{}) \le \x{} \le \overline{\sigma}_{\text{opt}}(\z{}), \lwz{} \le \z{} \le \upz{} \big\} \\
    = &\cap_{\{(\alpha, \beta, \gamma): \alpha^T z'+\beta\sigma(z') + \gamma \le 0, \forall z'\in [\lwz{}, \upz{}]\}} \{ \alpha^T z + \beta x + \gamma \le 0\} \\
    = &\overline{\text{conv}}\big(\{(\z{}, \x{}): \x{}=\sigma(\z{}), \lwz{} \le \z{} \le \upz{}\}\big),
\end{aligned}
\end{equation*}
where we use the definition of closed convex hull in the last identity.

\end{proof}

\subsection{The optimal convex relation of a nonlinear layer}
\label{sec:nonInteractivity}
When $\x{(l+1)} = \sigma^{(l)}(\z{(l)})$ is a nonlinear layer that has a vector output $\x{(l+1)} \in \R^{n^{(l+1)}}$, the optimal convex relaxation may not have a simple analytic form as $\underline{\sigma}_{\text{opt}}^{(l)}(\z{(l)}) \le \x{(l+1)} \le \overline{\sigma}_{\text{opt}}^{(l)}(\z{(l)})$. Fortunately, if there is no interaction (as defined below) among the output neurons, the optimal convex relaxation can be given as a simple analytic form.
\begin{definition}[non-interactive layer]\label{defn:nonInteractivity}
Let $\sigma: \R^m \to \R^n$ and $x = \sigma(z)$ be a nonlinear layer with input $z \in [\lwz{}, \upz{} ] \subset \R^m$ and output $x \in \R^n$. For each output $x_j$, let $I_j \subset [m]$ be the minimal set of $z$'s entries that affect $x_j$, where $x_j = \sigma(z_{I_j})$.
We call the layer $x = \sigma(z)$ \emph{non-interactive} if the sets $I_j$ ($j \in [n]$) are mutually disjoint. 
\end{definition}

Commonly used nonlinear activation layers are all non-interactive. It is obvious that all entry-wise nonlinear layers, such as (leaky-)ReLU and sigmoid, are non-interactive. A MaxPool layer with non-overlapping regions (stride no smaller than kernel size) is also non-interactive. Finally, any layer with scalar-valued output is non-interactive. When we treat a general nonlinear specification (as proposed in~\citet{qin2018verification}) as an additional nonlinear layer $\x{(L+1)} = F(\x{(0)}, \x{(L)})$, this layer is automatically non-interactive. This nice property ensures that our framework can deal with very general specifications. 

The optimal convex relaxation of a non-interactive layer has a simple analytic form as below.
\begin{proposition}\label{prop:hull_layer}
If the layer $\sigma^{(l)}:[\lwz{(l)}, \upz{(l)} ] \to \R^{n^{(l+1)}}$ is non-interactive, we have 
\footnotesize{
\begin{equation*}
\begin{aligned}
    \big\{(\z{(l)}, \x{(l+1)}):{}& \underline{\sigma}_{\text{opt}}^{(l)}(\z{(l)}) \le \x{(l+1)} \le \overline{\sigma}_{\text{opt}}^{(l)}(\z{(l)}) \big\} =\\
    &\overline{\text{conv}}\big(\{(\z{(l)}, \x{(l+1)}): \x{(l+1)}=\sigma^{(l)}(\z{(l)}),{}&\lwz{(l)} \le \z{(l)} \le \upz{(l)}\}\big),
\end{aligned}
\end{equation*}
}
where $\overline{\text{conv}}$ denotes the closed convex hull, and vector-valued functions $\underline{\sigma}_{\text{opt}}^{(l)}(\z{})$ and $\overline{\sigma}_{\text{opt}}^{(l)}(\z{})$ are defined in \eqref{def:underoversigma} for each output entry.
\end{proposition}

Thanks to its non-interaction, Proposition~\ref{prop:hull_layer} is a direct consequence of item 2 in Proposition~\ref{prop:hull}.

\section{Convex relaxations not included in Problem~\eqref{eq:rv_cmu}.}
\label{sec:outOfScopeRelaxations}
We emphasize that by \textit{optimal}, we mean the optimal convex relaxation of the \textit{single} nonlinear constraint $\x{(l+1)}=\sigma^{(l)}(\z{(l)})$ (see Proposition~\eqref{prop:hull_layer}) instead of the optimal convex relaxation of the nonconvex feasible set of the original problem~\eqref{eq:rvbnds}. In fact, for neural networks with more than two hidden layers ($L \ge 2$), the optimal convex relaxation of the nonconvex feasible set of problem~\eqref{eq:rvbnds} is a strict subset of the feasible set of problem~\eqref{eq:rv_cmu}, even with the tightest bounds $(\lwz{[L]}, \upz{[L]})$ and the optimal choice of $\underline{\sigma}_{\text{opt}}^{(l)}(\z{})$ and $\overline{\sigma}_{\text{opt}}^{(l)}(\z{})$ in \eqref{def:underoversigma}. It is possible to obtain other (maybe tighter) convex relaxations \cite{anderson2018strong}, but it comes with more assumptions on the nonlinear layers and more complex convex constraints. 

For example, \citet{raghunathan2018semidefinite} rewrites the ReLU nonlinearity as a quadratic constraint, and then proposes a semidefinite programming (SDP) relaxation for the resulting quadratic optimization problem. Problem~\eqref{eq:rv_cmu} does not cover this SDP-relaxation. Sometimes Problem~\eqref{eq:rv_cmu} provides tighter relaxation than the SDP-relaxation, e.g., the case when there is only one neuron in a layer, while sometimes the SDP-relaxation provides tighter relaxation than Problem~\eqref{eq:rv_cmu}, e.g., the examples provided in \citet{raghunathan2018semidefinite}. The SDP-relaxation currently only works for ReLU nonlinearity. It is not clear to us how to extend the SDP-relaxed verifier to general nonlinearities. On the other hand,  Problem~\eqref{eq:rv_cmu} can handle any non-interactive nonlinear layer and any nonlinear specification. 
\section{Greedily solving the primal with linear bounds.}
\label{sec:greedy}
In this section, we show how to greedily solve~\eqref{eq:rv_cmu} by over-relaxing the problem to give a lower bound directly, and discuss the relationships between algorithms in Figure~\ref{fig:main}, especially for the algorithms in primal view.

\paragraph{Relaxing the ReLU neurons.} We start with giving exactly one linear upper bound and exactly one linear lower bound for each activation function in~\eqref{eq:rv_cmu}:
\begin{equation}\label{eq:rv_cmu_approx}
\begin{aligned}
    \min_{(\x{[L+1]}, \z{[L]}) \in \CalD} \quad  &c^{\top} \x{(L)} + c_0 \\
    \text{s.t.}\quad &\z{(l)} = \W{(l)} \x{(l)} + \bias{(l)}, \quad l \in \{0, \cdots, L-1\}, \\
    \underline{a}^{(l)} z^{(l)} + \underline{b}^{(l)} \le &x^{(l+1)} \le \overline{a}^{(l)} z^{(l)} + \overline{b}^{(l)}, \quad l \in \{0, \cdots, L-1\},
\end{aligned}
\end{equation}
Typically, the selection of $\underline{a}^{(l)}$, $\overline{a}^{(l)}$,  $\overline{b}^{(l)}$, $\underline{b}^{(l)}$ can depend on $\overline{z}^{(l)}$ and $\underline{z}^{(l)}$ to minimize the error between the upper/lower bound and the activation function. For element-wise activation functions, the linear upper and lower bounds are usually also element-wise.
For example, for an unstable ReLU neuron with $\overline{z}_i^{(l)} > 0$ and $\underline{z}_i^{(l)} < 0$, one upper bound is $x_i^{(l+1)} \le \frac{\overline{z}_i^{(l)}}{\overline{z}_i^{(l)} - \underline{z}_i^{(l)}} z_i^{(l)} - \frac{\overline{z}_i^{(l)}\underline{z}_i^{(l)}}{\overline{z}_i^{(l)} - \underline{z}_i^{(l)}}$. According to Proposition~\ref{prop:hull}, this is the optimal convex relaxation for the upper bound. For the lower bound, the optimal convex relaxation $(x_i^{(l+1)} \geq z_i^{(l)}) \cap (x_i^{(l+1)} \geq 0)$ is not achievable as one linear function; we use any over-relaxed bounds $x_i^{(l+1)} \geq \underline{a}_i^{(l)} z_i^{(l)}$ with $0 \leq \underline{a}_i^{(l)} \leq 1$ as the lower bound. This perspective covers Fast-Lin~\cite{weng2018towards}, DeepZ~\cite{singh2018fast} and Neurify~\cite{wang2018efficient}, where the lower bound is fixed as $\underline{a}_i^{(l)} = \overline{a}_i^{(l)} = \frac{\overline{z}_i^{(l)}}{\overline{z}_i^{(l)} - \underline{z}_i^{(l)}}$; this is referred as a ``zonotope'' relaxation in AI$^2$~\cite{gehr2018ai} and DeepZ. AI$^2$ is a general technique of using ``abstract transformers'' (sound relaxations of neural network elements) to verify neural networks, but it uses suboptimal relaxations for ReLU non-linearity; DeepZ further refines the transformers for ReLU and significantly outperforms AI$^2$~\cite{singh2018fast}.
Other activation functions can be linearly bounded as discussed in CROWN \cite{zhang2018crown}, DeepZ and DeepPoly \cite{singh2019abstract}; CROWN and DeepPoly are also more general and do not require $\underline{a}_i^{(l)} = \overline{a}_i^{(l)}$ to allow a more flexible selection of bounds.

\paragraph{Deriving the Greedy Primal Method.}
Assuming we have obtained the linear upper and lower bounds for $x^{(l+1)}$ with respect to $z^{(l)}$, $\underline{z}_i^{(l+1)}$ can be formed greedily as a linear combination of these linear bounds: we greedily select the upper bound 
$x_i^{(l+1)} \leq \underline{a}_i^{(l)} z_i^{(l)} + \underline{b}_i^{(l)}$ 
when $\W{(l+1)}_{i,k}$ is negative, and select the lower bound  $x_i^{(l+1)} \geq \overline{a}_i^{(l)} z_i^{(l)} + \overline{b}_i^{(l)}$ 
otherwise. This bound reflects the worst case scenario without considering any other neurons:
\begin{align}
z_i^{(l+1)} \geq \underline{z}_i^{(l+1)} \coloneqq \underline{A}_{i,:}^{(l)} z^{(l)} + \underline{b}_i^{\prime(l)}
\end{align}
where matrix $
\underline{A}_{i,k}^{(l)} = \begin{cases}
    \W{(l+1)}_{i,k} \overline{a}_k^{(l)}, \W{(l+1)}_{i,k} < 0 \\
    \W{(l+1)}_{i,k} \underline{a}_k^{(l)}, \W{(l+1)}_{i,k} \geq 0
    \end{cases}$
reflects the chosen upper or lower bound based on the sign of $\W{(l+1)}_{i,k}$,
and vector $\underline{b}_i^{\prime(l)} = \sum_{k, \W{(l+1)}_{i,k} \geq 0} \W{(l+1)}_{i,k} \underline{b}_k^{(l)} + \sum_{k, \W{(l+1)}_{i,k} < 0} \W{(l+1)}_{i,k} \overline{b}_k^{(l)} + b_i^{(l)}$.
The lower bound $\overline{z}_i^{(l+1)}$ can also be formed similarly. Eventually, we get one linear upper bound and one linear lower bound for ${z}_i^{(l+1)}$, written as:
\begin{align}
\label{eq:layer_l_greedy}
\underline{A}_{i,:}^{(l)} z^{(l)} + \underline{b}_i^{\prime(l)} \leq z_i^{(l+1)} \leq \overline{A}_{i,:}^{(l)} z^{(l)} + \overline{b}_i^{\prime(l)}
\end{align}
A sharp-eyed reader can notice that it is possible to also get a similar bound for each component of $z^{(l)}$ and plug it into~\eqref{eq:layer_l_greedy}, thus obtaining a linear upper bound and a linear lower bound for $z_i^{(l+1)}$ with respect to $z^{(l-1)}$. To do this, we first substitute $\z{(l)} = \W{(l)} \x{(l)} + \bias{(l)}$ into Eq.~\eqref{eq:layer_l_greedy}, obtaining 
\begin{align*}
\underline{A}_{i,:}^{(l)} (\W{(l)} \x{(l)} + \bias{(l)}) + \underline{b}_i^{\prime(l)} \leq z_i^{(l+1)} \leq \overline{A}_{i,:}^{(l)} (\W{(l)} \x{(l)} + \bias{(l)}) + \overline{b}_i^{\prime(l)}
\end{align*}
Applying the bounds on $\x{(l)}$ with respect to $z^{(l-1)}$, and using a similar technique as we did above to obtain~\eqref{eq:layer_l_greedy}, we get linear upper and lower bounds for $z_i^{(l+1)}$ with respect to $z^{(l-1)}$ in the following form:
\begin{align}
\label{eq:layer_lminus1_greedy}
\underline{A}_{i,:}^{(l-1)} z^{(l-1)} + \underline{b}_i^{\prime(l-1)} \leq z_i^{(l+1)} \leq \overline{A}_{i,:}^{(l-1)} z^{(l-1)} + \overline{b}_i^{\prime(l-1)}
\end{align}
where $\underline{b}_i^{\prime(l-1)}$ and $\overline{b}_i^{\prime(l-1)}$ collect all bias terms in the substitution process. Caution has to be taken when forming $\underline{A}_{i,:}^{(l-1)}$ and $\overline{A}_{i,:}^{(l-1)}$, as we need to choose $\underline{a}_k^{(l-1)}$ or $\overline{a}_k^{(l-1)}$ based on the sign of $\underline{A}^{(l)}_{i,:} \W{(l)}_{:,k}$, since the coefficients before each inequality now become $\underline{A}^{(l)}_{i,:} \W{(l)}$ rather than just $\W{(l)}$:
\[
\underline{A}_{i,k}^{(l-1)} = \begin{cases}
    \underline{A}^{(l)}_{i,:} \W{(l)}_{:,k} \overline{a}_k^{(l-1)}, \quad \underline{A}^{(l)}_{i,:} \W{(l)}_{:,k} < 0 \\
    \underline{A}^{(l)}_{i,:} \W{(l)}_{:,k} \underline{a}_k^{(l-1)}, \quad \underline{A}^{(l)}_{i,:} \W{(l)}_{:,k} \geq 0
    \end{cases}
\]

An eagle-eyed reader can notice that we can continue this process until we have reached $z^{(0)}$, and obtain the following linear bounds:
\begin{equation}
\label{eq:layer_0_greedy0}
\underline{A}_{i,:}^{(0)} z^{(0)} + \underline{b}_i^{\prime(0)} \leq z_i^{(l+1)} \leq \overline{A}_{i,:}^{(0)} z^{(0)} + \overline{b}_i^{\prime(0)}
\end{equation}
where $\underline{A}_{i,:}^{(0)}$, $\overline{A}_{i,:}^{(0)}$, $\underline{b}_i^{\prime(0)}$ and $\overline{b}_i^{\prime(0)}$ can be formed similarly as above. Substituting $\z{(0)} = \W{(0)} \x{(0)} + \bias{(0)}$ ($\x{(0)}=x$ is the input of the neural network) simply yields:
\begin{equation}
\label{eq:layer_0_greedy}
\underline{A}_{i,:} x + \underline{b_i^\prime} \leq z_i^{(l+1)} \leq \overline{A}_{i,:} x + \overline{b_i^\prime}
\end{equation}
where $\underline{A}_{i,:}=\underline{A}_{i,:}^{(0)} \W{(0)}$, $\overline{A}_{i,:}=\overline{A}_{i,:}^{(0)} \W{(0)}$ captures the products of $\W{}$ of all layers and the chosen $\underline{a}_k^{(l)}$ or $\overline{a}_k^{(l)}$ for each layer; $\underline{b^\prime}, \overline{b^\prime}$ collects all bias terms (we refer the readers to Theorem 3.2 in~\citet{zhang2018crown} for the exact form of $\underline{A}, \overline{A}, \underline{b^\prime}, \overline{b^\prime}$). This procedure beautifully works  as the linear combination of linear bounds are still linear bounds. Eq. \eqref{eq:layer_0_greedy} is a remarkable result, as the output of a non-linear function (neural network) has been directly bounded linearly for all $x$ close to $x^{\text{nom}}$. This allows us to immediately give upper and lower bounds of $z_i^{(l+1)}$ by considering the worst case $x \in \setS_{in}(x^{\text{nom}})$. When the set is an $L_\infty$ normed ball, this is obvious,
\begin{equation}
\label{eq:layer_greedy_dualnorm}
-\epsilon \| \underline{A}_{i,:} \|_1 + \underline{A}_{i,:} x^{\text{nom}} + \underline{b_i^\prime} \leq z_i^{(l+1)} \leq \epsilon \| \overline{A}_{i,:} \|_1 + \overline{A}_{i,:} x^{\text{nom}} + \overline{b_i^\prime},
\end{equation}

The entire bound propagation process does not involve any LP solver, so it is efficient and can scale to quite large networks. The final objective $c^{\top} \x{(L)} + c_0$ can be treated as an additional linear layer after $z^{(L-1)}$. Because to form the bounds for $z^{(L-1)}$ we need to compute bounds for all $z^{(l)}, l \in [L-1]$ beforehand, each in $O(l)$ time, the time complexity of this method is quadratic in $L$.

\paragraph{Connections Between Existing Methods.} 
For each neuron, the selection of linear bounds are completely independent; this allows further improvements in this greedy algorithm. For example, the selection of $\underline{a}_k^{(l)}$ can depend on $\overline{z}_k^{(l)}$ and $\underline{z}_k^{(l)}$ to adaptively minimize the error between the lower bound and ReLU function. CROWN~\cite{zhang2018crown} and DeepPoly~\cite{singh2019abstract} used this strategy to achieve tighter verification results than Fast-Lin~\cite{wang2018efficient}, DeepZ~\cite{singh2018fast} and Neurify~\cite{wang2018efficient}. Note that although the bound propagation techniques used in these works can be viewed as using different linear relaxations and solve the primal problem greedily in our framework, each of the works has some unique features. For example, DeepPoly~\cite{singh2019abstract} and DeepZ~\cite{singh2018fast} carefully consider floating-point rounding during the computation; \citet{weng2018towards} gives a theoretical hardness proof based on a reduction from the set-cover problem; Neurify~\cite{wang2018efficient} combines the relaxed bound with a branch-and-bound search to give concrete instances of adversarial example if they exist, and also uses the bound for training~\cite{wang2018mixtrain}.

One the other hand, instead of propagating the bounds of $z^{(l+1)}$ to $z^{(l-k)}$ as shown above, we can decouple layer $z^{(l-(k-1))}$ and $z^{(l-k)}$ entirely: suppose we have obtained concrete upper and lower bounds for $z^{(l-(k-1))}$, we can treat $z^{(l-(k-1))}$ as the input layer and only consider a $k$-layer network to compute the bounds of $z^{(l+1)}$. This leads to interval bounds propagation (IBP)~\citep{gowal2018effectiveness} ($k=1$) and ``Box Domain'' ~\citep{mirman2018differentiable} which gives even looser bounds, but its computation cost is also greatly reduced.

The greedy algorithm in primal space is also closely connected to the greedy algorithm in dual space; the dual of~\eqref{eq:rv_cmu_approx} will recover a dual formulation with solution~\eqref{eq:dual_suboptimal_mu}, and the closed from solution are related to the chosen slopes $\overline{a}_i^{(l)}$ and $\underline{a}_i^{(l)}$. This explains the equivalence of Fast-Lin and the greedy algorithm to solve the dual problem presented in Algorithm 1 of ~\cite{wong2018provable}.

\paragraph{The Relationships Between Algorithms in Figure \ref{fig:main}.} Based on the above discussions, we now revisit Figure \ref{fig:main}, and discuss each arrow in this figure on the ``primal view'' side. 

First of all, the arrow from ``Optimal layer-wise convex relaxation'' to CROWN~\cite{zhang2018crown} trivially holds since CROWN is a greedy algorithm to solve LP relaxations (problem $\mathcal{C}$ plus Eq.~\eqref{eq:linearbounds}), which can be included in the convex relaxation framework. Additionally, CROWN is proposed as a more general variant of Fast-Lin~\cite{weng2018towards}. In Fast-Lin, the linear relaxation uses the same slope for the upper and lower bounds; in CROWN, the slopes can be different. In other words, in Eq.~\eqref{eq:linearbounds}, $\overline{a}^{(l)} = \underline{a}^{(l)}$ for Fast-Lin but this is not a requirement for CROWN.

Despite originating from different perspectives,  DeepZ~\citep{singh2018fast} and Fast-Lin~\citep{weng2018towards} share the same relaxations and give numerically identical bounds; so do DeepPoly~\citep{singh2019abstract} and CROWN. This can be observed by translating between the different notations of these papers. Particularly, \citet{Singh2019robustness} commented ``DeepZ has \textit{the same} precision as Fast-Lin and DeepPoly has \textit{the same} precision as CROWN'', although they have several implementation differences.

The arrows from ``LP-Relaxed Dual'' to CROWN and ``LP-Relaxed Dual'' to Fast-Lin come from equation~\eqref{eq:linearbounds}, where CROWN and Fast-Lin use one linear upper bound and one linear lower bound as constraints instead of the general convex constraints in~\eqref{eq:rv_cmu}, so the problem $\mathcal{C}$ becomes a special case of an LP-relaxed problem.

Fast-Lin and Neurify~\cite{wang2018efficient} use the same relaxation for ReLU neurons (and unlike other works, these two only deal with ReLU activation functions). This can be observed by comparing Figure 3 in \citet{wang2018efficient} and Figure 1 in \citet{weng2018towards}: the choice of the slopes $\underline{a}^{(l)}$ and $\overline{a}^{(l)}$ are the same. Numerically, both algorithms also produce the same results, but Neurify additionally implements a branch-and-bound search for solving the exact verification problem with the relaxation based bounds.
\clearpage
\section{Strong duality for Problem \eqref{eq:rv_cmu}: $p^*_{\CalC} = d^*_{\CalC}$}
\label{app:strongdual}
Consider the following perturbed version of problem~\eqref{eq:rv_cmu}:
\begin{equation}\label{eq:rv_cmu_perturb}
\begin{aligned}
    \tilde{p}^*_{\CalC}:= \min_{(\x{[L+1]}, \z{[L]}) \in \CalD} \quad  &c^{\top} \x{(L)} + c_0 \\
    \text{s.t.}\quad  &\z{(l)} = \W{(l)} \x{(l)} + \bias{(l)} + v^{(l)}, l \in [L], \\
    & \underline{\sigma}^{(l)}(\z{(l)}) - \underline{u}^{(l)} \le \x{(l+1)} \le \overline{\sigma}^{(l)}(\z{(l)}) + \overline{u}^{(l)}, l \in [L].
\end{aligned}
\end{equation}
 
\begin{lemma}\label{lem:perturb0}
We assume that for each $l \in [L]$, both $\underline{\sigma}^{(l)}$ and $\overline{\sigma}^{(l)}$ have a finite Lipschitz constant in the domain $[\lwz{(l)}, \upz{(l)}]$. 
There exists a positive constant $C_{\CalC}>0$ such that for any perturbations $(\underline{u}^{[L]}, \overline{u}^{[L]}, v^{[L]})$, we have
\begin{equation}\label{eq:pperturb0}
     \tilde{p}^*_{\CalC} \ge {p}^*_{\CalC} - C_{\CalC} \|(\underline{u}^{[L]}, \overline{u}^{[L]}, v^{[L]})\|_2
\end{equation}
\end{lemma}
Lemma~\ref{lem:perturb0} shows that the optimal value of the perturbed problem, i.e., $\tilde{p}^*_{\CalC}$, ``smoothly'' changes  with the perturbations. We delay the proof of Lemma~\ref{lem:perturb0} in Section~\ref{app:pperturb}. Combined with convexity, this ensures the strong duality for problem \eqref{eq:rv_cmu}.

\begin{figure}[t]
  \centering
  \includegraphics[width=0.48\textwidth]{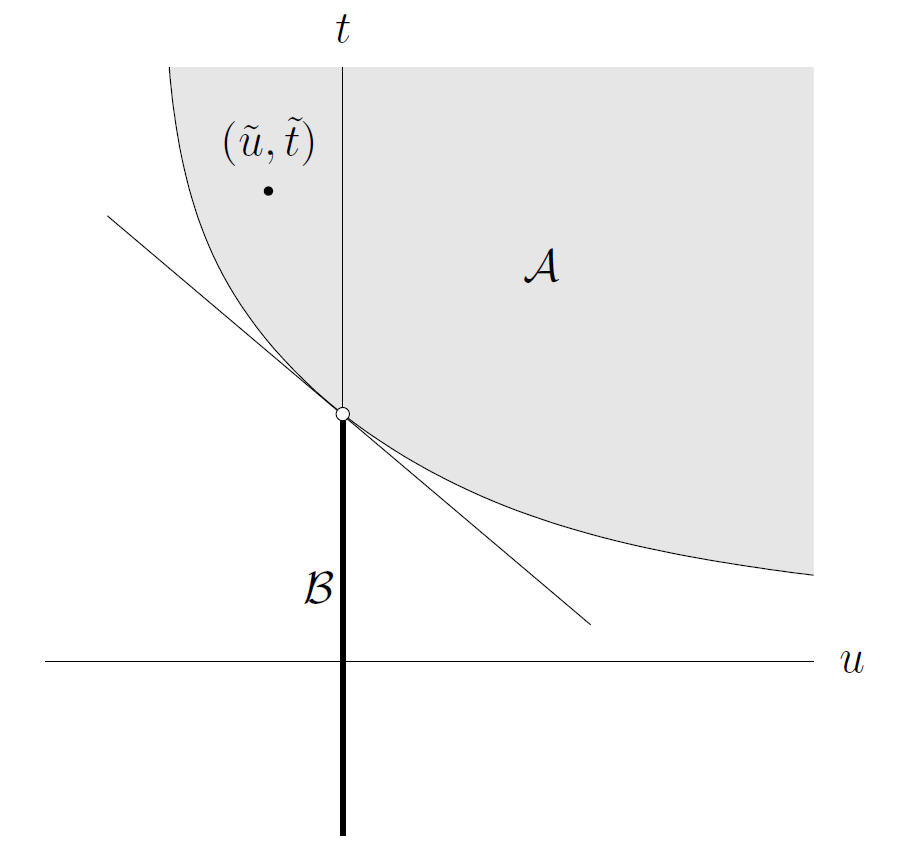}
  \includegraphics[width=0.48\textwidth]{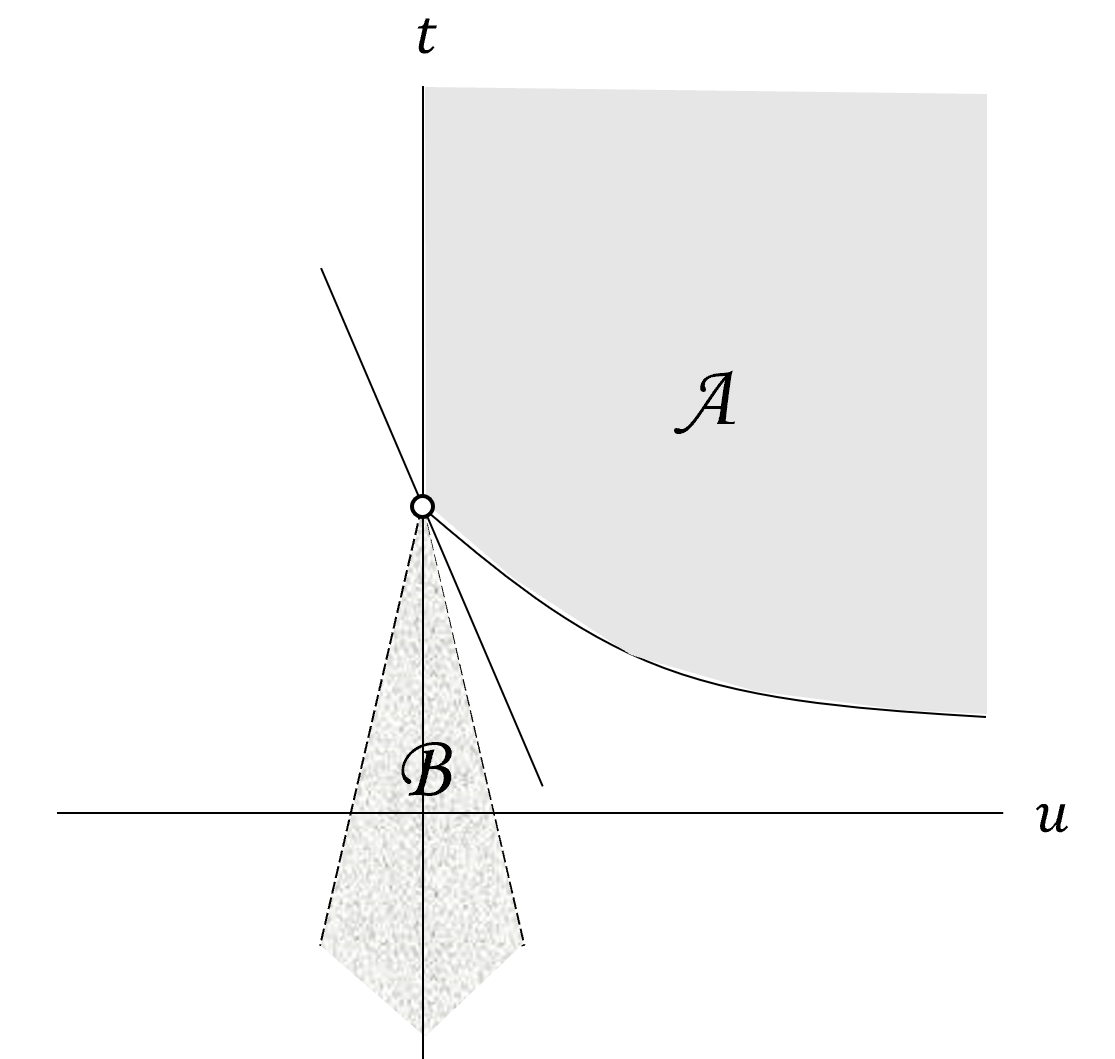}
  \caption{Illustration of strong duality proof for a convex problem. Left: proof under Slater's condition (picture from \cite{boyd2004convex} Section 5.3.2). Right: our proof. In both settings, the set $\CalA$ and $\CalB$ are convex and do not intersect, so they can be separated by a hyperplane. Slater's condition (Left) assumes existence of a point that strictly satisfies the inequality constraints, i.e., $(\tilde{u},\tilde{t})$ in Figure 5(left), and thus any separating hyperplane must be nonvertical. In our setting (Right), we take $B$ to be a much larger set (thanks to Lemma~\ref{lem:perturb0}), and thus any separating hyperplane must be nonvertical. Therefore, we can get strong duality without the Slater's condition.}
\label{fig:strongduality}
\end{figure}

\begin{proof}[Proof of Theorem~\ref{thm:strongdual}]
The structure of the proof follows the proof of strong duality given the Slater's condition in \cite{boyd2004convex} (Section 5.3.2). However, we do not assume the Slater's condition in our result here. Let's define
\begin{equation*}
\begin{aligned}
\CalA = \big\{(&\underline{u}^{[L]}, \overline{u}^{[L]}, v^{[L]}, t): \exists (\x{[L+1]}, \z{[L]}) \in \CalD, \underline{\sigma}^{(l)}(\z{(l)}) - \underline{u}^{(l)} \le \x{(l+1)} \le \overline{\sigma}^{(l)}(\z{(l)}) + \overline{u}^{(l)}, \\
& \z{(l)} = \W{(l)} \x{(l)} + \bias{(l)} + v^{(l)}, \forall l \in [L],  c^{\top}\x{(L)} + c_0 \le t \big\},
\end{aligned}
\end{equation*}
and 
\begin{equation*}
\begin{aligned}
\CalB = &\big\{(\underline{u}^{[L]}, \overline{u}^{[L]}, v^{[L]}, t): & t < {p}^*_{\CalC} - C_{\CalC} \|(\underline{u}^{[L]}, \overline{u}^{[L]}, v^{[L]})\|_2 \big\} .
\end{aligned}
\end{equation*}
$\CalA$ is convex because the problem \eqref{eq:rv_cmu_perturb} is convex. $\CalB$ is convex by definition. The sets $\CalA$ and $\CalB$ do not intersect, as illustrated in Figure~\ref{fig:strongduality}. To see this, suppose $(\underline{u}^{[L]}, \overline{u}^{[L]}, v^{[L]}, t) \in \CalA \cap \CalB$. Since $(\underline{u}^{[L]}, \overline{u}^{[L]}, v^{[L]}, t) \in \CalB$, we have $t < {p}^*_{\CalC} - C_{\CalC} \|(\underline{u}^{[L]}, \overline{u}^{[L]}, v^{[L]})\|_2$. Since $(\underline{u}^{[L]}, \overline{u}^{[L]}, v^{[L]}, t) \in \CalA$, there exists $(\x{[L+1]}, \z{[L]}) \in \CalD$ such that it satisfies the constraints in problem \eqref{eq:rv_cmu_perturb}, and $t \ge c^{\top}\x{(L)} + c_0 \ge \tilde{p}^*_{\CalC} \ge {p}^*_{\CalC} - C_{\CalC} \|(\underline{u}^{[L]}, \overline{u}^{[L]}, v^{[L]})\|_2$, where the last inequality comes from \eqref{eq:pperturb0}. This is a contradiction!  

By the separating hyperplane theorem, there exists $(\underline{\lambda}^{[L]}, \overline{\lambda}^{[L]}, \mu^{[L]}, \nu) \neq 0 $ and $\alpha$ such that
\begin{equation}\label{eq:sepA}
\begin{aligned}
    &(\underline{u}^{[L]}, \overline{u}^{[L]}, v^{[L]}, t) \in \CalA \Rightarrow \underline{\lambda}^{[L]{\top}} \underline{u}^{[L]} + \underline{\lambda}^{[L]{\top}}\overline{u}^{[L]} + \mu^{[L]{\top}} v^{[L]} + \nu t \ge \alpha,
\end{aligned}
\end{equation}
and 
\begin{equation}\label{eq:sepB}
\begin{aligned}
    &(\underline{u}^{[L]}, \overline{u}^{[L]}, v^{[L]}, t) \in \CalB \Rightarrow \underline{\lambda}^{[L]{\top}} \underline{u}^{[L]} + \underline{\lambda}^{[L]{\top}}\overline{u}^{[L]} + \mu^{[L]{\top}} v^{[L]} + \nu t \le \alpha,
\end{aligned}
\end{equation}
From \eqref{eq:sepA}, we conclude that $\underline{\lambda}^{[L]} \ge 0$, $\overline{\lambda}^{[L]} \ge 0$ and $\nu \ge 0$. Otherwise, $\underline{\lambda}^{[L]{\top}} \underline{u}^{[L]} + \underline{\lambda}^{[L]{\top}}\overline{u}^{[L]} + \nu t$ is unbounded from below over $\CalA$, contradicting \eqref{eq:sepA}. Since $(0,0,0,t) \in \CalB$ for any $t < {p}^*_{\CalC}$, we have $\nu t \le \alpha$ for any $t < {p}^*_{\CalC}$ thanks to \eqref{eq:sepB}, and thus $\nu {p}^*_{\CalC} \le \alpha$. Together with \eqref{eq:sepA}, we conclude that for any $(\x{[L+1]}, \z{[L]}) \in \CalD$, 
\begin{equation}\label{eq:rv_cmu_dual_sep}
\begin{aligned}
     \nu (c^{\top} \x{(L)} + c_0)  &+ \sum_{l=0}^{L-1} \mu^{(l){\top}} ( \z{(l)} - \W{(l)} \x{(l)} - \bias{(l)} ) + \sum_{l=0}^{L-1} \underline{\lambda}^{(l){\top}} (\underline{\sigma}^{(l)}(\z{(l)}) - \x{(l+1)}) \\
    &+ \sum_{l=0}^{L-1} \overline{\lambda}^{(l){\top}} ( \x{(l+1)}-\overline{\sigma}^{(l)}(\z{(l)}) ) \ge \alpha \ge \nu {p}^*_{\CalC}.
\end{aligned}
\end{equation}
Assume that $\nu > 0$. In that case, we can divide \eqref{eq:rv_cmu_dual_sep} by $\nu$ to obtain
\begin{equation*}
    L(\x{[L+1]}, \z{[L]}, \underline{\lambda}^{[L]}/\nu, \overline{\lambda}^{[L]}/\nu, \mu^{[L]}/\nu) \ge {p}^*_{\CalC}
\end{equation*}
for all $(\x{[L+1]}, \z{[L]}) \in \CalD$, where $L(\cdot)$, defined in \eqref{eq:rv_cmu_dual}, is the Lagrangian of \eqref{eq:rv_cmu}. Minimizing over $(\x{[L+1]}, \z{[L]}) \in \CalD$, we obtain $g_{\CalC}(\mu^{[L]}/\nu, \underline{\lambda}^{[L]}/\nu, \overline{\lambda}^{[L]}/\nu) \ge {p}^*_{\CalC}$. By weak duality, we have $g_{\CalC}(\mu^{[L]}/\nu, \underline{\lambda}^{[L]}/\nu, \overline{\lambda}^{[L]}/\nu) \le {p}^*_{\CalC}$, so in fact $g_{\CalC}(\mu^{[L]}/\nu, \underline{\lambda}^{[L]}/\nu, \overline{\lambda}^{[L]}/\nu) = {p}^*_{\CalC}$. This shows that strong duality holds, and that the dual optimum is attained, at least in the case when $\nu > 0$. 

Now we consider the case $\nu = 0$. From \eqref{eq:rv_cmu_dual_sep}, we conclude that for any $(\x{[L+1]}, \z{[L]}) \in \CalD$, 
\begin{equation}\label{eq:rv_cmu_dual_sep2}
\begin{aligned}
    & \sum_{l=0}^{L-1} \mu^{(l){\top}} ( \z{(l)} - \W{(l)} \x{(l)} - \bias{(l)} ) + \sum_{l=0}^{L-1} \underline{\lambda}^{(l){\top}} ( \underline{\sigma}^{(l)}(\z{(l)}) - \x{(l+1)} ) \\
    &+ \sum_{l=0}^{L-1} \overline{\lambda}^{(l){\top}} ( \x{(l+1)} - \overline{\sigma}^{(l)}(\z{(l)}) ) \ge \alpha \ge 0.
\end{aligned}
\end{equation}
Taking any feasible point of problem \eqref{eq:rv_cmu}, i.e., $(\x{[L+1]}, \z{[L]}) \in \setS_{\CalC}$ and combining with $\underline{\lambda}^{[L]} \ge 0, \overline{\lambda}^{[L]} \ge 0$, we know that the left-hand-side of \eqref{eq:rv_cmu_dual_sep2} is non-positive, and thus $\alpha = 0$. Then from \eqref{eq:sepB}, we conclude that for any $t \in \R$
\begin{equation*}
\begin{aligned}
    & \|(\underline{u}^{[L]}, \overline{u}^{[L]}, v^{[L]})\|_2 < \frac{{p}^*_{\CalC} - t}{C_{\CalC}} \Rightarrow \underline{\lambda}^{[L]{\top}} \underline{u}^{[L]} + \underline{\lambda}^{[L]{\top}}\overline{u}^{[L]} + \mu^{[L]{\top}} v^{[L]} \le 0,
\end{aligned}
\end{equation*}
which can only be possible when $(\underline{\lambda}^{[L]}, \overline{\lambda}^{[L]}, \mu^{[L]}) = 0 $. Combined with $\nu=0$, this contradicts with $(\underline{\lambda}^{[L]}, \overline{\lambda}^{[L]}, \mu^{[L]}, \nu) \neq 0 $, and thus $\nu$ cannot be 0.
\end{proof}

\subsection{Cases where the Slater's condition fails but strong duality holds true by Theorem~\ref{thm:strongdual}}
\label{app:slatervsours}
We emphasize that Theorem~\ref{thm:strongdual} guarantees the strong duality for any pre-specified activation bounds $[\underline{z}^{(l)}, \overline{z}^{(l)}]$ that can be either loose or tight, and for any $\underline{\sigma}^{(l)}$ and $\overline{\sigma}^{(l)}$ that have a finite Lipschitz constant in the domain $[\lwz{(l)}, \upz{(l)}]$. There are several important cases when Slater’s condition does not hold but strong duality holds true by Theorem~\ref{thm:strongdual}.

The first typical scenario is when the pre-specified activation bounds $[\underline{z}^{(l)}, \overline{z}^{(l)}]$ is loose and all the feasible activations $z^{(l)}$ are on the boundary. Let’s consider a simple one-layer neural network: 
\begin{equation*}
\begin{aligned}
&x^{(0)} \in \setS_{in}(x^{\text{nom}}), \\
&z^{(0)} = \W{(0)} x^{(0)} + b^{(0)}, \quad z^{(0)} \in [ \underline{z}^{(0)}, \overline{z}^{(0)} ], \\
&\underline{ReLU}(z^{(0)}) \le x^{(1)} \le  \overline{ReLU}(z^{(0)}).
\end{aligned}
\end{equation*}
Suppose that $\W{(0)} = 0$, $b^{(0)} = -1$, $\underline{z}^{(0)} = -1$ and $\overline{z}^{(0)} = 1$. Then $z^{(0)}$ can only be -1, and $\underline{ReLU}(z^0) = \overline{ReLU}(z^0) = 0$, and thus there does not exist $x^{(1)}$ such that $\underline{ReLU}(z^{(0)}) < x^{(1)} <  \overline{ReLU}(z^{(0)})$. In general, orthogonality between $x^{(l)}$ and $\text{span}(\W{(l)})$ easily leads to degeneracy of $z^{(l)}$, which can result in the failure of the Slater’s condition.

The second typical scenario is when the pre-specified activation bounds in later layers, e.g., $z^{(1)} \in [ \underline{z}^{(1)}, \overline{z}^{(1)} ]$, forces all feasible points in previous layers, e.g., $x^{(1)}$, to be on the boundary. This degenerate case may occur when one takes the branch-and-bound strategy to split unstable neurons. Let’s consider a simple two-layer neural network: 
\begin{equation*}
\begin{aligned}
&x^{(0)} \in \setS_{in}(x^{\text{nom}}):=[-1,1], \\
&z^{(0)} = x^{(0)}, \quad z^{(0)} \in [ -1, 1 ], \\
&\underline{ReLU}(z^{(0)}) \le x^{(1)} \le  \overline{ReLU}(z^{(0)}), \\
&z^{(1)} = x^{(1)} - 1, \quad z^{(1)} \in [ 0, 1 ], \\
&\underline{ReLU}(z^{(1)}) \le x^{(2)} \le  \overline{ReLU}(z^{(1)}).
\end{aligned}
\end{equation*}
Due to the pre-specified bound $z^{(1)} \in [0,1]$, $x^{(1)}$ can only take value 1, which is on the boundary of the nonlinear constraint $\underline{ReLU}(z^{(0)}) \le x^{(1)} \le  \overline{ReLU}(z^{(0)})$. This leads to failure of the Slater’s condition.

After all, there are many edge cases that the Slater’s condition does not cover to prove Theorem~\ref{thm:strongdual}. Therefore, we would like to take a novel approach, utilizing the Lipschitz continuity of problem~\eqref{eq:rv_cmu}, to prove the strong duality without the Slater's condition.

\subsection{Proof of Lemma~\ref{lem:perturb0}}
\label{app:pperturb}
Although the proof seems to be long, it is an elementary perturbation analysis for problem~\eqref{eq:rv_cmu}. We write down every detail so that one can easily check its correctness. 

\begin{proof}[Proof of Lemma~\ref{lem:perturb0}]
When problem~\eqref{eq:rv_cmu_perturb} is infeasible, i.e., $\tidX{(L)} = \emptyset$, $\tilde{p}^*_{\CalC} = + \infty$ and \eqref{eq:pperturb0} naturally holds true. In the following, we prove \eqref{eq:pperturb0} when problem~\eqref{eq:rv_cmu_perturb} is feasible. 

In this case, we define $\X{(0)} = \tidX{(0)} = \setS_{in}(x^{\text{nom}})$, $\Z{(l)}$, $\tidZ{(l)}$, $\X{(l)}$ and $\tidX{(l)}$ recursively as follows:
\begin{align*}
    &\Z{(l)} = \{ \W{(l)} \x{(l)} + \bias{(l)} : \x{(l)} \in \X{(l)} \} \cap [\lwz{(l)}, \upz{(l)}], \\
    &\tidZ{(l)} = \{ \W{(l)} \tidx{(l)} + \bias{(l)} + v^{(l)} : \tidx{(l)} \in \tidX{(l)} \} \cap [\lwz{(l)}, \upz{(l)}], \\
    &\X{(l+1)} = \{\x{(l+1)}: \underline{\sigma}^{(l)}(\z{(l)}) \le \x{(l+1)} \le \overline{\sigma}^{(l)}(\z{(l)}),  \z{(l)}\in \Z{(l)} \}, \\
    &\tidX{(l+1)} = \{\tidx{(l+1)}: \underline{\sigma}^{(l)}(\tidz{(l)}) - \underline{u}^{(l)} \le \tidx{(l+1)} \le \overline{\sigma}^{(l)}(\tidz{(l)}) + \overline{u}^{(l)}, \tidz{(l)}\in \tidZ{(l)} \}.
\end{align*}
Intuitively, $\Z{(l)}$, $\tidZ{(l)}$, $\X{(l)}$ and $\tidX{(l)}$ are the set of activations that are achievable by the original problem~\eqref{eq:rv_cmu} and the perturbed problem~\eqref{eq:rv_cmu_perturb} given $\x{(0)} \in \setS_{in}(x^{\text{nom}})$ and $\z{(l)} \in [\lwz{(l)}, \upz{(l)}]$. Since both problems are feasible, all the sets above are non-empty.

In the first step, we prove that for every $l \in [L+1]$, there exist positive constants $C_x^{(l)}$ and $C_z^{(l)}$ such that
\begin{align}
    \sup_{\tidx{(l)} \in \tidX{(l)}} \dist(\tidx{(l)}, \X{(l)}) &\le C_x^{(l)} \|(\underline{u}^{[L]}, \overline{u}^{[L]}, v^{[l]})\|_2, \label{eq:xperturb} \\
    \sup_{\tidz{(l)} \in \tidZ{(l)}} \dist(\tidz{(l)}, \Z{(l)}) &\le C_z^{(l)} \|(\underline{u}^{[L]}, \overline{u}^{[L]}, v^{[l+1]})\|_2, \label{eq:zperturb}
\end{align}
where $\dist(s, \setS) \coloneqq \inf_{s' \in \setS} \|s - s'\|_2$. This means that the perturbation in the achievable activations are ``smooth".

Since $\X{(l)} = \tidX{(l)} = \setS_{in}(x^{\text{nom}})$, we have that \eqref{eq:xperturb} holds true for $l=0$ with $C_x^{(0)} = 0$. In the following, we use mathematical induction to prove \eqref{eq:zperturb} for $0\le l \le L-1$ and \eqref{eq:xperturb} for $1\le l \le L$.

First, suppose $\dist(\tidx{(l)}, \X{(l)}) \le C_x^{(l)} \|(\underline{u}^{[L]}, \overline{u}^{[L]}, v^{[l]})\|_2$ holds true for any $\tidx{(l)} \in \tidX{(l)}$. Then for any $\tidz{(l)} = \W{(l)} \tidx{(l)} + \bias{(l)} + v^{(l)} \in \tidZ{(l)}$, we have
\begin{equation*}
\begin{aligned}
    & \dist(\tidz{(l)}, \Z{(l)}) := \inf_{ \z{(l)} \in \Z{(l)} } \| \tidz{(l)} - \z{(l)} \|  \le \inf_{\x{(l)} \in \X{(l)}} \|\W{(l)} (\tidx{(l)} - \x{(l)}) + v^{(l)}\| \\
    & \le \inf_{\x{(l)} \in \X{(l)}} \|\W{(l)}\| \|\tidx{(l)} - \x{(l)}\| + \|v^{(l)}\|  = \|\W{(l)}\| \dist(\tidx{(l)}, \X{(l)}) + \|v^{(l)}\| \\
    & \le \|\W{(l)}\| C_x^{(l)} \|(\underline{u}^{[L]}, \overline{u}^{[L]}, v^{[l]})\| + \|v^{(l)}\| \le \left( (C_x^{(l)})^2 \|\W{(l)}\|^2 + 1 \right)^2 \|(\underline{u}^{[L]}, \overline{u}^{[L]}, v^{[l+1]})\|.
\end{aligned}
\end{equation*}
Therefore, \eqref{eq:zperturb} holds true with $C_z^{(l)} = \left( (C_x^{(l)})^2 \|\W{(l)}\|^2 + 1 \right)^2$. 

Then by definition, for any $\tidx{(l+1)} \in \tidX{(l+1)}$, there exists $\tidz{(l)} \in \tidZ{(l)}$ such that 
\begin{equation*}
\begin{aligned}
    \underline{\sigma}^{(l)}(\tidz{(l)}) - \underline{u}^{(l)} \le \tidx{(l+1)} \le \overline{\sigma}^{(l)}(\tidz{(l)}) + \overline{u}^{(l)}.
\end{aligned}
\end{equation*} 
By the induction assumption, there exists $\z{(l)} \in \Z{(l)}$ such that 
$$\dist(\tidz{(l)}, \z{(l)}) \le C_z^{(l)} \|(\underline{u}^{[L]}, \overline{u}^{[L]}, v^{[l+1]})\|_2.$$
Thus, we have
\begin{equation*}
\begin{aligned}
    \dist(\tidx{(l+1)}, \X{(l+1)}) &= \inf_{ \x{(l+1)} \in \X{(l+1)} } \| \tidx{(l+1)} - \x{(l+1)} \| \\
    & \le \inf \{\| \tidx{(l+1)} - \x{(l+1)} \| :  \underline{\sigma}^{(l)}(\z{(l)}) \le \x{(l+1)} \le \overline{\sigma}^{(l)}(\z{(l)}) \}.
\end{aligned}
\end{equation*}
We re-parametrize $\tidx{(l+1)}$ and $\x{(l+1)}$ as
\begin{equation*}
\begin{aligned}
    & \tidx{(l+1)} = \underline{\sigma}^{(l)}(\tidz{(l)}) + \tilde{t}, \quad \x{(l+1)} = \underline{\sigma}^{(l)}(\z{(l)}) + t,
\end{aligned}
\end{equation*}
where
\begin{equation*}
\begin{aligned}
    - \underline{u}^{(l)} \le \tilde{t} \le \Delta \sigma^{(l)}(\tidz{(l)}) + \overline{u}^{(l)}, \quad 0 \le t \le \Delta \sigma^{(l)}(\z{(l)}),\\
    \Delta \sigma^{(l)}(\tidz{(l)}) = \overline{\sigma}^{(l)}(\z{(l)}) - \underline{\sigma}^{(l)}(\z{(l)}).
\end{aligned}
\end{equation*}
It is easy to prove that if $\underline{\sigma}^{(l)}$ and $\overline{\sigma}^{(l)}$ have Lipschitz constant $\lwL{(l)}$ and $\upL{(l)}$ respectively, $\Delta \sigma^{(l)}$ has a Lipschitz constant $\lwL{(l)}+\upL{(l)}$. Then we have
\begin{equation*}
\begin{aligned}
    \dist(\tidx{(l+1)}, \X{(l+1)})  & \le \|\underline{\sigma}^{(l)}(\tidz{(l)}) - \underline{\sigma}^{(l)}(\z{(l)})\| + \inf_{t\in [0, \Delta \sigma^{(l)}(\z{(l)})]} \|\tilde{t} - t\| \\
    & \le \lwL{(l)} \|\tidz{(l)} - \z{(l)}\| + \left(\sum_k \inf_{t_k\in [0, \Delta \sigma_k^{(l)}(\z{(l)})]} |\tilde{t}_k - t_k|^2\right)^{1/2}.
\end{aligned}
\end{equation*}
We have the entry-wise bound for $\tilde{t} - t$:
\begin{equation*}
\begin{aligned}
\inf_{t_k\in [0, \Delta \sigma_k^{(l)}(\z{(l)})]} |\tilde{t}_k - t_k|^2 &\le \max(|\underline{u}_k^{(l)}|^2, |\Delta \sigma_k^{(l)}(\tidz{(l)}) - \Delta \sigma_k^{(l)}(\z{(l)}) + \overline{u}^{(l)}|^2) \\
&\le 2 \left(|\Delta \sigma_k^{(l)}(\tidz{(l)}) - \Delta \sigma_k^{(l)}(\z{(l)})|^2 + |\underline{u}_k^{(l)}|^2 + |\overline{u}^{(l)}|^2\right)
\end{aligned}
\end{equation*}
Therefore, we get
\begin{equation*}
\begin{aligned}
    \inf_{t\in [0, \Delta \sigma^{(l)}(\z{(l)})]} \|\tilde{t} - t\| & \le \sqrt{2} \left( \|\Delta \sigma^{(l)}(\tidz{(l)}) - \Delta \sigma^{(l)}(\z{(l)})\|^2 +  \|\underline{u}_k^{(l)}\|^2 + \|\overline{u}^{(l)}\|^2 \right)^{1/2} \\
    & \le \sqrt{2} \left( (\lwL{(l)}+\upL{(l)})^2\|\tidz{(l)} - \z{(l)}\|^2 +  \|\underline{u}_k^{(l)}\|^2 + \|\overline{u}^{(l)}\|^2 \right)^{1/2} \\
    & \le \sqrt{ 2(\lwL{(l)}+\upL{(l)})^2 (C_z^{(l)})^2 + 2} ~ \|(\underline{u}^{[l+1]}, \overline{u}^{[l+1]}, v^{[l+1]})\|_2
\end{aligned}
\end{equation*}
Similarly, we have $\lwL{(l)} \|\tidz{(l)} - \z{(l)}\| \le \lwL{(l)}C_z^{(l)} \|(\underline{u}^{[l+1]}, \overline{u}^{[l+1]}, v^{[l+1]})\|_2$. Therefore, we obtain
\begin{equation*}
\begin{aligned}
    \dist(\tidx{(l+1)}, \X{(l+1)}) \le C_x^{(l+1)} \|(\underline{u}^{[l+1]}, \overline{u}^{[l+1]}, v^{[l+1]})\|_2,
\end{aligned}
\end{equation*}
where $C_x^{(l+1)} = \lwL{(l)}C_z^{(l)} + \sqrt{ 2(\lwL{(l)}+\upL{(l)})^2 (C_z^{(l)})^2 + 2}$.

Then by mathematical induction, we proved that \eqref{eq:zperturb} for $0\le l \le L-1$ and \eqref{eq:xperturb} for $1\le l \le L$.

In the second step, we prove \eqref{eq:pperturb0}. Thanks to \eqref{eq:xperturb} with $l=L$, we have for any $\tidx{(L)} \in \tidX{(L)}$, there exists $\x{(L)} \in \X{(L)}$ such that
\begin{equation*}
    \dist(\tidx{(L)}, \x{(L)}) \le C_x^{(L)} \|(\underline{u}^{[L]}, \overline{u}^{[L]}, v^{[L]})\|_2.
\end{equation*}
Then we obtain
\begin{equation*}
\begin{aligned}
     {p}^*_{\CalC} - (c^{\top} \tidx{(L)} + c_0) &\le c^{\top} (\x{(L)} - \tidx{(L)}) \le \|c\| \|\tidx{(L)}-\x{(L)}\| \\
    & \le C_x^{(L)} \|c\| \|(\underline{u}^{[L]}, \overline{u}^{[L]}, v^{[L]})\|_2.
\end{aligned}
\end{equation*}
Taking the infimum over $\tidx{(l)} \in \tidX{(l)}$, we have proved \eqref{eq:pperturb0} with $C_{\CalC} = C_x^{(L)} \|c\|$.
\end{proof}

\section{Equivalence of the optimal layer-wise dual relaxations: $d^*_{{\CalC}_{\text{opt}}} = d^*_{\myopt}$}\label{app:cmudm}
\begin{lemma}\label{lem:fenchel}
Suppose the activation function $\sigma:[\lwz{}, \upz{} ] \to \R$ is bounded from above and below and that $\underline{\sigma}(z) \le \sigma(z) \le \overline{\sigma}(z)$ for all $z \in [\lwz{}, \upz{} ]$. Define
\begin{align}
    f_{\myopt}(\mu, \lambda) &\coloneqq \inf_{ z\in [\lwz{}, \upz{}]} \mu z - \lambda\sigma(z), \label{def:fdm}\\
    f_{\CalC}(\mu, \underline{\lambda}, \overline{\lambda}) &\coloneqq \inf_{ z\in [\lwz{}, \upz{}]} \mu z + \underline{\lambda}\underline{\sigma}(z) - \overline{\lambda}\overline{\sigma}(z). \label{def:fcmu}
\end{align}
For any $\mu$, $\underline{\lambda} \ge 0$ and $\overline{\lambda} \ge 0$, we have
\begin{equation}\label{eq:dmlecmu_app}
    f_{\CalC}(\mu, \underline{\lambda}, \overline{\lambda}) \le f_{\CalC}(\mu, -\left(\overline{\lambda}-\underline{\lambda}\right)_{-}, \left(\overline{\lambda}-\underline{\lambda}\right)_{+}),
\end{equation}
where $\lambda_{+} = \max(\lambda,  0)$ and $\lambda_{-} = \min(\lambda,  0)$.

When $\underline{\sigma}_{\text{opt}}$ and $\overline{\sigma}_{\text{opt}}$ are the optimal convex relaxations defined in \eqref{def:underoversigma_app}, we write $f_{\CalC}$ as $f_{\CalC_{\text{opt}}}$. In this case, we have that for any $\mu$ and $\lambda$
\begin{equation}\label{eq:cmuledm_app}
    f_{\CalC_{\text{opt}}}(\mu, - \lambda_{-}, \lambda_{+}) = f_{\myopt}(\mu, \lambda).
\end{equation}
\end{lemma}
\begin{proof}
First let's prove \eqref{eq:dmlecmu_app}. For $\underline{\lambda} \ge \overline{\lambda} \ge 0$, we have
$$f_{\CalC}(\mu, -\left(\overline{\lambda}-\underline{\lambda}\right)_{-}, \left(\overline{\lambda}-\underline{\lambda}\right)_{+}) = f_{\CalC}(\mu, \underline{\lambda}-\overline{\lambda}, 0)$$
and 
\begin{equation*}
\begin{aligned}
    f_{\CalC}(\mu, \underline{\lambda}, \overline{\lambda}) &= \inf_{ z\in [\lwz{}, \upz{}]} \mu z + \underline{\lambda}\underline{\sigma}(z) - \overline{\lambda}\overline{\sigma}(z) = \inf_{ z\in [\lwz{}, \upz{}]} \mu z + (\underline{\lambda}-\overline{\lambda})\underline{\sigma}(z) - \overline{\lambda}(\overline{\sigma}(z)-\underline{\sigma}(z)) \\
    & \stackrel{(i)}{\le} \sup_{ z\in [\lwz{}, \upz{}]} \mu z - (\underline{\lambda}-\overline{\lambda})\underline{\sigma}(z) = f_{\CalC}(\mu, \underline{\lambda}-\overline{\lambda}, 0),
\end{aligned}
\end{equation*}
where we use $\overline{\lambda}(\overline{\sigma}(z)-\underline{\sigma}(z)) \ge 0$ in (i). Similarly for $\overline{\lambda} \ge \underline{\lambda} \ge 0$, we have
$$f_{\CalC}(\mu, -\left(\overline{\lambda}-\underline{\lambda}\right)_{-}, \left(\overline{\lambda}-\underline{\lambda}\right)_{+}) = f_{\CalC}(\mu, 0, \overline{\lambda}-\underline{\lambda})$$
and 
\begin{equation*}
\begin{aligned}
    f_{\CalC}(\mu, \underline{\lambda}, \overline{\lambda}) &= \inf_{ z\in [\lwz{}, \upz{}]} \mu z + \underline{\lambda}\underline{\sigma}(z) - \overline{\lambda}\overline{\sigma}(z) = \sup_{ z\in [\lwz{}, \upz{}]} \mu z - (\overline{\lambda}-\underline{\lambda})\overline{\sigma}(z) - \underline{\lambda}(\overline{\sigma}(z)-\underline{\sigma}(z)) \\
    & \stackrel{(i)}{\le} \sup_{ z\in [\lwz{}, \upz{}]} \mu z - (\overline{\lambda}-\underline{\lambda})\overline{\sigma}(z) = f_{\CalC}(\mu, 0, \overline{\lambda}-\underline{\lambda}),
\end{aligned}
\end{equation*}
where we use $\underline{\lambda}(\overline{\sigma}(z)-\underline{\sigma}(z)) \ge 0$ in (i).

Then let's prove \eqref{eq:cmuledm_app}. For $\lambda < 0$ ($\lambda_{+} = 0$ and $\lambda_{-} = \lambda$), by definition we have
\begin{equation*}
\begin{aligned}
    &f_{\CalC_{\text{opt}}}(\mu, -\lambda, 0) = \inf_{ z\in [\lwz{}, \upz{}]} \mu z - \lambda \underline{\sigma}_{\text{opt}}(z) = \lambda \sup_{ z\in [\lwz{}, \upz{}]} \frac{\mu}{\lambda} z - \underline{\sigma}_{\text{opt}}(z) \\
    & \stackrel{(i)}{=:} \lambda \left(\underline{\sigma}_{\text{opt}}\right)^*(\mu/\lambda) \stackrel{(ii)}{=} \lambda \left(\sigma\right)^*(\mu/\lambda)  \stackrel{(iii)}{:=} \inf_{ z\in [\lwz{}, \upz{}]} \mu z - \lambda\sigma(z) = f_{\myopt}(\mu, \lambda),
\end{aligned}
\end{equation*}
where we use the definition of convex conjugate in (i) and (iii) and the Fenchel-Moreau theorem (Theorem 12.2 in \citet{rockafellar2015convex}) in (ii). For $\lambda = 0$, it is obvious. Similarly, for $\lambda > 0$ ($\lambda_{+} = \lambda$ and $\lambda_{-} = 0$), by definition we have
\begin{equation*}
\begin{aligned}
    &f_{\CalC_{\text{opt}}}(\mu, 0, \lambda) = \inf_{ z\in [\lwz{}, \upz{}]} \mu z - \lambda \overline{\sigma}_{\text{opt}}(z) = -\lambda \sup_{ z\in [\lwz{}, \upz{}]} -\frac{\mu}{\lambda} z - (-\overline{\sigma}_{\text{opt}})(z) \\
    &\stackrel{(i)}{=:} -\lambda \left(-\overline{\sigma}_{\text{opt}}\right)^*(-\mu/\lambda) \stackrel{(ii)}{=} -\lambda \left(-\sigma\right)^*(-\mu/\lambda) \stackrel{(iii)}{:=} \inf_{ z\in [\lwz{}, \upz{}]} \mu z - \lambda\sigma(z) = f_{\myopt}(\mu, \lambda),
\end{aligned}
\end{equation*}
where we use the definition of convex conjugate in (i) and (iii) and the Fenchel-Moreau theorem in (ii), again. 
\end{proof}

\begin{proof}[Proof of Theorem~\ref{thm:cmudm}]
In the first step, we simplify the form of $g_{\CalC}(\mu^{[L]},\underline{\lambda}^{[L]},\overline{\lambda}^{[L]})$. By definition \eqref{eq:rv_cmu_dual}, we have
\begin{equation}\label{eq:rv_cmu_dual3}
\begin{aligned}
    g_{\CalC}(\mu^{[L]}, \underline{\lambda}^{[L]}, \overline{\lambda}^{[L]}) &= g^{(0)}(\mu^{(0)}) + \sum_{l=1}^{L-1} g^{(l)}(\mu^{(l)}, \overline{\lambda}^{(l-1)}-\underline{\lambda}^{(l-1)}) + g^{(L)}(c, \overline{\lambda}^{(l-1)}-\underline{\lambda}^{(l-1)}) \\
    & + \sum_{l=0}^{L-1} \left( \tilde{g}_{\CalC}^{(l)}(\mu^{(l)}, \underline{\lambda}^{(l)}, \overline{\lambda}^{(l)}) - \bias{(l){\top}} \mu^{(l)} \right),
\end{aligned}
\end{equation}
where
\begin{equation}\label{eq:rv_g0}
\begin{aligned}
    g^{(0)}(\mu^{(0)}) = \inf_{\x{(0)} \in \setS_{in}(x^{\text{nom}})} \left( -\W{(0){\top}} \mu^{(0)} \right)^{\top} \x{(0)} 
\end{aligned}
\end{equation}
\begin{equation}\label{eq:rv_gl}
\begin{aligned}
    g^{(l)}(\mu^{(l)}, \lambda^{(l-1)}) &=\inf_{\x{(l)}} \left( \lambda^{(l-1)} - \W{(l){\top}} \mu^{(l)} \right)^{\top} \x{(l)} = 1_{\lambda^{(l-1)} = \W{(l){\top}} \mu^{(l)}},
\end{aligned}
\end{equation}
\begin{equation}\label{eq:rv_gL}
\begin{aligned}
    g^{(L)}(c, \lambda^{(L-1)}) &=\inf_{\x{(L)}} \left( \lambda^{(L-1)} + c \right)^{\top} \x{(L)} + c_0 = 1_{\lambda^{(L-1)}= -c} + c_0,
\end{aligned}
\end{equation}
and 
\begin{equation}\label{eq:rv_cmu_tildegl_org}
\begin{aligned}
    &\tilde{g}_{\CalC}^{(l)}(\mu^{(l)}, \underline{\lambda}^{(l)}, \overline{\lambda}^{(l)}) = \inf_{\lwz{(l)} \le \z{(l)} \le \upz{(l)}} \bigg\{ \mu^{(l){\top}} \z{(l)} +\underline{\lambda}^{(l){\top}} \underline{\sigma}^{(l)}(\z{(l)}) - \overline{\lambda}^{(l){\top}} \overline{\sigma}^{(l)}(\z{(l)}) \bigg\}.
\end{aligned}
\end{equation}

In the second step, for any $\mu^{[L]}$, $\underline{\lambda}^{[L]} \ge 0$ and $\overline{\lambda}^{[L]} \ge 0$, we apply \eqref{eq:dmlecmu_app} in Lemma~\ref{lem:fenchel} entry-wisely on \eqref{eq:rv_cmu_tildegl_org}, and obtain
\begin{equation*}
     \tilde{g}_{\CalC}^{(l)}(\mu^{(l)}, \underline{\lambda}^{(l)}, \overline{\lambda}^{(l)}) \le \tilde{g}_{\CalC}^{(l)}(\mu^{(l)}, - \lambda^{(l)}_{-}, \lambda^{(l)}_{+}),
\end{equation*}
in which $\lambda^{(l)} := \underline{\lambda}^{(l)} -\overline{\lambda}^{(l)}$. After Plugging $\lambda^{[L]} = \lambda^{[L]}_{+} + \lambda^{[L]}_{-}$ and $\lambda^{[L]} := \underline{\lambda}^{[L]} -\overline{\lambda}^{[L]}$ into equation~\eqref{eq:rv_cmu_dual3}, we obtain that
\begin{equation}\label{eq:dmlecmu}
    g_{\CalC}(\mu^{[L]}, \underline{\lambda}^{[L]}, \overline{\lambda}^{[L]}) \le g_{\CalC}(\mu^{[L]}, -\lambda_{-}^{[L]}, \lambda_{+}^{[L]}).
\end{equation}
Therefore, the dual problem~\eqref{eq:rv_cmu_dual_weak} can be rewritten as an unconstrained optimization problem as
\begin{equation}\label{eq:rv_cmu_dual_simp0}
d^*_{\CalC} = \max_{\mu^{[L]},\lambda^{[L]}} g_{\CalC}(\mu^{[L]}, -\lambda_{-}^{[L]}, \lambda_{+}^{[L]}).
\end{equation}

In the third step, we simplify $g_{{\myopt}}(\mu^{[L]}, \lambda^{[L]})$ based on its definition \eqref{eq:rv_dm} and obtain 
\begin{equation}\label{eq:rv_dm_dual2}
    \begin{aligned}
    g_{{\myopt}}(\mu^{[L]}, \lambda^{[L]}) &= g^{(0)}(\mu^{(0)}) + \sum_{l=1}^{L-1} g^{(l)}(\mu^{(l)}, \lambda^{(l-1))}) + g^{(L)}(c, \lambda^{(L-1)}) \\
    & + \sum_{l=0}^{L-1} \left( \tilde{g}_{{\myopt}}^{(l)}(\mu^{(l)}, \lambda^{(l)}) - \bias{(l)}{}^\top \mu^{(l)} \right)
\end{aligned}
\end{equation}
in which 
\begin{equation}\label{eq:rv_dm_tildegl}
\begin{aligned}
    &\tilde{g}_{{\myopt}}^{(l)}(\mu^{(l)}, \lambda^{(l)}) = \inf_{\lwz{(l)} \le \z{(l)} \le \upz{(l)}}  \mu^{(l)}{}^\top \z{(l)} - \lambda^{(l)}{}^\top \sigma^{(l)}(\z{(l)}).
\end{aligned}
\end{equation}

In the forth step, for any $\mu^{[L]}$ and $\lambda^{[L]}$, since all the nonlinear layers are non-interactive, we apply \eqref{eq:cmuledm_app} in Lemma~\ref{lem:fenchel} entry-wisely on \eqref{eq:rv_cmu_tildegl_org} and \eqref{eq:rv_dm_tildegl} and obtain
\begin{equation*}
    \tilde{g}_{\CalC_{\text{opt}}}^{(l)}(\mu^{(l)}, -\lambda^{(l)}_{-}, \lambda^{(l)}_{+}) = \tilde{g}_{\myopt}^{(l)}(\mu^{(l)}, \lambda^{(l)}).
\end{equation*}
After plugging $\lambda^{[L]} = \lambda^{[L]}_{+} + \lambda^{[L]}_{-}$ into \eqref{eq:rv_cmu_dual3}, we see that the other three terms in $g_{\CalC_{\text{opt}}}(\mu^{[L]}, - \lambda^{[L]}_{-}, \lambda^{[L]}_{+})$ and $g_{\myopt}(\mu^{[L]}, \lambda^{[L]})$ are the same. Therefore, we have proved that for any $\mu^{[L]}$ and $\lambda^{[L]}$, we have 
\begin{equation}\label{eq:cmuledm}
    g_{\CalC_{\text{opt}}}(\mu^{[L]}, - \lambda^{[L]}_{-}, \lambda^{[L]}_{+}) = g_{\myopt}(\mu^{[L]}, \lambda^{[L]})
\end{equation}
Finally, combining \eqref{eq:rv_dm_dual_weak}, \eqref{eq:rv_cmu_dual_simp0} and \eqref{eq:cmuledm}, we obtain $d^*_{{\CalC}_{\text{opt}}} = d^*_{\myopt}$. 
\end{proof}

\section{A greedy algorithm to solve the dual problems}
\label{app:greedydual}
\subsection{Some useful results to simplify the dual problems}
We provide the following useful results when solving \eqref{eq:rv_cmu_dual_weak} and \eqref{eq:rv_dm_dual_weak}. First, the dual problem~\eqref{eq:rv_cmu_dual_weak} can be rewritten as an unconstrained optimization problem inspired by \eqref{eq:rv_cmu_dual_simp0}. We define a two-argument function, reusing the name $g_{\mathcal{C}}$, as 
$$g_{\mathcal{C}}(\mu^{[L]},\lambda^{[L]}) := g_{\mathcal{C}}(\mu^{[L]},-\lambda_{-}^{[L]}),\lambda_{+}^{[L]}).$$
Then we have the following useful results.
\begin{proposition}\label{prop:dual}
Denote $\lambda_{+} = \max(\lambda,  0)$ and $\lambda_{-} = \min(\lambda,  0)$. 
\begin{enumerate}
    \item For dual of the convex relaxed problem~\eqref{eq:rv_cmu} defined in \eqref{eq:rv_cmu_dual_weak}, we have
    \begin{equation}\label{eq:rv_cmu_dual_simp}
    \begin{aligned}
    d^*_{\mathcal{C}} = \max_{\mu^{[L]},\lambda^{[L]}} \bigg\{ g_{\mathcal{C}}(\mu^{[L]}, \lambda^{[L]}) \coloneqq c_0 + g^{(0)}(\mu^{(0)}) 
     + \sum_{l=0}^{L-1} \left( \tilde{g}_{\mathcal{C}}^{(l)}(\mu^{(l)}, \lambda^{(l)}) - \bias{(l)}{}^\top \mu^{(l)} \right) \bigg\},
\end{aligned}
\end{equation}
where
\begin{equation}\label{eq:lambdarecur}
\begin{aligned}
    \lambda^{(L-1)} &= -c,\quad \lambda^{(l)} &=  \W{(l+1)}{}^\top \mu^{(l+1)} \quad \forall l\in [L-1],
\end{aligned}
\end{equation}
\begin{equation}\label{eq:g0}
\begin{aligned}
    g^{(0)}(\mu^{(0)}) = \inf_{\x{(0)} \in \setS_{in}(x^{\text{nom}})} \left( -\W{(0)}{}^\top \mu^{(0)} \right)^{\top} \x{(0)},
\end{aligned}
\end{equation}
and 
\begin{equation*}
\begin{aligned}
    & \tilde{g}_{\mathcal{C}}^{(l)}(\mu^{(l)}, \lambda^{(l)}) = \inf_{\lwz{(l)} \le \z{(l)} \le \upz{(l)}} \bigg\{ \mu^{(l)}{}^\top \z{(l)}  - \lambda_{-}^{(l)}{}^\top \underline{\sigma}^{(l)}(\z{(l)}) - \lambda_{+}^{(l)}{}^\top \overline{\sigma}^{(l)}(\z{(l)}) \bigg\}.
\end{aligned}
\end{equation*}

\item For the dual of the original nonconvex problem~\eqref{eq:rvbnds} defined in \eqref{eq:rv_dm_dual_weak}, we have
\begin{equation}\label{eq:rv_dm_dual_simp}
    \begin{aligned}
    d^*_{{\myopt}} \coloneqq \max_{\mu^{[L]},\lambda^{[L]}} \bigg\{ g_{{\myopt}}(\mu^{[L]}, \lambda^{[L]}) = c_0 + g^{(0)}(\mu^{(0)})
    + \sum_{l=0}^{L-1} \left( \tilde{g}_{{\myopt}}^{(l)}(\mu^{(l)}, \lambda^{(l)}) - \bias{(l)}{}^\top \mu^{(l)}\right) \bigg\},
\end{aligned}
\end{equation}
where \eqref{eq:lambdarecur} still holds true and
\begin{equation*}
\begin{aligned}
    &\tilde{g}_{{\myopt}}^{(l)}(\mu^{(l)}, \lambda^{(l)}) = \inf_{\lwz{(l)} \le \z{(l)} \le \upz{(l)}}  \mu^{(l)}{}^\top \z{(l)} - \lambda^{(l)}{}^\top \sigma^{(l)}(\z{(l)}).
\end{aligned}
\end{equation*}
\item Suppose that a nonlinear neuron $x_j^{(l+1)} = \sigma^{(l)}(z_{I_j}^{(l)})$ is effectively linear within the input domain $\setS_{in}(x^{\text{nom}})$, i.e., there exists a linear relation $x_j^{(l+1)} = V_j^{(l)} z_{I_j}^{(l)} + d_j^{(l)}$ for all $\x{(0)} \in \setS_{in}(x^{\text{nom}})$, then we can simplify the convex relaxed problem~\eqref{eq:rv_cmu} by setting
\begin{equation*}
    \underline{\sigma}^{(l)}_i(\z{(l)}) = \overline{\sigma}^{(l)}_i(\z{(l)}) = V_j^{(l)} z_{I_j}^{(l)} + d_j^{(l)},
\end{equation*}
or simplify the original nonconvex problem~\eqref{eq:rvbnds} by setting
\begin{equation*}
    \sigma^{(l)}_i(z^{(l)}) = V_j^{(l)} z_{I_j}^{(l)} + d_j^{(l)}.
\end{equation*}
If this neuron does not interact with other neurons in the same layer, i.e., $z_{I_j}^{(l)}$ is not the input of $x_{k}^{(l+1)}$ for any $k \neq j$.
Then for any optimal point for both dual problems, we have
\begin{equation}\label{eq:stableneuron}
    \mu_{I_j}^{(l)} = V_j^{(l)}{}^\top \lambda_j^{(l)}.
\end{equation}
\end{enumerate}
\end{proposition}
Similar results have been obtained in several previous works \cite{wong2018provable,dvijotham2018dual,wong2018scaling,qin2018verification}. 
\begin{proof}
$ $
\begin{enumerate}
\item \eqref{eq:rv_cmu_dual_simp} is a straightforward rewriting of \eqref{eq:rv_cmu_dual3} with \eqref{eq:rv_gl}, \eqref{eq:rv_gL}, \eqref{eq:rv_g0} and \eqref{eq:rv_cmu_tildegl_org}.
\item \eqref{eq:rv_dm_dual_simp} is a straightforward rewriting of \eqref{eq:rv_dm_dual2} with \eqref{eq:rv_gl}, \eqref{eq:rv_gL}, \eqref{eq:rv_g0} and \eqref{eq:rv_dm_tildegl}.
\item This can be proved with the same treatment of linear layers in the two items above.
\end{enumerate}
\end{proof}

\subsection{Greedily solving the dual with linear bounds}
\label{app:dualgreedyalg}
Suppose the relaxed bounds $\underline{\sigma}$ and $\overline{\sigma}$ are linear, i.e.,
\begin{align*}
    \underline{\sigma}^{(l)}(\z{(l)}) \coloneqq \underline{a}^{(l)} \z{(l)} + \underline{b}^{(l)}, \quad \overline{\sigma}^{(l)}(\z{(l)})  \coloneqq \overline{a}^{(l)} \z{(l)} + \overline{b}^{(l)}.
\end{align*}
In this case, in the dual problem~\eqref{eq:rv_cmu_dual_simp} we have
\begin{equation*}
\begin{aligned}
    d^*_{\mathcal{C}} = \max_{\mu^{[L]},\lambda^{[L]}} \bigg\{ g_{\mathcal{C}}(\mu^{[L]}, \lambda^{[L]}) \coloneqq c_0 + g^{(0)}(\mu^{(0)}) 
     + \sum_{l=0}^{L-1} \left( \tilde{g}_{\mathcal{C}}^{(l)}(\mu^{(l)}, \lambda^{(l)}) - \bias{(l)}{}^\top \mu^{(l)} \right) \bigg\},
\end{aligned}
\end{equation*}
where
\begin{equation*}
\begin{aligned}
    & \tilde{g}_{\mathcal{C}}^{(l)}(\mu^{(l)}, \lambda^{(l)}) = \inf_{\lwz{(l)} \le \z{(l)} \le \upz{(l)}}  \bigg\{ \left(\mu^{(l)}{} - \lambda^{(l)}_{+} \overline{a}^{(l)} - \lambda^{(l)}_{-} \underline{a}^{(l)} \right) \z{(l)} + \left(\lambda^{(l)}_{+} \overline{b}^{(l)} - \lambda^{(l)}_{-} \underline{b}^{(l)} \right) \bigg\}.
\end{aligned}
\end{equation*}

In the following, we propose a dual greedy algorithm to {\it greedily} ({approximately}) solve the dual problem~\eqref{eq:rv_cmu_dual_weak} and/or its simplified version~\eqref{eq:rv_cmu_dual_simp}. Let $\lambda^{[L]}$ be determined by \eqref{eq:lambdarecur} and $\mu^{[L]}$, for stable neurons, be determined by \eqref{eq:stableneuron}. Both of these are optimal. For {\it unstable} neurons ($\underline{z}_i^{(l)} \le 0 \le \overline{z}_i^{(l)}$), a suboptimal $\mu^{[L]}$ can be obtained by 
\begin{equation*}
    \mu^{(l)} = \argmax_{\mu^{(l)}} \tilde{g}_{\mathcal{C}}^{(l)}(\mu^{(l)}, \lambda^{(l)}),
\end{equation*}
which has a closed form solution 
\begin{equation*}
% \label{eq:dual_suboptimal_mu}
    \mu_i^{(l)} = \overline{a}_i^{(l)} \left(\lambda_{i}^{(l)}\right)_{+} + \underline{a}_i^{(l)} \left(\lambda_{i}^{(l)}\right)_{-}.
\end{equation*}
Notice that the above suboptimal solution for unstable neurons and the optimal solution \eqref{eq:stableneuron} for stable neurons ($\underline{a}^{(l)}=\overline{a}^{(l)}$ and $\underline{b}^{(l)}=\overline{b}^{(l)}$) can be unified in a single formulae. 

Finally, we summarize our algorithm to greedily solve the dual problem as 
\begin{equation}
\label{eq:dual_suboptimal_mu}
    \lambda^{(L-1)} = -c, \quad \mu^{(l)} = \overline{a}^{(l)} \left(\lambda^{(l)}\right)_{+} + \underline{a}^{(l)} \left(\lambda^{(l)}\right)_{-} \quad \lambda^{(l)} =  \W{(l+1)}{}^\top \mu^{(l+1)} \quad \forall l\in [L-1],
\end{equation}
and the corresponding lower bound is
\begin{equation}
\label{eq:dual_suboptimal_val}
\begin{aligned}
    & g_{\mathcal{C}}(\mu^{[L]}, \lambda^{[L]}) = c_0 + g^{(0)}(\mu^{(0)}) +  \sum_{l=0}^{L-1} \left( \overline{b}^{(l)\top} \left(\lambda^{(l)}\right)_{+} - \underline{b}^{(l)\top} \left(\lambda^{(l)}\right)_{-} - \bias{(l)}{}^\top \mu^{(l)} \right) . 
\end{aligned}
\end{equation}
We point out that the algorithm above can exactly recover what was proposed in Theorem 1 in \citet{wong2018provable}. Their $\nu$ is our $\mu$ and their $\hat{\nu}$ is our $\lambda$. 

\section{Which problem to solve in practice?}
\label{sec:whichProblem}
Thanks to the strong duality, the same lower bound can be achieved from both the primal and the dual problems, and thus we have the freedom to choose the problem to solve. When the relaxed upper and lower bounds, i.e., $\underline{\sigma}^{(l)}$ and $\overline{\sigma}^{(l)}$, are piece-wise linear (e.g. \eqref{eq:relu_ehlers} for ReLU networks), both the primal and dual problems are linear programs and can be efficiently solved  by existing LP solvers (which is what we use in the coming sections). In other cases, we recommend to solve the dual problem~\eqref{eq:rv_dm_dual_weak} for two reasons. First, the primal relaxed problem~\eqref{eq:rv_cmu} is a constrained optimization problem, and its constraints may not have a simple analytic form when $\underline{\sigma}^{(l)}$ and $\overline{\sigma}^{(l)}$ are not piecewise linear; see examples in Fig.~\ref{fig:relax}. On the contrary, the dual problem~\eqref{eq:rv_dm_dual_weak} can be framed as an unconstrained optimization problem and its objective function has a simple analytic form for some common activation functions~\cite{dvijotham2018dual}.
Second, the optimization process of \eqref{eq:rv_dm_dual_weak} can be stopped anytime to give a lower bound of $p_{\myopt}^*$, thanks to weak duality, but this is not true of the primal view.
Of course, $\underline{\sigma}^{(l)}$ and $\overline{\sigma}^{(l)}$ must be in the form of \eqref{def:underoversigma} to achieve the optimal value.
\clearpage

\section{Additional Experimental Details}

\subsection{Neural Networks Used}
\label{sec:appendix-architectures}
Here is a list of the network architectures that we use in this paper along with their references if applicable.
\paragraph{MNIST robust error experiment}
\begin{itemize}
    \item \textsc{MLP-A}: a multilayer perceptron consisting of  1 hidden layer with 500 neurons \cite{tjeng2018evaluating}.
    \item \textsc{MLP-B}: a multilayer perceptron consisting of 2 hidden layers with 100 neurons each.
\end{itemize}

\paragraph{MNIST $\epsilon$-search experiment}
\begin{itemize}
    \item \textsc{CNN-small}: ConvNet architecture with two convolutional layers with 16 and 32 filters respectively (size (size $4 \times 4$ and stride of 2 in both), followed by two fully-connected layers with 100 and 10 units respectively \cite{wong2018scaling}. 
    \item \textsc{CNN-wide-k}: ConvNet architecture with two convolutional layers of $4 \times k$ and $8 \times k$ filters (size $4 \times 4$ and stride of 2 in both) followed by a $128 \times k$ fully connected layer followed by two fully-connected layers of sizes $128 \times k$ and 10 respectively. The parameter k is used to control the width of the network \cite{wong2018scaling}.
    \item \textsc{CNN-deep-k}: ConvNet architecture with  $k$ convolutional layers with 8 filters followed by $k$ convolutional filters with 16 filters followed by two fully-connected layers of sizes $100 \times k$ and 10 respectively. The parameter $k$ is used to control the depth of the network \cite{wong2018scaling}.
    
    \item \textsc{MLP-[9]-500}: a multilayer perceptron consisting of 9 hidden layer with 500 neurons each. 
    
    \item \textsc{MLP-[9]-100}: a multilayer perceptron consisting of 9 hidden layer with 100 neurons each.
    
    \item \textsc{MLP-[2]-100}: a multilayer perceptron consisting of 2 hidden layer with 100 neurons each.
\end{itemize}

\paragraph{CIFAR-10 $\epsilon$-search experiment}
\begin{itemize}
    \item \textsc{CNN-small}: ConvNet architecture with two convolutional layers with 16 and 32 filters respectively (size (size $4 \times 4$ and stride of 2 in both), followed by two fully-connected layers with 100 and 10 units respectively. 
    \item \textsc{CNN-wide-k}: ConvNet architecture with two convolutional layers of $4 \times k$ and $8 \times k$ filters (size $4 \times 4$ and stride of 2 in both) followed by a $128 \times k$ fully connected layer followed by two fully-connected layers of sizes $128 \times k$ and 10 respectively. The parameter $k$ is used to control the width of the network.
    \end{itemize}

\subsection{Training Modes}
\label{sec:appendix-training-modes}
In this paper, we use only one pre-trained network from the literature, and we train the rest from scratch.
\paragraph{Pre-trained Networks}
\begin{itemize}
    \item \textsc{Adv-MLP-A}: this is a multilayer perceptron with 1 hidden layer having 500 units. It is trained using PGD with $l_\infty$ perturbation of $\epsilon=0.1$,  and is used in \citet{tjeng2018evaluating} and \citet{raghunathan2018certified}. It can be found at \url{https://github.com/vtjeng/MIPVerify_data/tree/master/weights/mnist/RSL18a.}
\end{itemize}

\paragraph{Networks Trained from Scratch.} We train all models in parallel on a GPU-cluster with P100 GPUs.
\begin{itemize}
    \item All networks in the paper that have the prefix or training mode \textsc{Adv} are trained with PGD using the code available at \url{https://github.com/locuslab/convex_adversarial/blob/master/examples/mnist.py}.
    \item All networks in the paper that have the prefix or training mode \textsc{LPd} are trained with the robust training method of \citet{wong2018scaling} using the code available at \url{https://github.com/locuslab/convex_adversarial/blob/master/examples/mnist.py}.
    \item All networks in the paper that have the prefix \textsc{Nor} or training mode \textsc{Normal} are trained the regular cross-entropy loss using the code available at \url{https://github.com/locuslab/convex_adversarial/blob/master/examples/mnist.py}.

    \item All the CIFAR-10 networks in the paper have the same naming convention as above, but are trained using the code available at \url{https://github.com/locuslab/convex_adversarial/blob/master/examples/cifar.py}.
\end{itemize}

\section{Parallel Computation Details}
\label{sec:cluster}

\paragraph{Why do we need parallel computing to solve \lpo{}?}
The nature of our \lpo{} algorithm requires solving a number of LP that scales with the number of neurons in the network we are verifying. For example, if we want to verify a network with 10k neurons on ten samples the MNIST dataset. We need to solve roughly $10\text{k}~\text{LPs/sample} \times 10~\text{samples} = 100$k  LPs. 

The average time for solving an LP varies with the size of the network (see Fig.~\ref{fig:durations_mnist} and \ref{fig:durations_cifar}). It also varies depending on which layer in the network the neuron, for which we are solving the LP, is in (see Fig.~\ref{fig:durations_cifar_per_layer}). Let us say on average the duration for solving one LP is $10$ sec on the CPUs we use, which is reasonable for networks that we consider in this paper. \textit{Therefore, for verifying one network, we need around 1 million sec which is roughly 11 days}. 

Doing this for all the models in the paper and for more samples would take years. This is why parallelizing the computation was crucial. Therefore we conduct all the experiments on a cluster with \textbf{1000 CPU-nodes}. Another key point here was to make sure that the scheduling pipeline on the cluster has \textbf{very low latency}, because we need to solve around 100 million jobs in total in the paper, each of which is on the order of seconds. So any latency in the pipeline can cause significant overhead. The details of the scheduling pipeline are beyond the scope of this paper.

\paragraph{CPU specifications.}
Each CPU-node we used has 2 virtual CPUs with a 2.4 GHz Intel(R) Xeon(R) E5-2673 v3 (Haswell) processor and 7GB of RAM.

\paragraph{Linear programming (LP) solver used.}
We construct all the LP models in python using CVXPY \cite{cvxpy}, and the models are solved using an open-source solver, ECOS \cite{bib:Domahidi2013ecos}. We found this solver to be the fastest among other open-source solvers for our application.

\section{Computational Time for Solving \lpo{}}
\label{sec:computation-time}
The solve time of the LP in \eqref{eq:rv_cmu} depends mainly on the size and the training method of a neural network. It also depends on the input-space dimension. 

\paragraph{Dependence on architecture and training mode.}
Fig.~\ref{fig:durations_mnist} and \ref{fig:durations_cifar} shows the average solve time of the LP in \eqref{eq:rv_cmu} for various networks and training methods that are used in the paper on MNIST and CIFAR-10 datasets, respectively. This averaging is over all the neurons in each network, and over ten samples of each dataset. Note how the solve time increases as the network becomes wider or deeper. This is because the number of decision variables and constraints in the LP increases as the network becomes wider or deeper. Another observation is that, in contrast to MILP \cite{tjeng2018evaluating}, the solve time for robustly trained networks seems to be larger than those which are trained using the regular cross-entropy loss or those which are randomly initialized. This is possibly due to the fact that we are not exploiting the stability of neurons in our implementation of the LP as opposed to what is done in the MILP implementation of \citet{tjeng2018evaluating}.
\paragraph{Dependence on which layer we are solving for.}
Fig.~\ref{fig:durations_cifar_per_layer} shows the average solve time per neuron per layer of the LP in \eqref{eq:rv_cmu} for each of the networks that are used in the paper on the CIFAR-10 dataset. Notice how the solve time of the LP increases as we go deeper into the network.

\begin{figure}[p]
    \centering
    \includegraphics[width=0.8\textwidth]{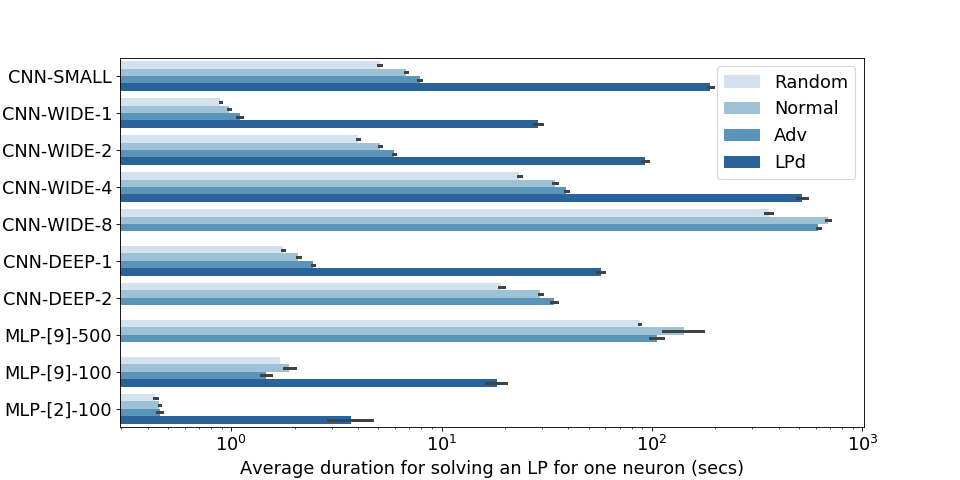}
    \caption{Average duration for solving the LPs for each model (averaged over the neurons in the model and over 10 samples of  the MNIST dataset.}
    \label{fig:durations_mnist}
\end{figure}

\begin{figure}[p]
    \centering
    \includegraphics[width=0.8\textwidth]{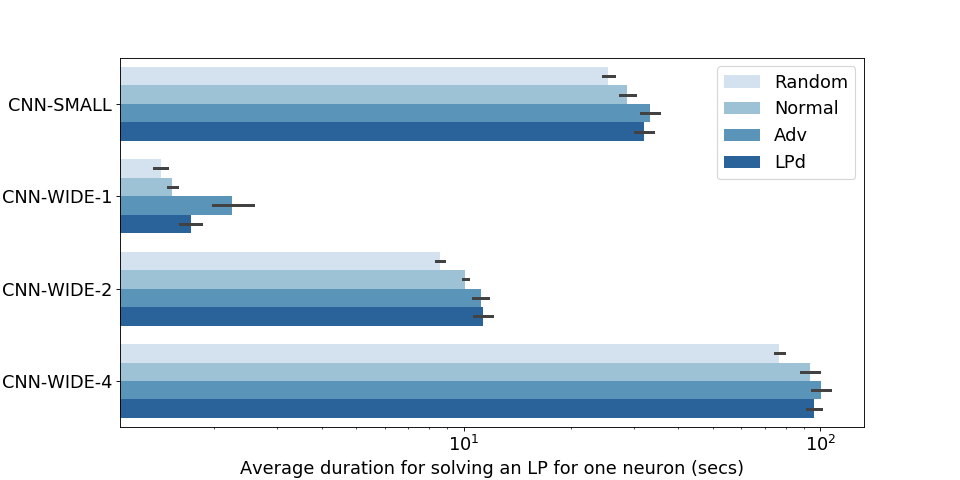}
    \caption{Average duration for solving the LPs for each model averaged over the neurons in the model and over 10 samples of  the CIFAR-10 dataset.}
    \label{fig:durations_cifar}
\end{figure}

\begin{figure}[p]
    \centering
    \includegraphics[width=0.8\textwidth]{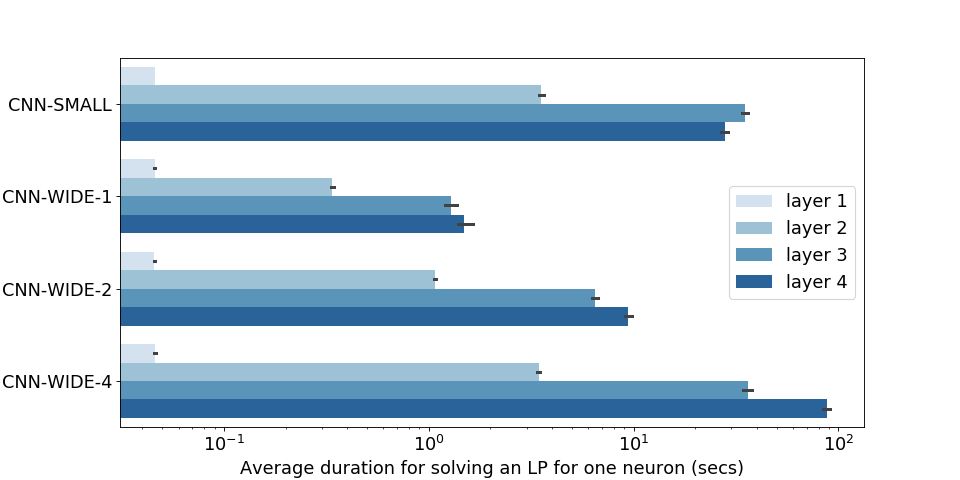}
    \caption{Average duration for solving the LPs per layer  for each model averaged over the neurons in the model and over 10 samples of  the CIFAR-10 dataset.}
    \label{fig:durations_cifar_per_layer}
\end{figure}

\section{Full Results of  Certified Bounds on the Minimum Adversarial Distortion Experiment}\label{sec:exp2_appendix}

\subsection{Implementation details}\label{sec:exp2_implementation}
In this experiment, we are interested in searching for the minimum adversarial distortion $\epsilon$, which is the $l_\infty$ radius of largest $l_\infty$ ball in which no adversarial examples can be crafted. 

An upper bound on $\epsilon$ can be calculated by using PGD in a binary search setting: given an initial guess of $\epsilon$, PGD can be used to find an adversarial example. 
If successful, divide $\epsilon$ by 2; else multiply $\epsilon$ by 2; and repeat until the change in $\epsilon$ is below some tolerance ($10^{-5}$ in our case).

Lower bounds on $\epsilon$ are calculated using \lpd{}, \lpp{} , or our \lpo{} algorithm in a binary search setting; given an initial guess of $\epsilon$, any of these algorithms can be used to check whether the network is robust within $\epsilon$-perturbation of the input. If robust, multiply $\epsilon$ by 2; else divide $\epsilon$ by 2; and repeat until the change in $\epsilon$ is below a tolerance. The tolerances used in the paper are:
\begin{itemize}
    \item  tol($\epsilon_\text{\lpd})=10^{-5}$ because \lpd{} is computationally very cheap.
    \item  tol($\epsilon_\text{\lpp}) =5\%\times\epsilon_\text{\lpd} $ because \lpp{} is computationally expensive.   
    \item  tol($\epsilon_\text{\lpo}) =5\%\times\epsilon_\text{\lpd} $ because \lpo{} is computationally expensive. 
\end{itemize}  

Since solving \lpo{} is really expensive, we find the $\epsilon$-bounds only for ten samples of the MNIST and CIFAR-10 datasets. In this experiment, both \textsc{Adv-} and \textsc{LPd-}networks are trained with an $l_\infty$ maximum allowed perturbation of 0.1 and $8/255$ on MNIST and CIFAR-10, respectively. The full results are reported in Tables~\ref{table:epsilon_bounds} and \ref{table:epsilon_bounds-cifar} respectively.

\subsection{Results}\label{sec:exp2results_appendix}
Tables~\ref{table:epsilon_bounds} and \ref{table:epsilon_bounds-cifar} both report, for ten samples of MNIST and CIFAR-10 respectively, for a wide range of networks :
\begin{enumerate}
    \item  The training mode, whether the network is trained using regular CE loss (Normal), using adversarial examples generated by PGD (\textsc{Adv}), or using the robust loss in \citet{wong2018provable} (\textsc{LPd}).
    \item Mean lower bounds on $\epsilon$ found by \lpd, \lpp, and \lpo. Note that naturally 
    $$\epsilon_\text{\lpd}\le\epsilon_\text{\lpp}\le\epsilon_\text{\lpo}$$
    \item A mean upper bound on $\epsilon$ found by PGD.
    \item The median percentage gap between PGD and the three LP-relaxed algorithms. The percentage gap is defined as
    \begin{equation}\nonumber
        \text{\%gap} = \frac{(\epsilon_{\textsc{PGD}} - \epsilon_{\textsc{LP-x}})}{\epsilon_{\textsc{PGD}}}\times 100.
    \end{equation}
    It is also easy to see that naturally, $$\%\text{gap}_{\text{\lpd}} \ge \%\text{gap}_{\text{\lpp}} \ge \%\text{gap}_{\text{\lpo}}$$
\end{enumerate}

The results of both tables show that for all networks, the certified lower bounds on $\epsilon$ using \lpd{}, \lpp{}, or \lpo{} are 1.5 to 5 times smaller than the upper bound found by PGD on MNIST, and 1.5 to 2 times smaller than the upper bound found by PGD on MNIST. This gap can also clearly be observed in Fig.~\ref{fig:gap_mnist} and Fig.~\ref{fig:gap_cifar} for MNIST and CIFAR-10, respectively.

Therefore, the improvement that we get using  \lpo{} and \lpp{} over \lpd{} is not significant and doesn't close the gap with the PGD upper bound.

% \vspace{-1in}
\begin{table*}[htbp]
\centering
\vspace{-.5in}
\caption{Certified bounds on the minimum adversarial distortion $\epsilon$ for ten random samples from the test set of MNIST.}
\vskip 0.15in
\label{table:epsilon_bounds}
\begin{sc}
\begin{small}
\begin{adjustbox}{max width=1\textwidth}
\begin{tabular}{lcccccccc}
\toprule

% \multirow{3}{*}{\begin{tabular}[c]{@{}c@{}}Training\\ Method \end{tabular}}
\multicolumn{1}{c}{\multirow{2}[2]{*}{Network}}
&\multicolumn{1}{c}{\multirow{2}[2]{*}{\begin{tabular}[c]{@{}c@{}}Training\\ Mode\end{tabular}}}
& \multicolumn{3}{c}{\begin{tabular}{@{}c@{}}
Mean Lower Bound \\\tiny{ ($\times 10^{-3}$)}\end{tabular}}
& \multicolumn{1}{c}{\begin{tabular}{@{}c@{}}
Mean Upper Bound \\\tiny{ ($\times 10^{-3}$)}\end{tabular}}
& \multicolumn{3}{c}{\begin{tabular}[c]{@{}c@{}}Median \\Percentage Gap (\%)\end{tabular}}
\\\cmidrule(lr){3-5} \cmidrule(lr){6-6} \cmidrule(lr){7-9}

& & \multicolumn{1}{c}{\tiny\lpd{}} & \multicolumn{1}{c}{\tiny\lpp{}}& \multicolumn{1}{c}{\tiny\lpo{}} & \multicolumn{1}{c}{\tiny {PGD}}& \multicolumn{1}{c}{\tiny\lpd{}} & \multicolumn{1}{c}{\tiny\lpp{}} & \multicolumn{1}{c}{\tiny\lpo{}}\\

\midrule

CNN-small 
 & Normal  &  14.98 & 16.29  & 18.87  & 52.70  &  69.12 & 66.03  & 61.40 
\\
 & Adv  &  73.42 & 77.09  & 85.94  & 155.16  &  52.52 & 50.14  & 44.42 
\\
 & LPd  &  153.17 & 160.83  & 160.83  & 226.72  &  29.72 & 26.21  & 26.21 
\\
\midrule
CNN-Wide-1 
 & Normal  &  14.09 & 15.76  & 16.92   & 39.61  &  58.84 & 54.69  & 52.66 
\\
 & Adv  &  81.52 & 86.25  & 91.76  & 142.89  &  43.59 & 40.77  & 37.58 
\\
 & LPd  &  116.72 & 122.55  & 122.55  & 183.66  &  33.90 & 30.59  & 30.59 
\\
\midrule

CNN-Wide-2 
 & Normal  &  13.29 & 14.83  & 16.82  & 43.95  &  68.34 & 64.40  & 60.42 
\\
 & Adv  &  91.50 & 96.08  & 104.02  & 179.86  &  49.83 & 47.32  & 41.98 
\\
 & LPd  &  148.07 & 156.78  & 169.77  & 221.67  &  32.45 & 27.76  & 21.03 
\\
\midrule

CNN-Wide-4
 & Normal  &  12.84 & 14.37  & 16.45  & 47.23  &  72.23 & 68.06  & 63.06 
\\
 & Adv  &  67.64 & 72.34  & 79.90  & 178.01  &  62.72 & 59.37  & 55.17 
\\
 & LPd  &  142.30 & 149.41  & 155.23  & 217.64  &  34.92 & 31.67  & 29.34 
\\
\midrule

CNN-Wide-8 
 & Normal  &  10.82 & 11.72  & 13.35  & 47.75  &  75.49 & 71.85  & 69.36 
\\
 & Adv  &  62.57 & 67.42  & 77.66  & 181.09  &  64.45 & 62.17  & 55.57 
\\
 & LPd  &  N.A    & N.A  &   N.A  &  N.A    & N.A  &   N.A  &   N.A   
\\
\midrule

CNN-Deep-1 
 & Normal  &  15.21 & 16.78  & 19.58  & 44.79  &  66.50 & 62.04  & 55.44 
\\
 & Adv  &  94.68 & 99.41  & 100.20  & 166.38  &  39.81 & 36.80  & 35.93 
\\
 & LPd  &  136.09 & 142.89  & 142.89  & 184.23  &  22.10 & 18.20  & 18.20 
\\
\midrule

CNN-Deep-2 
 & Normal  &  6.12 & 6.42  & 8.76  & 43.32  &  84.47 & 83.69  & 78.65 
\\
 & Adv  &  102.47 & 107.60  & 112.82  & 185.70  &  39.35 & 36.32  & 36.32 
\\
 & LPd  &  N.A    & N.A  &   N.A  &  N.A    & N.A  &   N.A  &   N.A   
\\
\midrule

MLP-[9]-500 
 & Normal  &  12.64 & 13.27  & 16.84  & 45.84  &  74.57 & 73.30  & 63.14 
\\
 & Adv  &  20.77 & 21.99  & 28.50  & 129.45  &  84.60 & 83.83  & 79.05 
\\
 & LPd  &  N.A    & N.A  &   N.A  &  N.A    & N.A  &   N.A  &   N.A   
\\
\midrule

MLP-[9]-100 
 & Normal  &  11.35 & 11.92  & 14.23  & 31.37  &  64.13 & 62.34  & 57.03 
\\
 & Adv  &  19.41 & 21.12  & 25.41  & 94.57  &  75.15 & 71.42  & 63.96 
\\
 & LPd  &  68.25 & 71.51  & 73.96  & 103.87  &  29.79 & 26.28  & 26.28 
\\
\midrule

MLP-[2]-100 
 & Normal  &  14.19 & 15.11  & 15.83  & 28.14  &  52.66 & 47.82  & 45.56 
\\
 & Adv  &  41.68 & 43.76  & 43.76  & 81.22  &  36.23 & 33.04  & 33.04 
\\
 & LPd  &  81.50 & 85.33  & 85.33  & 118.10  &  25.01 & 21.26  & 21.26 
\\
\bottomrule
\end{tabular}
\end{adjustbox}
\end{small}
\end{sc}
\vskip -1in
\end{table*}

\begin{table*}[htbp]
\centering
\caption{Certified bounds on the minimum adversarial distortion $\epsilon$ for ten random samples from the test set of CIFAR-10.}
\vskip 0.15in
\label{table:epsilon_bounds-cifar}
\begin{sc}
\begin{small}
\begin{adjustbox}{max width=1\textwidth}
\begin{tabular}{lcccccccc}
\toprule

\multicolumn{1}{c}{\multirow{2}[2]{*}{Network}}
&\multicolumn{1}{c}{\multirow{2}[2]{*}{\begin{tabular}[c]{@{}c@{}}Training\\ Mode\end{tabular}}}
& \multicolumn{3}{c}{\begin{tabular}{@{}c@{}}
Mean Lower Bound \\\tiny{ ($\times 10^{-3}$)}\end{tabular}}
& \multicolumn{1}{c}{\begin{tabular}{@{}c@{}}
Mean Upper Bound \\\tiny{ ($\times 10^{-3}$)}\end{tabular}}
& \multicolumn{3}{c}{\begin{tabular}[c]{@{}c@{}}Median \\Percentage Gap (\%)\end{tabular}}
\\\cmidrule(lr){3-5} \cmidrule(lr){6-6} \cmidrule(lr){7-9}

& & \multicolumn{1}{c}{\tiny\lpd{}} & \multicolumn{1}{c}{\tiny\lpp{}}& \multicolumn{1}{c}{\tiny\lpo{}} & \multicolumn{1}{c}{\tiny {PGD}}& \multicolumn{1}{c}{\tiny\lpd{}} & \multicolumn{1}{c}{\tiny\lpp{}} & \multicolumn{1}{c}{\tiny\lpo{}}\\

\midrule
CNN-small 
 & Normal  &  7.48 & 7.86  & 8.46   & 20.13 &  49.40 & 46.87  & 44.23  
\\
 & Adv  &  24.33 & 26.53  & 27.59   & 37.90 &  34.50 & 24.67  & 24.67  
\\
 & LPd  &  67.34 & 72.27  & 77.84   & 157.01 &  52.94 & 48.13  & 43.13  
\\
\midrule
CNN-Wide-1
 & Normal  &  6.97 & 7.32  & 7.56   & 14.57 & 43.01 & 40.16  & 39.39  
\\
 & Adv  &  58.52 & 63.26  & 67.84   & 115.47 &  49.83 & 46.63  & 42.15  
\\
 & LPd  &  57.03 & 62.51  & 65.83   & 122.00 &  41.22 & 38.29  & 32.40  
\\
\midrule

CNN-Wide-2
 & Normal  &  8.27 & 8.86  &  9.46   & 22.16 & 58.66 & 54.53  & 52.46  
\\
 & Adv  &  42.05 & 45.99  & 49.09   & 74.13 &  35.10 & 29.85  & 25.54  
\\
 & LPd  &  73.19 & 81.75  & 87.38   & 157.03 &  47.64 & 39.78  & 39.78  
\\
\midrule

CNN-Wide-4 
 & Normal  &  4.14 & 4.35  & 4.63   & 10.97 &  40.27 & 37.28  & 33.03  
\\
 & Adv  &  29.11 & 32.84  & 35.45   & 71.57 &  50.59 & 44.21  & 43.18  
\\
 & LPd  &  41.62 & 47.17  & 48.51   & 104.49 &  45.19 & 39.67  & 39.67  
\\
\bottomrule

\end{tabular}
\end{adjustbox}
\end{small}
\end{sc}
\vskip -0.1in
\end{table*}

\begin{figure}[t]
    \centering
    \includegraphics[width=0.7\textwidth]{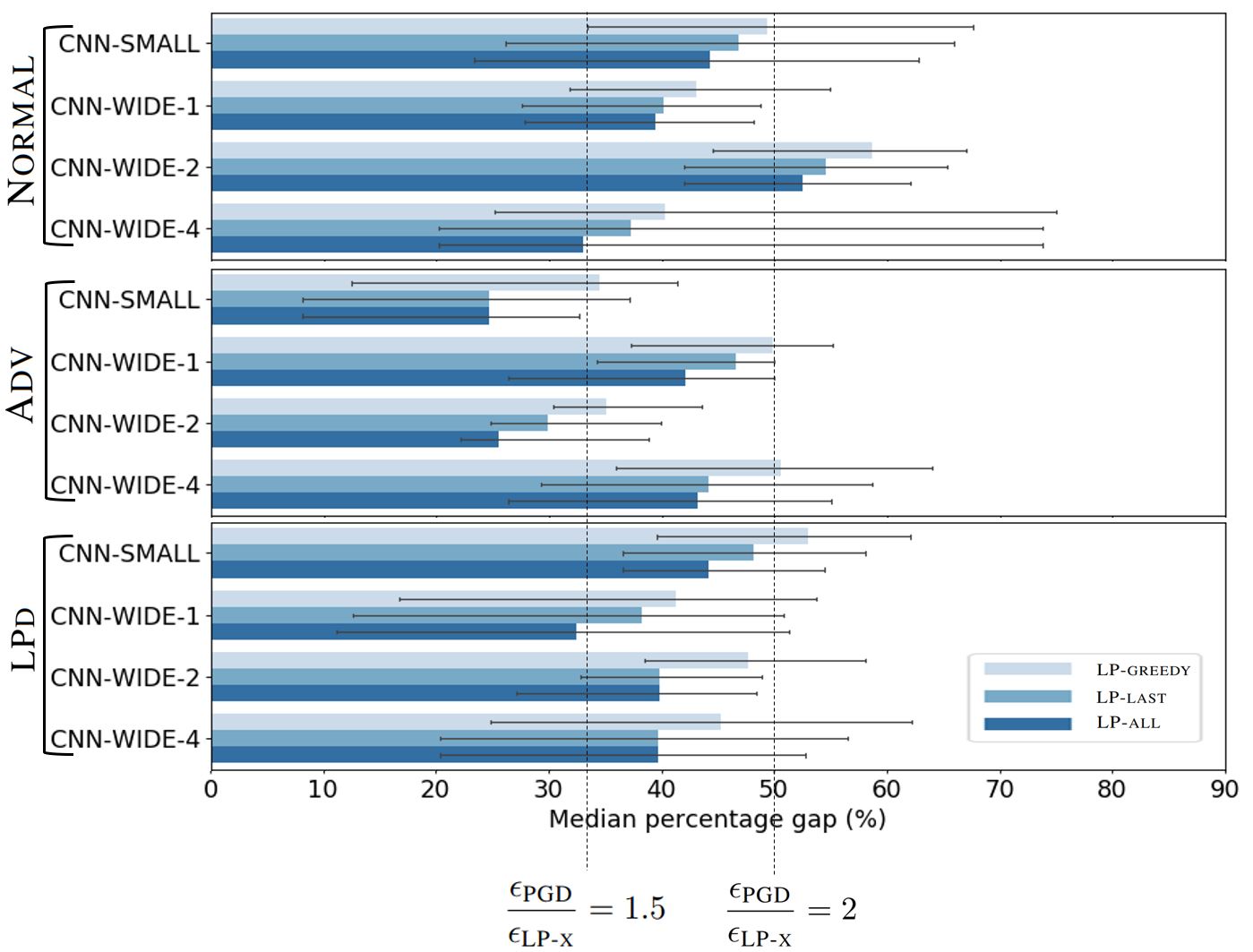}
    \caption{The median percentage gap of minimum adversarial distortion for CIFAR-10, in the same format as Fig.~\ref{fig:gap_mnist}. For more details, please refer to Table~\ref{table:epsilon_bounds-cifar} in Appendix~\ref{sec:exp2results_appendix}.}
    \label{fig:gap_cifar}
\end{figure}

\section{Results on Randomly Initialized Networks}
\label{sec:appendix-random-networks}
In this section, we report additional results for the $\epsilon$-search experiment because they might be of interest as a comparison.
The results are reported in Table~\ref{table:epsilon_bounds-random-initialized-nets}. The results are in accordance to what was discussed in Seciton~\ref{sec:exp2} i.e. for all networks and both datasets, the certified lower bounds on $\epsilon$ using \lpd{}, \lpp{}, or \lpo{} are 2 to 3 times smaller than the upper bound found by PGD. Furthermore, the improvement that we get using  \lpo{} and \lpp{} over \lpd{} is not significant and doesn't close the gap with the PGD upper bound.

\begin{table*}[ht!]
\centering
\caption{Certified bounds on the minimum adversarial distortion $\epsilon$  for ten random samples from the test set of MNIST and CIFAR-10 on randomly initialized networks (no training).}
\vskip 0.15in
\label{table:epsilon_bounds-random-initialized-nets}
\begin{sc}
\begin{small}
\begin{adjustbox}{max width=1\textwidth}
\begin{tabular}{lcccccccc}
\toprule

\multicolumn{1}{c}{\multirow{2}[2]{*}{Network}}
&\multicolumn{1}{c}{\multirow{2}[2]{*}{\begin{tabular}[c]{@{}c@{}}Training\\ Mode\end{tabular}}}
& \multicolumn{3}{c}{\begin{tabular}{@{}c@{}}
Mean Lower Bound \\\tiny{ ($\times 10^{-3}$)}\end{tabular}}
& \multicolumn{1}{c}{\begin{tabular}{@{}c@{}}
Mean Upper Bound \\\tiny{ ($\times 10^{-3}$)}\end{tabular}}
& \multicolumn{3}{c}{\begin{tabular}[c]{@{}c@{}}Median \\Percentage Gap (\%)\end{tabular}}
\\\cmidrule(lr){3-5} \cmidrule(lr){6-6} \cmidrule(lr){7-9}

& & \multicolumn{1}{c}{\tiny\lpd{}} & \multicolumn{1}{c}{\tiny\lpp{}}& \multicolumn{1}{c}{\tiny\lpo{}} & \multicolumn{1}{c}{\tiny {PGD}}& \multicolumn{1}{c}{\tiny\lpd{}} & \multicolumn{1}{c}{\tiny\lpp{}} & \multicolumn{1}{c}{\tiny\lpo{}}\\

\midrule

\textbf{MNIST}\\
CNN-small 
& Random  &  5.79 & 6.08  & 6.25  & 14.86 &  51.37 & 48.94  & 48.94  \\

CNN-Wide-1 
& Random  &  10.42 & 10.94  & 11.98   & 33.77  &  67.09 & 65.45  & 62.16  \\

CNN-Wide-2 
& Random  &  8.12 & 8.53  & 9.34   & 29.43  &  72.54 & 71.17  & 68.42  \\

CNN-Wide-4 
& Random  &  8.68 & 9.12  & 9.99   & 45.26  &  78.65 & 77.59  & 75.45  \\

CNN-Deep-1 
& Random  &  11.12 & 11.81  & 12.79   & 42.28  &  72.76 & 71.40  & 68.67  \\

MLP-[2]-100 
& Random  &  4.69 & 5.16  & 5.25   & 15.71  &  64.85 & 58.53  & 57.83  \\
%  \midrule
 \midrule

 \textbf{CIFAR-10} \\
CNN-small
& Random  &  8.77 & 10.01  & 10.13   & 24.50 &  62.61 & 57.04  & 57.04  \\

CNN-Wide-1
& Random  &  5.61 & 5.89  & 6.09   & 11.33 &  45.27 & 42.53  & 41.46  \\

CNN-Wide-2
& Random  &  2.83 & 3.31  &  3.31  & 6.24 &  50.60 & 46.13  & 46.13  \\

CNN-Wide-4 
& Random  &  8.93 & 8.52  & 9.00   & 28.69 &  69.63 & 68.11  & 68.11  \\

\bottomrule
\end{tabular}
\end{adjustbox}
\end{small}
\end{sc}
%\vspace*{6in}
\end{table*}

\end{document}